%% file: Main.tex
\documentclass[twoside,11pt]{article}

\usepackage{blindtext}

\usepackage{jmlr2e}
\usepackage{subcaption}
\input{inputs/packages}
\input{inputs/mathcommands}

\input{inputs/Commenting}

\usepackage{booktabs}
\usepackage{longtable}
\usepackage{multirow}
\usepackage{geometry}
\usepackage{lastpage}

\usepackage{pdfpages}

\ShortHeadings{Online Federation For Mixtures of Proprietary Agents with Black-Box Access}{Xuwei Yang, Fatemeh Tavakoli, Anastasis Kratsios, and David B. Emerson}
\firstpageno{1}

\begin{document} 

\title{Online Federation For Mixtures of Proprietary Agents\\ with Black-Box Encoders}

\author{\name Xuwei Yang 
  \email yangx212@mcmaster.ca \\
\addr Department of Mathematics\\
  McMaster University\\
  Ontario, Canada 
  \AND
   \name
   Fatemeh Tavakoli
   \email 
fatemeh.tavakoli@vectorinstitute.ai
   \\
\addr Vector Institute\\
    Ontario, Canada 
  \AND
  \name David B. Emerson 
  \email 
david.emerson@vectorinstitute.ai \\ 
\addr  Vector Institute\\
    Ontario, Canada 
    \AND 
    \name
  Anastasis Kratsios\thanks{Corresponding Author. 
  Funded by NSERC DG No.\ RGPIN-2023-04482.} 
  \email 
  kratsioa@mcmaster.ca 
  \\
 \addr  Department of Mathematics\\
  McMaster University and the Vector Institute\\
  Ontario, Canada 
}

\editor{}

\maketitle

\begin{abstract} 
Most industry-standard generative AIs and feature encoders are proprietary, offering only black-box access: their outputs are observable, but their internal parameters and architectures remain hidden from the end-user.  This black-box access is especially limiting when constructing mixture-of-expert type ensemble models since the user cannot optimize each proprietary AI's internal parameters.  Our problem naturally lends itself to a non-competitive game-theoretic lens where each proprietary AI (agent) is inherently competing against the other AI agents, with this competition arising naturally due to their obliviousness of the AI's to their internal structure.  In contrast, the user acts as a central planner trying to synchronize the ensemble of competing AIs.  

We show the existence of the unique Nash equilibrium in the online setting, which we even compute in closed-form by eliciting a feedback mechanism between any given time series and the sequence generated by each (proprietary) AI agent.  Our solution is implemented as a decentralized, federated-learning algorithm in which each agent optimizes their structure locally on their machine without ever releasing any internal structure to the others.  We obtain refined expressions for pre-trained models such as transformers, random feature models, and echo-state networks.  Our ``proprietary federated learning'' algorithm is implemented on a range of real-world and synthetic time-series benchmarks. It achieves orders-of-magnitude improvements in predictive accuracy over natural benchmarks, of which there are surprisingly few due to this natural problem still being largely unexplored.
\end{abstract}

\section{Introduction} 
\label{s:Intro}

The modern era of mainstream deep learning has witnessed the emergence of ensembles of proprietary models with public \textit{black-box} access; that is, the users do not know nor can they modify the inner structure of these AIs used to build their ensemble models. 
Examples include standard large pre-trained language and image encoders such as CLIP \citep{radford2021learning}, purely image-based encoders like DINO(v2) \citep{caron2021emerging,oquab2024dinov2learningrobustvisual}, pre-trained large language encoders SLIP(v2) \citep{mu2022slip,tschannen2025siglip}, or directly using the outputs of proprietary LLMs, e.g. \citep{chatgpt} or \citep{guo2024deepseek} as deep features. 
A similar scenario arises within organizations where different teams -- for example, trading desks in a financial institution -- employ distinct models whose structures are not disclosed to one another. In such cases, only the predictions are shared, rendering traditional machine-learning techniques inapplicable since the user does not have access to, and thus \textit{cannot train}, these models.

Our problem naturally lends itself to a \textit{federated learning} (FL) approach, where a single computational task is distributed among many decentralized agents. These models are synchronized by the user/server.  This is because, in this setting, each agent can hold their own model's information locally while only releasing their predictions to the ensemble. Though the FL view may be natural when optimizing an ensemble of proprietary AIs, the black-box constraint pulls it out of the standard FL toolbox. This is because traditional FL setups, e.g. \citep{Even22SampleOptimalityFL,Fraboni23FO-AsynchronHetero, FedOptimHeteroNet, FL:all-in-one}, work by (partially) averaging local model parameters on the central server or their (partial) gradients, requiring full (white-box) access to each agent.

It is, thus, natural to approach our \textit{proprietary federated learning (PFL)} problem through a non-cooperative game-theoretic lens, recognizing that each agent can optimize their own predictions but they are constrained to act non-cooperatively by virtue of them not knowing the structure of any other agent.  Instead, the agents are guided towards cooperation by the user, acting as a central planner, who continually re-weights the contributions of each expert’s prediction to maximize the ensemble’s overall predictive performance.  This non-cooperative game-theoretic perspective provides a meaningful notion of an \textit{optimal} PFL algorithm, namely one which achieves a Nash equilibrium among the agents.

\paragraph{Main Contribution}
Though a solution to the general PFL problem is yet elusive, we have completely solved it within the generative pre-trained online setting. Each AI agent and the user/server have access to the same dynamically evolving time series and the corresponding outputs generated by each AI.   Using the information in these evolving input-output pairs, we can construct a feedback mechanism, enabling every agent to optimize their ensemble while preserving the proprietary nature of each AI (Algorithm \ref{alg:sync_subroutine}).  

Our non-cooperative game-theoretic approach precisely quantifies what ``optimality'' means: the agent's act in a \textit{Nash Equilibrium}.  This Nash equilibrium not only exists, but our PFL algorithm implements the unique \textit{closed-form} solution to this game (Theorem \ref{thm:NashEqmExistence}).  

\paragraph{Secondary Contributions}
We further focus on three classes of online deep learning models, where our algorithm’s updates can be streamlined for improved computational efficiency. First, we examine deterministic (non-generative) pre-trained deep learning models -- such as those commonly used for feature encoding -- which have no latent memory (Corollary \ref{cor:PreTrainedTransformer}). Further, we consider a class of generative extreme learning machines, specifically random shallow neural networks with continuously re-generated random biases (Corollary \ref{cor:Random_NeuralNetwork}), motivated by their theoretical interest propelled by the NTK literature. Unlike the first class, these models are generative but lack latent memory. Finally, in Section \ref{s:Examples_ReservoirComputer}, we explore echo-state networks as an example of deep learning models leveraging randomized internal states and latent memory. Each model class is supported by numerical studies demonstrating the effectiveness of our PFL algorithm across a range of online prediction tasks, using both real-world and synthetic datasets. Our PFL algorithm outperforms the benchmark by orders of magnitude in all cases.

\paragraph{Organization of the Paper}
In Section \ref{s:Setup}, we introduce the generative sequential learning models and the FL scheme. Section \ref{s:Main_Resuls} contains our main theoretical guarantees. In Section \ref{s:Examples}, we examine various FL agents and demonstrate, through numerical experiments, the advantage of our game-theoretic FL approach.   All proofs and additional supporting experimental details are relegated to comprehensive appendices.

\subsection{Related Work}
\label{s:Introduction__ss:RelatedWork}

Traditional FL methods, such as FedAvg \citep{mcmahan2017FedAvg}, 
aggregate locally-optimized parameters or gradient updates on a central server, which periodically synchronizes local models \citep{Even22SampleOptimalityFL}. These methods rely on idealized assumptions such as homogeneous agents, identically and independently distributed datasets, and synchronized updates from clients.  However, these stylized assumptions tend not to match real-world use cases, wherein agents face heterogeneous data distributions, different system capacities, and de-synchronized updates  \citep{Kairouz21AdvancesFL, Li20FLfuture}.

Due to heterogeneity between agents, averaging of model weights as a means of synchronizing distributed model training \citep{mcmahan2017FedAvg} leads to biased global models and lower prediction accuracy \citep{Zhao22nonIID}. 
Some robust optimization strategies have been designed to address this issue. For example, FedProx \citep{FedOptimHeteroNet} incorporates regularization to stabilize updates across heterogeneous agents, while SCAFFOLD \citep{Karimireddy20Scaffold} employs variance reduction to mitigate local drift. Another challenge in FL is desynchronization of updates from agents due to varying computational speeds and capacity. Asynchronous FL methods address this by allowing each agent to update the global model independently as soon as it completes its local computation \citep{Fraboni23FO-AsynchronHetero, FedOptimHeteroNet, FL:all-in-one}.

\paragraph{Game-Theoretic Approaches to Machine Learning}
Federated learning can be naturally cast as an interactive multi-agent system wherein optimal control and game theory can be applied. In collaborative FL settings, agents share their data and update local models to achieve better performance of the global model \citep{Blum21EqmOptmFL, Donahue21OptimalityStability, Huang23Collaboration, Murhekar23Incentives}. 
Several challenges arise in collaborative FL, such as privacy leakage \citep{Wei20DiffPrivacy}, communication costs \citep{Yoon25LessCommunication}, and conflicting objectives \citep{Yin24PrivacyAccuracyEnergy}. 
Non-cooperative game-theoretic models have been used to address challenges such as adversarial threats \citep{Zhang23RobustGameFL,Xie24mixed} and privacy protection \citep{Blanco-Justicia21SecurityPrivacy, Yin21PrivacyPreserving}. Utilizing non-cooperative game theory to formulate and improve multi-agent model performance, especially with proprietary AI agents, remains an emerging area of research.

\subsection{Notation and Conventions} 
\label{s:notation}
In this section, the notation used throughout the manuscript is gathered for clarity and accessibility. We denote $\mathbb{N} \eqdef \{0, 1, 2, \cdots \}$ and $\mathbb{N}_+\eqdef \{x\in \mathbb{N}:\, x>0\}$.  
We henceforth fix a latent, complete probability space $(\Omega,\mathcal{F},\mathbb{P})$.  For any $d,D\in \mathbb{N}_+$ and $1\le p<\infty$, use $L^p(\mathbb{R}^d)$ (resp.\ $L^p(\mathbb{R}^{d\times D})$) to denote the space of $\mathbb{R}^d$ (resp.\ $\mathbb{R}^{d\times D})$) valued random variables $X$ whose expected $p^{\text{th}}$ power of their Euclidean $\|\cdot\|_2$ (resp.\ Frobenius $\|\cdot\|_F$) norm is finite $\mathbb{E}[\|X\|^p_2]<\infty$ (resp.\ $\mathbb{E}[\|X\|_F^p]<\infty$), where the expectation is taken with respect to $\mathbb{P}$. As shorthand for an arbitrary Cartesian product of a single space $X$, $X \times \ldots \times X$, we write $(X)^{\mathbb{N}}$.

\paragraph{Row and Column Vectors and Matrices}
We will nearly always specify if a vector is a row or column vector and its dimensions.  However, when clear from the context, we identify the space of row vectors $\mathbb{R}^{d\times 1}\cong \mathbb{R}^d$ and the space of column vectors $\mathbb{R}^{1\times d}\cong \mathbb{R}^d$ (for the relevant $d\in \mathbb{N}_+$) to keep notation light.  
For each $d\in \mathbb{N}_+$, $I_d$ is the $d \times d$-identity matrix.  Moreover, $\bar{1}_d\in \mathbb{R}^d$ has all its entries equal to $1$. 
We denote by $0$ the zero scalar, zero vector, or zero matrix with dimensions specified in the context. 
The Kronecker product of any two matrices $A$ and $B$ is denoted by $A\otimes B$. 
The Euclidean ($\ell^2$) norm of any vector $x\in \mathbb{R}^d$, for any given dimension $d\in \mathbb{N}_+$, is denoted by $\| \cdot \|$.

\section{The Online Proprietary Federated Setting}
\label{s:Setup}

\paragraph{The Agents - Online Sequential Learners}
We consider $N\in \mathbb{N}_+$ generative \textit{proprietary} sequential learners $f^{1},\dots,f^{N}$ with at most three inputs at any time $t\in \mathbb{N}$.  For each such $t\in \mathbb{N}_+$, the $i^{\text{th}}$ expert model takes the (common) historical input sequence $x_{[0:t]}\in \mathbb{R}^{d_x \, (1+t)}$, an individual source of randomness $\epsilon_t^i$ which is an integrable $d_y \times d_z$-valued random variable, and an individual hidden state $Z_t^i\in \mathbb{R}^{d_y\times d_z}${; where $d_x, d_y, d_z \in \mathbb{N}_+$}.  
These three inputs are used to generate the latent state $Z^i_{t+1}$ using a \textit{static} \textbf{encoder} $\varphi^{i}:\mathbb{R}^{d_x(1+t)} \times \mathbb{R}^{d_y \times d_z}\times 
 \mathbb{R}^{d_y\times d_z }\to \mathbb{R}^{d_y 
\times d_z}$ via
\begin{equation}
\label{eq:ENCODER}
    Z_{t}^i 
\eqdef 
    \underbrace{
        \varphi^{i}\big(x_{[0:t]}
        ,Z_{t-1}^i,\epsilon_t^i\big)
    }_{\text{(Deep) Feature Encoder}}
.
\end{equation}
The (deep) feature encoder is either pre-trained or trained on an initial segment of time-series data, formally regarded as before the time when prediction begins to roll out online; i.e.\ $t=0$ and any negative times are considered part of an offline pre-training phase not considered here.  Thus, on and after time $t>0$, the feature encoder is no longer updated in this paper, and the (deep) learning model effectively becomes a kernel function.  We require that $Z_t^{i}\in L^1(\mathbb{R}^{d_y\times d_z})$ for each $t\in \mathbb{N}$.
Once the latent state is generated, it is used to update the \textit{residual} of the prediction generated at the previous time $t$ via a \textit{dynamically updating} linear \textbf{decoder} parameterized by $\beta^i$.
\begin{equation} 
\label{eq:DECODER_RESIDUALS}
    \underbrace{
        \hat{Y}^i_{t+1}
    }_{\text{Prediction}}
= 
        \underbrace{
            \hat{Y}^i_t 
        }_{\text{Historical State (Memory)}}
    +  
    \underbrace{
        \overbrace{Z^i_t 
        }^{\text{feature}}
            \beta^{i}_t
        ,
    }_{\text{Residual (Update)}}
\end{equation} 
where $\hat{Y}_t^i \in \mathbb{R}^{d_y \times 1}$, $Z_t^i\in \mathbb{R}^{d_y \times d_z}$, and $\beta_t^i \in \mathbb{R}^{d_z \times 1}$. Both initial states $\hat{Y}_0^i\in \mathbb{R}^{d_y\times 1}$ and $Z_0^i\in \mathbb{R}^{d_y \times d_z}$ are fixed a priori.  

Following the theoretical reservoir computing literature, e.g. \citep{grigoryeva2018echo}, each generative sequential is a \textit{causal} deep learning agent defined by
\begin{equation*}
\begin{aligned}
    f^{i}: (\mathbb{R}^{d_{x}})^{\mathbb{N}}&\to L^1(\mathbb{R}^{d_{y}})^{\mathbb{N}}, \\
    (x_t)_{t=0}^{\infty} & \to \big(\hat{Y}^i_t\big)_{t=0}^{\infty}.
\end{aligned}
\end{equation*}
where, for each $t\in \mathbb{N}$, $\hat{Y}^{i}_{t+1}$ is computed recursively  by first encoding the input path $x_{[0:t]}$ via Equation \eqref{eq:ENCODER}, and then updating the residual via Equation \eqref{eq:DECODER_RESIDUALS}.  By \textit{causality}, we mean that for each time $t\ge 0$, the prediction $\hat{Y}_{t+1}^i$ only depends on the input sequence as \textit{observed thus far}, i.e.\ on $x_{[0:t]}$, and not on any of its future realizations, i.e.\ not on $x_s$ for any $s>t$.

For example, in echo-state networks (ESNs) \citep{jaeger2001echo,maass2002real} or extreme learning machines such as random feature networks (RFNs) \citep{huang2006universal,rahimi2008weighted}, the deep features in Equation \eqref{eq:ENCODER} are randomly initialized and left untrained and residual updates, parameterized by $(\beta_t)_{t=0}^{\infty}$, in Equation \eqref{eq:DECODER_RESIDUALS} are trained before the online-time $t=0$ and then frozen (thus are constant) for all times $t \ge 0$.  In pre-trained models, such as transformers \citep{vaswani2017attention}, the deep features are pre-trained on an offline dataset, i.e.\ before time $t=0$, and the residual decoding/readout layer in Equation \eqref{eq:DECODER_RESIDUALS} is also typically left fixed.

\paragraph{The Server - Adaptive Mixture of Experts/Agents}
We assume that $N>1$, with the case of $N=1$ being trivial, as no multi-agent coordination is required. The objective of the server is to adaptively \textit{synchronize} these $N$ generative models, into a single generative ensemble model $F:(\mathbb{R}^{d_{x}})^{\mathbb{N}}\to L^1(\mathbb{R}^{d_{y}})^{\mathbb{N}}$ to maximize their performance.  The resulting ensemble is represented for each $x \in (\mathbb{R}^{d_x})^{\mathbb{N}}$ at every time $t\in \mathbb{N}$ by
\begin{equation} 
\label{eq:enbsemble_server__predictions}
F(x_{[0:t]})_t \eqdef 
\overbrace{
\sum_{i=1}^N\, 
    w_t^i 
    \underbrace{
        f^{i} \big(x_{[0:t]}\big)_t
    \,
    ,
    }_{i^{\text{th}} \text{Agent (Expert)}}
}^{\text{Mixture of Experts}}
\end{equation}
where the \textit{mixture coefficients} $w_t$ belongs to $\mathbb{R}^N$. Observe that, by construction, the server's mixture prediction is also \textit{causal}; that is, $F(x_{[0:t]})_t$ only depends on $x_{[0:t]}$ and not $x_s$ for any $s>t$.

\paragraph{The Proprietary Constraint}
Each \textit{agent} can dynamically update, or \textit{control}, its sequence of linear residual update parameters $(\beta_t)_{t=0}^{\infty}$ in Equation \eqref{eq:DECODER_RESIDUALS}, while the central \textit{server} dynamically updates its own parameters in Equation \eqref{eq:enbsemble_server__predictions}. In standard federated or online optimization problems, there are no restrictions on how agents and the server update their trainable parameters, apart from the requirement that updates be causal. That is, they must not depend on inaccessible future information.  

However, the \textit{proprietary} nature of each model introduces an additional layer of structural complexity. Specifically, neither the server nor any agent has access to the internal structure of any other model, including its architecture, parameter values, or (partial) gradient information. Together, the \textit{causality} and \textit{proprietary} constraints imply that, 
at every time $t\in \mathbb{N}_+$, the $i^{\text{th}}$ \textit{agent/client-side} parameter updates to $\beta_t^{i}$ in Equation \eqref{eq:DECODER_RESIDUALS} are of the form
\begin{equation}
\label{eq:PROPRIETARY_CONSTRAINT__Agent}
    \beta_t^{i} = b_t\Bigg(
                \left(\hat{Y}_{[0:t]}^{j}\right)_{j=1}^N
        , 
            \left(Z_{[0:t]}^{j}\right)_{j=1}^N
        ,
                y_{[0:t]}
        ,
                x_{[0:t]}
        ,
            w_{[0:t]}
    \Bigg),
\end{equation}
for each $t\in \mathbb{N}_+$, where $b_t:\mathbb{R}^{d_y(1+t)\times N }\times \mathbb{R}^{(d_y \times d_z)(1+t)\times N }\times \mathbb{R}^{d_y(1+t)}\times \mathbb{R}^{d_x(1+t)}\times \mathbb{R}^{N(1+t)} \to \mathbb{R}^{d_z \times 1}$ is a continuous function that may depend on the current time $t$ and various observed quantities, including all historical expert predictions $\hat{Y}^{1}_{[0:t]}, \dots, \hat{Y}_{[0:t]}^{N}$, their historical deep features $Z_{[0:t]}^{1},\dots,Z_{[0:t]}^{N}$ the server weights $w_{[0:t]}$, and the input data $x_{[0:t]}$.  We highlight that the $i^{\text{th}}$ agent’s parameter updates $\beta_{[0:t]}^{i}$ are known to themselves and can therefore be incorporated into their own decisions.%
\footnote{Expression \eqref{eq:PROPRIETARY_CONSTRAINT__Agent} could be extended to include the agent's internal parameters if the hidden layers were trainable, but they are not here.}

Similarly, as the server does not have black-box access to any other proprietary model, then, 
at each time $t\in \mathbb{N}_+$, the server-side mixture coefficients $w_t$ in Equation \eqref{eq:enbsemble_server__predictions}, are of the form
\begin{equation}
\label{eq:PROPRIETARY_CONSTRAINT__Server}
    w_t = \varpi_t\Bigg(
                \left(\hat{Y}_{[0:t+1]}^{j}\right)_{j=1}^N
        ,
            \left(Z_{[0:t]}^{j}\right)_{j=1}^N
        ,
                y_{[0:t]}
        ,
                x_{[0:t]}
        ,
            w_{[0:t]}
    \Big),
\end{equation}
where $\varpi_t:\mathbb{R}^{d_y(2+t)\times N}\times \mathbb{R}^{ ( d_y \times d_z)(1+t)\times N}\times \mathbb{R}^{d_y(1+t)}\times \mathbb{R}^{d_x(1+t)}\times \mathbb{R}^{N(1+t)} \to \mathbb{R}^N$ is a continuous function that may depend on the current time $t \in \mathbb{N}_+$.  
Note that, since the agents each predict first, then the server's update can leverage the information in $\hat{Y}_{t+1}^{1},\dots,\hat{Y}_{t+1}^{N}$. 

We emphasize that each client-side parameter $\beta_t^{i}$ is computed using strictly more information than the server-side weights $w_t$. This distinction arises because each agent has access to its own internal parameters, whereas the server does not, as these parameters are proprietary and remain inaccessible to the server. 
\begin{remark}[The Proprietary Constraint is a Measurability Restriction]
\label{rem:adapedness}
If $b_t$ and $\varpi$ were only required to be Borel measurable, then Conditions \eqref{eq:PROPRIETARY_CONSTRAINT__Agent} and \eqref{eq:PROPRIETARY_CONSTRAINT__Server} could be formulated as \textit{adaptedness} conditions. Specifically, $\beta_t^{i}$ would be measurable with respect to the $\sigma$-algebra generated by $\{\hat{Y}_{s}^{j},\,w_{u}, \beta_{s}^{i},Z_s^{i}\}_{s,u,j=1}^{t,t+1,N}$, while $w_t$ would be required to be measurable with respect to the $\sigma$-algebra generated by $\{\hat{Y}_{s}^{j},\,w_{u},Z_s^{i}\}_{s,j=1}^{t,N}$.   However, imposing additional continuity constraints on $b_t$ and $w_t$ enhances the structural regularity of the problem.  
\end{remark}

\section{Main Results}
\label{s:Main_Resuls}
{Our first result formalizes and derives the optimal server-side update (Theorem \ref{thm:OptimalWeights}).  We derive the optimal client-side update in (Theorem \ref{thm:NashEqmExistence}).  We additionally provide a greedy scheme which each client can use between these synchronization times (Proposition \ref{prop:localgreedyupdates}) for an added speedup at the cost of sub-optimal decisions.}

\subsection{Optimal Server-Side Weights}
\label{s:Main_Results__Server}

Let $t\in \mathbb{N}_+$ and fix a regularization parameter $\kappa>0$.  
Given any $\left(\hat{Y}_{t}^{j}\right)_{j=1}^N\in L^1(\mathbb{R}^{d_y})^{N}$ and $y_t\in \mathbb{R}^{d_y}$, 
we define the server-side objective as
\begin{equation}
\label{eq:server}
\begin{aligned}
    \min_{w_t\in \mathbb{R}^N}\,
    &
        \Big\| y_{t}  - \sum_{j=1}^N w^j_t \hat{Y}^{j}_{t} \Big\|^2 
    + 
        \sum_{j=1}^N \kappa | w^j_t |^2,
\\
\mbox{s.t.} & \,
w^1_t+\dots+w^N_t = \eta
.
\end{aligned}
\end{equation}
The normalization constraint $w^1_t+\dots+w^N_t = \eta$ ensures that the \textit{total weight size} is equal to $\eta$. When $\eta=1$, this condition serves as a \textit{signed} variant of the usual probabilistic normalization constraint found in the PAC-Bayes literature \citep{alquier2024user} optimized by a Gibbs-type measure, where instead there $w_t$ is constrained to the $N$-simplex, meaning each weight must satisfy $0 \leq w^i_t \leq 1$ for each $i$.   However, this additional constraint is removed since we are not \textit{sampling} agent predictions but rather \textit{combining} them. Doing so allows the central server to make “negative bets against” an agent’s predictions. Even if an agent consistently performs poorly, its signed opposite prediction may still be useful. This added freedom enables the server to better exploit available information across all agents.

In essence, the hyperparameter $\kappa$ controls the amount we allow the server to borrow its investment level in any given expert to leverage it towards the predictive influence of any other expert, and $\eta$ controls the total investment level of the server into each agent.
\begin{theorem}[Server: Optimal Mixture Weights]
\label{thm:OptimalWeights}
Let $\kappa,\eta>0$, $t\in \mathbb{N}_+$, $\left(\hat{Y}_{t}^{j}\right)_{j=1}^N\in L^1(\mathbb{R}^{d_y})^{N}$, $y_{[0:t]}\in \mathbb{R}^{d_y(1+t)}$, $x_{[0:t]}\in \mathbb{R}^{d_x(1+t)}$, and $w_{[0:t]}\in \mathbb{R}^{N(1+t)}$.  
Then
\[
        w_t^{\star}  
    = 
         A^{-1}\,
         \biggl(
             b -  
                \frac{
                    \mathbf{1}_N^{\top}A^{-1}\,b-
                    \eta
                }{
                    \mathbf{1}_N^{\top}A^{-1}\mathbf{1}_N
                }
            \cdot
             \,\mathbf{1}_N
         \biggr)
\]
is the \textit{unique} solution $w^{\star}_t=[ w^1_t, \dots, w^N_t]^\top$ 
to System \eqref{eq:server}, where $A\eqdef 2 ( \hat{\mathbf{Y}}^\top_{t} \hat{\mathbf{Y}}_{t} + \kappa I_N )$, $b\eqdef 2 (y^\top_{t}  \hat{\mathbf{Y}}_{t})^{\top}$, $\hat{\mathbf{Y}}_{t} = [ \hat{Y}^1_{t}, \dots, \hat{Y}^N_{t} ]$.
\end{theorem}

Having settled the optimal server-side behaviour in Theorem~\ref{thm:OptimalWeights}, we now turn our attention to the optimal agent/client-side behaviour given the weights produced by the server.
\subsection{Optimal Client-Side Behaviour} \label{nash_sync}

Since each agent is unaware of the internal parameters and models of other agents, the best they can do is optimize their contribution to the ensemble mixture, $\sum_{j=1}^N w_t^j \hat{Y}_{t+1}^{j}$, given the mixture weight assigned by the central server and the observable predictions of other agents over a look-back window of length $T \in \mathbb{N}_+$. An exponentially weighted average, where recent predictive accuracy is deemed more important than historical predictions up to pre-specified weights $\alpha_i$ for $i=1,\dots,N$, is used. 
For notational simplicity, we denote $t=0$ to be the start of the look-back window; that is, it serves as a stand-in for $s=(t-T)\vee 0$.  
The objective function is written 
\begin{align} 
\label{eq:Objective_nth_expert}
 J_i(\beta^i, \beta^{-i}) = 
 \mathbb{E} \Bigg\{ \sum_{t=0}^{T-1} 
    e^{ -\alpha_i (T - 1 - t)} 
 \,
  \Big[    
 \Big\| y_{t+1} - \sum_{j=1}^N w^j_t \hat{Y}^j_{t+1} \Big\|^2 
  + \gamma_i \| \beta^i_t \|^2    
  \Big] 
  \Bigg\},
\end{align} 
for $i = 1, \dots, N$, $\beta^{-i}=
(\beta^j)_{j=1; j\neq i}^N$ and $\alpha_i\ge 0$.

Thus, each agent/client is naturally in \textit{competition} with all other agents to maximize their influence on the mixture $\sum_{j=1}^N\, w_t^j\hat{Y}_{t+1}^{j}$.  Thus, the \textit{optimal} behaviour of each agent is that of a \textit{Nash equilibrium}; that is, the action of each agent at any time cannot reduce the value of their running cost function $J_i$ in Equation \eqref{eq:Objective_nth_expert}. The following is our main theoretical contribution, which identifies and computes the \textit{unique} Nash equilibrium behaviour for all agents in \textit{closed-form}.

\begin{theorem}[Agents: Optimal Re-synchronization via The Nash Equilibrium]
\label{thm:NashEqmExistence}
Let $T, t \in \mathbb{N}_+$ with $t\ge T$, $\left(\hat{Y}_{[0:t]}^{j}\right)_{j=1}^N\in L^1(\mathbb{R}^{d_y})^{N(1+t)}$, $y_{[0:t]}\in \mathbb{R}^{d_y(1+t)}$, $x_{[0:t]}\in \mathbb{R}^{d_x(1+t)}$, and $w_{[0:t]}\in \mathbb{R}^{N(1+t)}$ with $\sum_{i=1}^N w^{i}_t =\eta$ for each $t\in [0:t]$. 
Suppose the system in Equation \eqref{eqn:Pi} for $( P_i )_{i=1}^N$ admits a solution for $t=0$, $1$, $\dots$, $T$.  
Then the $N$-agent Nash game with cost functions defined in Equation \eqref{eq:Objective_nth_expert} admits a unique feedback Nash equilibrium for $t=0$, $1$, $\dots$, $T-1$
\begin{align} 
\label{eqmvecbeta}
 \boldsymbol{\beta}_{t} =  [\beta^{1\top}_t, \dots, \beta^{N \top}_t ]^\top  
=  \mathbf{G}(t)  [ \hat{Y}^{1\top}_t, \dots, \hat{Y}^{N\top}_t ]^\top + \mathbf{H}(t) , 
\end{align} 
where the matrix-valued coefficients $\mathbf{G}(t)$ and $\mathbf{H}(t)$ are determined by the following system of matrix-valued equations for $(P_i, S_i)_{i=1}^N$ on the time span 
$t=0$, $1$ $\dots$, $T$, 
\begin{equation} 
\begin{split} 
\left \{ 
\begin{aligned}
  P_i(t)  = &   
  \mathbf{G}(t)^\top 
  \big[ \widehat{\mathbf{A}}(t) + \mathbf{D}_i(t) + e^{-\alpha_i(T-1-t)} 
  \gamma_i \widehat{\mathbf{e}}_i \widehat{\mathbf{e}}_i^\top  \big] 
   \mathbf{G}(t)  
   \\ 
   & + \big[ e^{-\alpha_i(T-1-t)}\mathbf{w}_t \mathbf{w}_t^\top + P_i(t+1) \big] 
   \mathbf{D}(t) \mathbf{G}(t) \\ 
   & + \mathbf{G}(t)^\top \mathbf{D}(t)^\top \big[ e^{-\alpha_i(T-1-t)} \mathbf{w}_t \mathbf{w}_t^\top + P_i(t+1) \big]^\top  \\ 
   & + e^{-\alpha_i(T-1-t)} \mathbf{w}_t \mathbf{w}_t^\top + P_i(t+1)  , \\ 
  & \hspace{-0.6cm} e^{ - \boldsymbol{\alpha}(T-1-t)} \Gamma + \mathbf{A}(t) + e^{ - \boldsymbol{\alpha}(T-1-t)} \widehat{\mathbf{A}}(t) 
   \, \text{is invertible} , 
  \\ 
  P_i(T)   = &    0  ,
\end{aligned}
\right.
\end{split}
\label{eqn:Pi}
\end{equation} 
\begin{equation}  
\begin{split} 
\left \{ 
\begin{aligned}
  S_i(t)  = &  \mathbf{G}^\top(t) \big[ \widehat{\mathbf{A}}(t) + \mathbf{D}_i(t) 
   + e^{-\alpha_i(T-1-t)} \gamma_i \widehat{\mathbf{e}}_i \widehat{\mathbf{e}}_i^\top  \big] \mathbf{H}(t) \\ 
& + \big[ e^{-\alpha_i(T-1-t)} \mathbf{w}_t \mathbf{w}_t^\top + P_i(t+1)  \big] \mathbf{D}(t) \mathbf{H}(t) \\  
& + \mathbf{G}^\top(t) \mathbf{D}^\top(t) \big[ - \mathbf{w}_t y_{t+1} e^{-\alpha_i(T-1-t)} + S_i(t+1) \big] \\ 
& - \mathbf{w}_t y_{t+1} e^{-\alpha_i(T-1-t)} + S_i(t+1) , \\ 
 S_i(T) = & 0   . 
\end{aligned}
\right.
\end{split} 
\label{eqn:Si}
\end{equation} 
The matrices in Equations \eqref{eqn:Pi} and \eqref{eqn:Si} are defined in Table \ref{tab:matricesEqnPiSi}.  
\end{theorem} 
\begin{table}[H] 
\caption{The matrices defining the closed-form updates in Theorem \ref{thm:NashEqmExistence}.}
\label{tab:matricesEqnPiSi}
\begin{tabular}{ll}
\toprule
\textbf{Matrix} & \textbf{Closed-Form Expression}
\\
\midrule
$\mathbf{A}(t)$ & $\mathbb{E} \big\{ \big[ P_1(t+1) \mathbf{e}_1 Z^1_t, \dots, P_N(t+1) \mathbf{e}_N Z^N_t \big]^\top \big[ \mathbf{e}_1 Z^1_t, \dots,  \mathbf{e}_N Z^N_t  \big]  \big\}$                                                             
\\
$\widehat{\mathbf{A}}(t)$ & $\mathbb{E} \big\{ \big[ \mathbf{e}_1 Z^1_t, \dots, \mathbf{e}_N Z^N_t \big]^\top \mathbf{w}_t \mathbf{w}_t^\top \big[ \mathbf{e}_1 Z^1_t, \dots,  \mathbf{e}_N Z^N_t  \big]  \big\}$                               \\
$\mathbf{B}(t)$ & $ \mathbb{E} \big[ P_1(t+1) \mathbf{e}_1 Z^1_t, \dots, P_N(t+1) \mathbf{e}_N Z^N_t \big]^\top$                    
\\
$\mathbf{C}(t) $ & $\mathbb{E} \big[ S_1^\top(t+1) \mathbf{e}_1 Z^1_t, \dots, S_N^\top(t+1) \mathbf{e}_N Z^N_t \big]^\top$                                                                                                                                              
\\
$\mathbf{D}(t)$ & $ \mathbb{E} \big[ \mathbf{e}_1 Z^1_t , \, \mathbf{e}_2 Z^2_t ,  \, \dots , \, \mathbf{e}_N Z^N_t  \big]$                                                                                                                                          
\\
$\mathbf{D}_i(t) $ & $\mathbb{E} \big\{  \big[ \mathbf{e}_1 Z^1_t  ,  \,\dots , \, \mathbf{e}_N Z^N_t    \big]^\top P_i(t+1) \big[ \mathbf{e}_1 Z^1_t, \,\dots , \, \mathbf{e}_N Z^N_t \big]  \big\} , \quad i = 1, \dots, N,$                                            
\\
$\mathbf{G}(t)$ & $ - \big[ e^{-\boldsymbol{\alpha}(T-1-t)} \Gamma   + \mathbf{A}(t) + e^{-\boldsymbol{\alpha}(T-1-t)} \widehat{\mathbf{A}}(t)  \big]^{-1} \big[ e^{-\boldsymbol{\alpha}(T-1-t)} \mathbf{D}^\top(t) \mathbf{w}_t \mathbf{w}_t^\top + \mathbf{B}(t) \big]$ 
\\
$\mathbf{H}(t) $ & $ \big[ e^{-\boldsymbol{\alpha}(T-1-t)} \Gamma + \mathbf{A}(t) + e^{-\boldsymbol{\alpha}(T-1-t)} \widehat{\mathbf{A}}(t)  \big]^{-1} \big[ e^{-\boldsymbol{\alpha}(T-1-t)} \mathbf{D}(t)^\top \mathbf{w}_t y_{t+1} - \mathbf{C}(t) \big]$               
\\ 
\bottomrule
\end{tabular}

Here $e_i$ is the $i^{\text{th}}$ standard basis vector in $\mathbb{R}^N$,  $\mathbf{e}_i = e_i \otimes I_{d_y} = [0, \dots, I_{d_y} , \dots, 0]^\top , \quad \widehat{\mathbf{e}}_i = e_i \otimes I_{d_z} = [0, \dots, I_{d_z} , \dots, 0]^\top$, and 
$\mathbf{w}_t=[ w^1_t I_{d_{y}}, \dots, w^N_t I_{d_{y}} ]^\top$,  $\Gamma = \diag[ \gamma_1 I_{d_z} , \dots, \gamma_N I_{d_z}]$, and $\boldsymbol{\alpha} = \diag [\alpha_1 I_{d_z} , \dots, \alpha_N I_{d_z} ]$; where $i=1,\dots,N$.
\end{table}

In the most general form of our main result, as stated in Theorem \ref{thm:NashEqmExistence}, the computational bottleneck in the FL algorithm lies in computing the matrices described therein. However, explicit closed-form expressions are derived for deterministic (non-generative) feature encoders in Section \ref{s:Analysis__ss:PreTrained} and for RFNs in Section \ref{s:Analysis__ss:ELN}.  
The algorithmic formulation of Theorem \ref{thm:NashEqmExistence} is recorded in Algorithm \ref{alg:sync_subroutine}. Before detailing these simplified cases, we first illustrate a variant of our procedure where experts are temporarily restricted in the information governing their actions, leading to sub-optimal behaviour. This relaxation enables a more rapid computational throughput of their optimal actions at the cost of a controllable degree of sub-optimal behaviour.

\subsubsection{Speedy and Greedy: Client Acceleration via Infrequent Games} \label{greedy_algorithm}
{Though the optimal behaviour of each agent, in the proprietary setting, is now known explicitly in closed-form due to Theorem \ref{thm:NashEqmExistence}, the optimal strategy does involve several large matrix computations which can be computationally intensive.  Therefore, one can consider a more efficient \textit{variant} of the client-side procedure, which only computes the optimal updates in Equation \eqref{eqmvecbeta} periodically with a frequency of $\tau\in \mathbb{N}_+$.\footnote{Note when $\tau=1$ then both this section and the previous sections' strategies are identical.} We call these times $\{\tau t\}_{t=0}^{\infty}$ \textit{synchronization times} where the system of proprietary agents synchronize.  This is because the agents act greedily, and thus Nash-sub-optimally, for all other times, then their behaviour is \textit{de-synchronizing} for the optimal behaviour given by the Nash equilibrium computed in Theorem \ref{thm:NashEqmExistence}.

We consider the following \textit{greedy} objective for each agent between synchronization times.  We emphasize that this greedy objective ignores the prediction of all other agents and the weights issued/used by the central server. For each $t\in \mathbb{N}_+$ not divisible by $\tau$, the $i^{\text{th}}$ expert's greedy objective is defined to be
\begin{align} 
\label{eq:ridge_local}
        \tilde{J}_i(\beta_t^{i})
    \eqdef 
        \sum_{s=(t-T)\vee 0}^{t-1}\,
        e^{-\alpha_i(t-1-s)} 
        \big\| y_{s+1} - ( \hat{Y}_{s}^i 
        + Z_{s}^i \beta^i_t ) \big\|^2
        +
        \gamma_i\big\|\beta_t^i\big\|^2  , 
\end{align} 
where the parameters $\alpha_i$ and $\gamma_i$ are the same as Equation \eqref{eq:Objective_nth_expert}. 
This allows for rapid and fully parallelized online optimization of each model but has the drawback that the ensemble may deviate from its optimum. Note that the $T$ in Equation \eqref{eq:ridge_local} need not be the same as that in Equation \eqref{eq:Objective_nth_expert}. When required, the former is referred to as the Client $T$ while the later is the Game $T$.  
The following is a variant of the well-known optimal solution to the exponentially-weighted ridge-regression problem of Equation~\eqref{eq:ridge_local}.

\begin{proposition}[Agents: Optimal Linear Decoder With No Communication] 
\label{prop:localgreedyupdates} 
The unique minimum for Objective \eqref{eq:ridge_local} is
\begin{align} 
   \beta_t^i =
        \big(
            X^{i \top}_{t-1} X^i_{t-1} + \gamma_i I_{d_z}
        \big)^{-1}\, X_{t-1}^{i \top} \, \bar{y}^i_t  ,   \notag 
\end{align} 
where
\begin{align} 
& X^i_{t-1}  = \big[ e^{ - \alpha_i (t-1-0)/2} Z^{i\top}_0 , \ldots, e^{-\alpha_i (t-1-s)/2} Z^{i \top }_s, \ldots,  Z^{i \top }_{t-1} \big]^\top ,  \notag  \\ 
& \bar{y}^i_t  = \big[ e^{-\alpha_i (t-1-0)/2} (y_1 - \hat{Y}^i_{0} )^\top, \ldots, e^{-\alpha_i (t-1-s)/2} (y_{s+1} - \hat{Y}^i_{s} )^\top, \ldots,  (y_{t} - \hat{Y}^i_{t-1} )^\top \big]^\top . \notag  
\end{align} 
\end{proposition}
}

\subsection{Algorithmic Compute, Communication, and Privacy Trade-offs}

Two computational, and subtly different, implementations of the proposed algorithm are possible, both of which minimize the agent objectives in Equation \eqref {eq:Objective_nth_expert} using the iterative method from Theorem \ref{thm:NashEqmExistence}. In the first, detailed in Algorithm \ref{alg:sync_subroutine}, each agent computes $\mathbf{D}_i(t)$, $P_i(t)$, and $S_i(t)$ in parallel. Alternatively, all computations can be centralized on the server. While both are mathematically equivalent, they differ regarding computational burden, communication overhead, and privacy implications.

Computing the Nash equilibrium is resource-intensive. If the server is compute-constrained, one can parallelize the construction of $\mathbf{D}_i(t)$, $P_i(t)$, and $S_i(t)$ across agents to alleviate the bottleneck. However, this introduces trade-offs. Each agent must access $\{Z_t^i\}_{i=1}^N$, which exposes potentially proprietary model activations to other participants in the network. Moreover, while all required history $\{Z_t^i\}_{i=1}^N$ may be gather and shared once, quantities like $\mathbf{A}(t)$ and $\mathbf{B}(t)$ depend on $\{P_i(t+1)\}_{i=1}^N$ and must be updated and communicated at each time step -- incurring additional communication costs and privacy risks.

Alternatively, if each agent transmits all required $\{Z_t^i\}_{i=1}^N$, either all at once or alongside predictions, to the server. Then all information needed to perform the iterative updates of Theorem \ref{thm:NashEqmExistence} exists at the server's disposal. This one-shot communication significantly reduces overhead and enhances privacy, as agents do not access potentially sensitive activations from others. Such concealment is standard, for example, in commercial LLM APIs to prevent model inversion or theft \citep{Oliynyk1}. The trade-off is that the server must now centrally perform the full linear algebra computations, which can be resource intensive. This centralized setup is used in the experiments below.

\section{Use-Cases: From Pre-Trained Encoders to Echo-State Agents}
\label{s:Examples}

This section puts our theory into practice for various AI agents ranging from ``static'' kernel regressors to sequential deep learning models with internal latent ``memory states.'' In each case, we derive simplified and specialized expressions of the main result (Theorem \ref{thm:NashEqmExistence}), yielding streamlined algorithms.  We consider three cases: pre-trained feature encoders, RFNs, and ESNs. In each case, the proposed theoretical PFL algorithm is tested across multiple datasets, validating its reliability on synthetic and real-world data.

\subsection{Experiment Details} 

Before detailing each case, the experiments used to evaluate the quality of the proposed PFL algorithm are outlined. We consider five challenging time-series datasets. Three are synthetic and and two are drawn from real-world benchmarks. Each is discussed in detail in Appendix \ref{s:ExperimentDetials__ss::datasets}. The experiments involve multiple agents and a server, with predictions made using two or five collaborating agents. Periodic synchronization via the Nash game, defined in Section \ref{nash_sync}, is compared with purely greedy local optimization, detailed in Section \ref{greedy_algorithm}, forgoing any Nash computations. These two approaches are often referred to throughout as the game and non-game settings, respectively. Experimental results for the distinct model varieties are reported in Sections \ref{s:Analysis__ss:PreTrained}--\ref{s:Examples_ReservoirComputer}, with visualizations focusing on the five-client settings. These sections also derive how the main algorithm (Theorems \ref{thm:OptimalWeights} and \ref{thm:NashEqmExistence}) simplifies for each model. A model's predictive performance is primarily quantified by its predictive mean-squared error (MSE), with optimal hyperparameters reported for both Nash and non-Nash settings.  

Various other quantitative tools and qualitative visualizations are also incorporated to provide a complete perspective on each model and algorithm's performance.  The first set of plots overlay the predictions made by the server in both the game and non-game settings with the ground truth signal for comparison. The second set of figures illustrates the distribution of relative squared errors between the game and non-game predictions at every time step,
\begin{align*}
\frac{(\hat{Y}^n_{i, t} - y_{i, t})^2}{(\hat{Y}^g_{i, t} - y_{i, t})^2},
\end{align*}
where $y_{i, t}$ is the $i^{\text{th}}$ dimension of the target variable at time step $t$. The corresponding predictions with and without the Nash-game are denoted $\hat{Y}^g_{i, t}$ and $\hat{Y}^n_{i, t}$, respectively. If the value of the ratio at a given time step is greater than one, the server's prediction error in the Nash setting is lower than that of the non-Nash setting.

Thereafter, in Section \ref{s:Experiments__ss:Ablation}, three ablation studies are presented. First, the impact of the Nash game on a server's mixture weights over time is illustrated, highlighting its positive influence towards producing higher quality mixtures. In the second study, the impact of varying the look-back length on Nash game performance is analyzed, offering deeper insights into its internal mechanisms. Finally, the Nash game's performance under different synchronization frequencies is considered, along with runtime overhead measurements, including the total runtime incurred by playing the Nash game each round for all agent types across two datasets. For additional experimental details, hyperparameters, and further ablations, such as motivation for using Hard Sigmoid activations with ESNs, see Appendices \ref{s:ExperimentDetials__ss::datasets}--\ref{a:Ablations}.\footnote{All experimental code is found at: \url{https://github.com/fatemetkl/FedMOE}.}

\subsection{Experimental Results Summary}

The main results are reported in Table \ref{tab:summary_results}. For all datasets and all models, application of Nash-game synchronization improves MSE, sometimes by an order of magnitude. The Nash game is especially helpful for the Concept Shift, BoC Exchange Rate, and ETT time-series. For both ESN and RFN models, there is also typically an improvement in prediction accuracy, both with and without the Nash game, with more experts contributing to the mixture. Such improvement is largely expected, especially since these models incorporate randomness in their feature representations.

\begin{table}[ht]
\centering
\caption{Summary of experimental results for feature encoder, RFN, and ESN models across datasets. The reported MSE is the mean of three separate runs with different random seeds for ESNs and RFNs. For the pre-trained encoder models, there are no fluctuations due to randomness.}
\label{tab:summary_results}
\resizebox{0.9\linewidth}{!}{\begin{tabular}{cclccc}
\toprule
Model & \# of Experts & Dataset & Nash game & No Nash game \\
\midrule
\multirow{10}{*}{Transformer} & \multirow{5}{*}{2} & Periodic Signal & $\mathbf{1.78749\mathrm{e}{\text{-}1}}$ & $5.86464\mathrm{e}{\text{-}1}$ \\
& & Logistic Map & $\mathbf{4.53782\mathrm{e}{\text{-}2}}$ & $5.46232\mathrm{e}{\text{-}2}$ \\
& & Concept Shift & $\mathbf{6.45116\mathrm{e}{\text{-}2}}$ & $6.77600\mathrm{e}{\text{-}1}$ \\
& & BoC Exchange Rate & $\mathbf{4.48450\mathrm{e}{\text{-}5}}$ & $8.06623\mathrm{e}{\text{-}5}$ \\
& & ETT Series & $\mathbf{7.50786\mathrm{e}{\text{-}4}}$ & $1.05523\mathrm{e}{\text{-}3}$ \\
\cmidrule{2-6}
& \multirow{5}{*}{5} & Periodic Signal & $\mathbf{6.95994\mathrm{e}{\text{-}2}}$ & $5.86337\mathrm{e}{\text{-}1}$ \\
& & Logistic Map & $\mathbf{3.88076\mathrm{e}{\text{-}2}}$ & $6.99996\mathrm{e}{\text{-}2}$ \\
& & Concept Shift & $\mathbf{5.56220\mathrm{e}{\text{-}2}}$ & $3.25625\mathrm{e}{\text{-}1}$ \\
& & BoC Exchange Rate & $\mathbf{4.48687\mathrm{e}{\text{-}5}}$ & $8.35143\mathrm{e}{\text{-}5}$ \\
& & ETT Series & $\mathbf{7.27732\mathrm{e}{\text{-}4}}$ & $1.26992\mathrm{e}{\text{-}3}$ \\
\midrule
\multirow{10}{*}{RFN} & \multirow{5}{*}{2} & Periodic Signal & $\mathbf{2.56902\mathrm{e}{\text{-}1}}$ & $6.28833\mathrm{e}{\text{-}1}$ \\
& & Logistic Map & $\mathbf{4.97284\mathrm{e}{\text{-}2}}$ & $5.54874\mathrm{e}{\text{-}2}$ \\
& & Concept Shift & $\mathbf{7.85654\mathrm{e}{\text{-}3}}$ & $9.45541\mathrm{e}{\text{-}0}$ \\
& & BoC Exchange Rate & $\mathbf{7.86047\mathrm{e}{\text{-}5}}$ & $2.10533\mathrm{e}{\text{-}3}$ \\
& & ETT Series & $\mathbf{1.19513\mathrm{e}{\text{-}3}}$ & $3.51060\mathrm{e}{\text{-}3}$ \\
\cmidrule{2-6}
& \multirow{5}{*}{5} & Periodic Signal & $\mathbf{ 5.96514\mathrm{e}{\text{-}2}}$ & $6.15566\mathrm{e}{\text{-}1}$ \\
& & Logistic Map & $\mathbf{4.64415\mathrm{e}{\text{-}2}}$ & $5.50781\mathrm{e}{\text{-}2}$ \\
& & Concept Shift & $\mathbf{1.01067\mathrm{e}{\text{-}2}}$ & $1.28812\mathrm{e}{\text{-}0}$ \\
& & BoC Exchange Rate & $\mathbf{7.65036\mathrm{e}{\text{-}5}}$ & $3.15237\mathrm{e}{\text{-}3}$ \\
& & ETT Series & $\mathbf{7.11836\mathrm{e}{\text{-}4}}$ & $2.26659\mathrm{e}{\text{-}3}$ \\
\midrule
\multirow{10}{*}{ESN} & \multirow{5}{*}{2} & Periodic Signal & $\mathbf{1.02870\mathrm{e}{\text{-}1}}$ & $5.62155\mathrm{e}{\text{-}1}$ \\
& & Logistic Map & $\mathbf{4.71543\mathrm{e}{\text{-}2}}$ & $5.30976\mathrm{e}{\text{-}2}$ \\
& & Concept Shift & $\mathbf{1.30703\mathrm{e}{\text{-}1}}$ & $1.98226\mathrm{e}{\text{-}0}$ \\
& & BoC Exchange Rate & $\mathbf{5.15198\mathrm{e}{\text{-}5}}$ & $2.18067\mathrm{e}{\text{-}4}$ \\
& & ETT Series & $\mathbf{7.13503\mathrm{e}{\text{-}4}}$ & $1.11331\mathrm{e}{\text{-}3}$ \\
\cmidrule{2-6}
& \multirow{5}{*}{5} & Periodic Signal & $\mathbf{3.22272\mathrm{e}{\text{-}2}}$ & $5.57315\mathrm{e}{\text{-}1}$  \\
& & Logistic Map & $\mathbf{4.65061\mathrm{e}{\text{-}2}}$ & $5.28456\mathrm{e}{\text{-}2}$ \\
& & Concept Shift & $\mathbf{8.86346\mathrm{e}{\text{-}2}}$ & $4.85359\mathrm{e}{\text{-}0}$ \\
& & BoC Exchange Rate & $\mathbf{1.15900\mathrm{e}{\text{-}4}}$ & $ 4.08432\mathrm{e}{\text{-}4}$ \\
& & ETT Series & $\mathbf{7.15441\mathrm{e}{\text{-}4}}$ & $1.08915\mathrm{e}{\text{-}3}$ & \\
\bottomrule
\end{tabular}}
\end{table}
\subsection{{Pre-Trained Transformer with Fine-Tunable Linear Decoder}}
\label{s:Analysis__ss:PreTrained}

We first consider the setting in which each (deep) feature encoder $\varphi^{i}$ in Equation \eqref{eq:ENCODER} is pre-trained and deterministic. This is likely the most common case in our setting, occurring when $\varphi^{i}$ is, for instance, a pre-trained transformer or a foundation model such as  CLIP \citep{radford2021learning}, DINO(v2) \citep{caron2021emerging,oquab2024dinov2learningrobustvisual}, SLIP(v2) \citep{mu2022slip,tschannen2025siglip}, among many others. Beyond its prevalence in modern deep learning, this setup offers a key advantage. The dynamically-evolving matrices in Theorem \ref{thm:NashEqmExistence} simplify dramatically, yielding surprisingly straightforward expressions. This is largely due to all expectations vanishing, as the encoder has no remaining internal sources of randomness.

\begin{corollary}[Pre-Trained Deterministic Feature Encoders]
\label{cor:PreTrainedTransformer}
In the setting of Theorem \ref{thm:NashEqmExistence}.  If additionally, for $i=1,\dots,N$, the map $\varphi^{i}:\mathbb{R}^{
d_x
\times 1}\to \mathbb{R}^{d_z\times 1}$ is of the form
\begin{equation*}
    \varphi^{i}(x_{[0:t]},Z_{t-1}^i, \epsilon_t^i ) \eqdef \phi^{i}(x_t).
\end{equation*}
{Then the matrix-valued processes in Theorem \ref{thm:NashEqmExistence} are}
\allowdisplaybreaks 
\begin{align} 
& \mathbf{A}(t) = 
\begin{bmatrix} 
\phi^{i}(x_t)^{\top} \mathbf{e}_i^\top P_i(t+1) \mathbf{e}_j\phi^{(j)}(x_t) 
\end{bmatrix}_{i,j=1}^N , 
\notag \\ 
& \widehat{\mathbf{A}}(t) = 
\begin{bmatrix} 
\phi^{i}(x_t)^{\top} \mathbf{e}_i^\top \mathbf{w}_t \mathbf{w}_t^\top  \mathbf{e}_j\phi^{(j)}(x_t) 
\end{bmatrix}_{i,j=1}^N , 
\notag \\ 
& \mathbf{B}(t) = 
 \begin{bmatrix} 
     P_1(t+1)  \mathbf{e}_1 \phi^{(1)}(x_t) , \dots,  P_N(t+1) \mathbf{e}_N \phi^{(N)}(x_t)
 \end{bmatrix}^\top  
 ,  
 \notag \\ 
 & \mathbf{C}(t)  = 
 \begin{bmatrix} 
    S_1(t+1)^\top  \mathbf{e}_1 \phi^{(1)}(x_t) , \dots, S_N(t+1)^\top  \mathbf{e}_N \phi^{(N)}(x_t)
 \end{bmatrix}^\top  
 , 
 \notag \\ 
& \mathbf{D}(t) 
 = 
 \begin{bmatrix} 
 \mathbf{e}_1 \phi^{(1)}(x_t) , \dots , \mathbf{e}_N \phi^{(N)}(x_t) 
 \end{bmatrix} , 
 \notag \\ 
 & \mathbf{D}_i(t) 
 = \big[ \phi^{(j)}(x_t)^\top \mathbf{e}_j^\top P_i(t+1) \mathbf{e}_k \phi^{(k)}(x_t) \big]_{j,k=1}^N , 
 \quad i = 1, \dots, N . 
 \notag 
\end{align} 
\end{corollary}

In this section, the deep encoder considered is a pre-trained transformer-based encoder model. Pre-training details are outlined in Appendix \ref{a:TransExpHyper}. The final linear layer constitutes the trainable $\beta$. In subsequent sections, sources of randomness are incrementally incorporated into the models.

\begin{figure}[ht!]
    \centering
    \begin{subfigure}[b]{0.99\linewidth}
    \centering
    \begin{minipage}[c]{0.05\linewidth}
    \caption{}
    \end{minipage}
    \begin{minipage}[c]{0.8\linewidth}
        \centering
        \includegraphics[width=\linewidth]{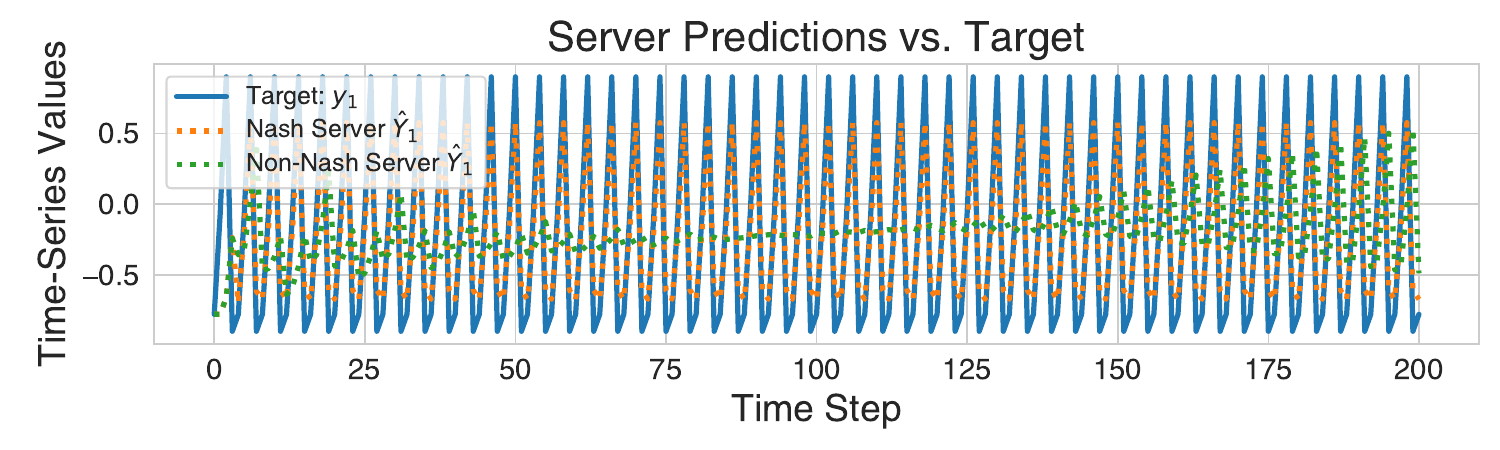}
    \end{minipage}
    \label{fig:transformer_periodic_overlaid_predictions}
    \end{subfigure}

    \begin{subfigure}[b]{0.99\linewidth}
    \centering
    \begin{minipage}[c]{0.05\linewidth}
    \caption{}
    \end{minipage}
    \begin{minipage}[c]{0.8\linewidth}
        \centering
        \includegraphics[width=\linewidth]{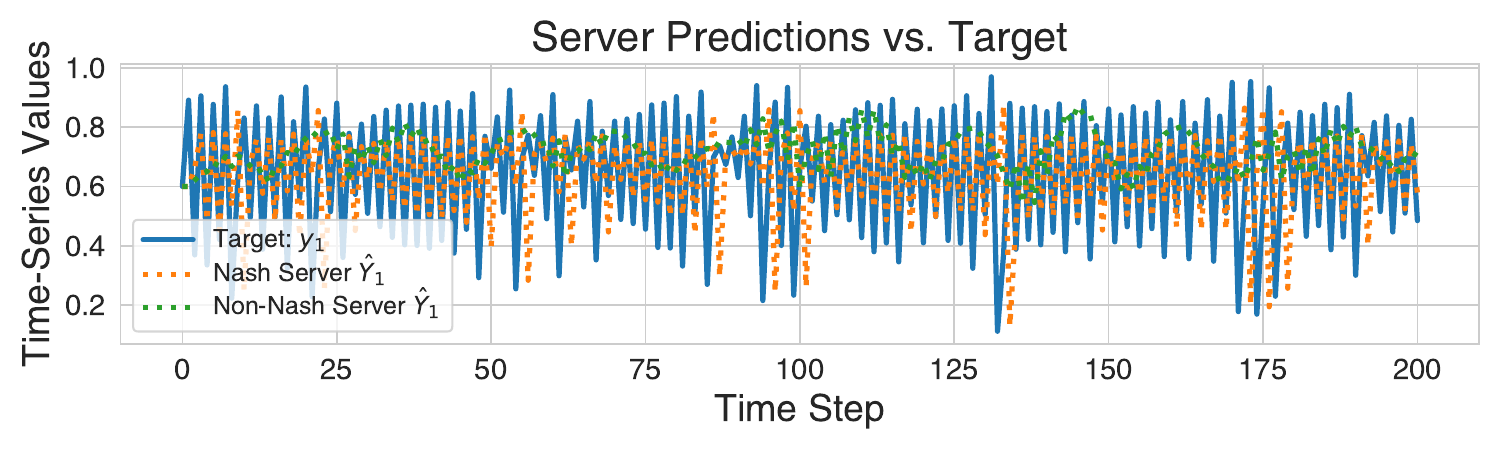}
    \end{minipage}
    \label{fig:transformer_logistic_overlaid_predictions}
    \end{subfigure}

    \begin{subfigure}[b]{0.99\linewidth}
    \centering
    \begin{minipage}[c]{0.05\linewidth}
    \caption{}
    \end{minipage}
    \begin{minipage}[c]{0.8\linewidth}
        \centering
        \includegraphics[width=\linewidth]{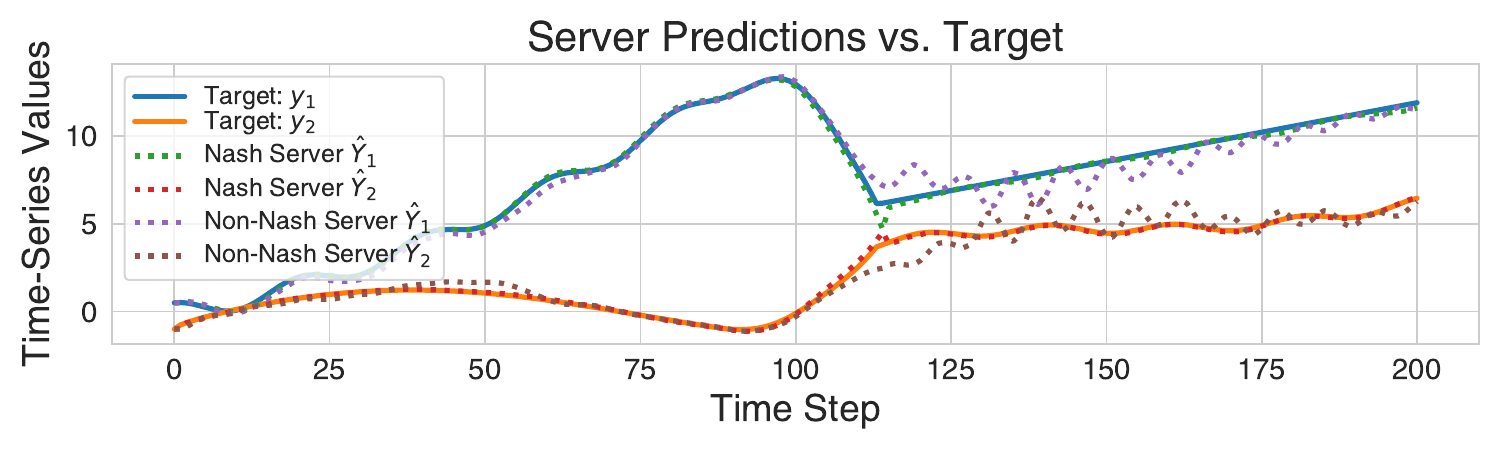}
    \end{minipage}
    \label{fig:transformer_concept_overlaid_predictions}
    \end{subfigure}
    
    \begin{subfigure}[b]{0.99\linewidth}
    \centering
    \begin{minipage}[c]{0.05\linewidth}
    \caption{}
    \end{minipage}
    \begin{minipage}[c]{0.8\linewidth}
        \centering
        \includegraphics[width=\linewidth]{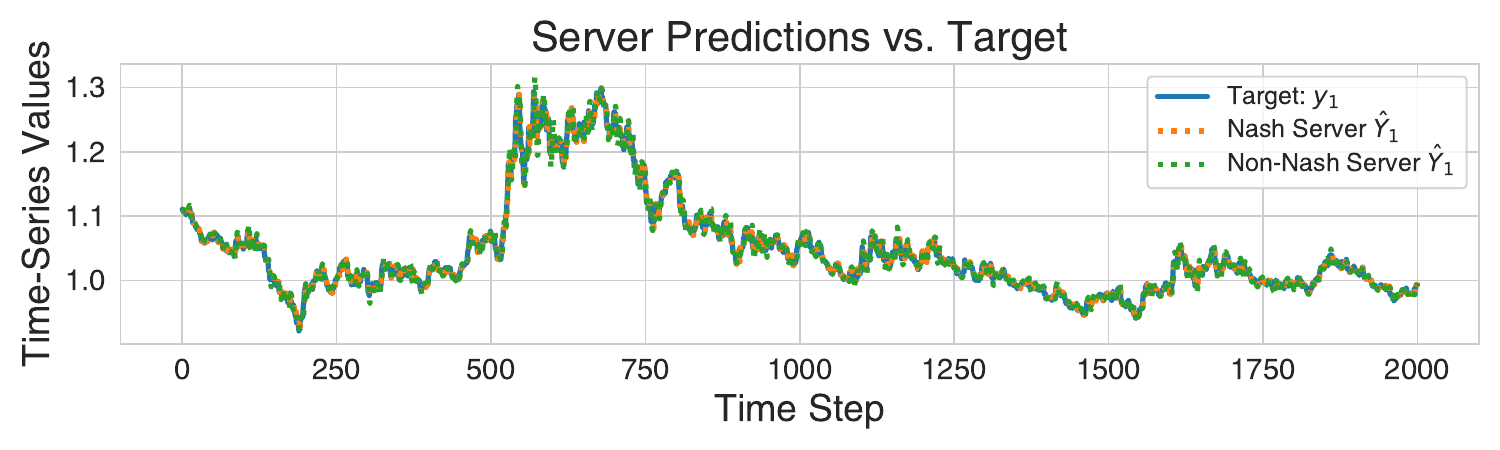}
    \end{minipage}
    \label{fig:transformer_boc_overlaid_predictions}
    \end{subfigure}

    \begin{subfigure}[b]{0.99\linewidth}
    \centering
    \begin{minipage}[c]{0.05\linewidth}
    \caption{}
    \end{minipage}
    \begin{minipage}[c]{0.8\linewidth}
        \centering
        \includegraphics[width=\linewidth]{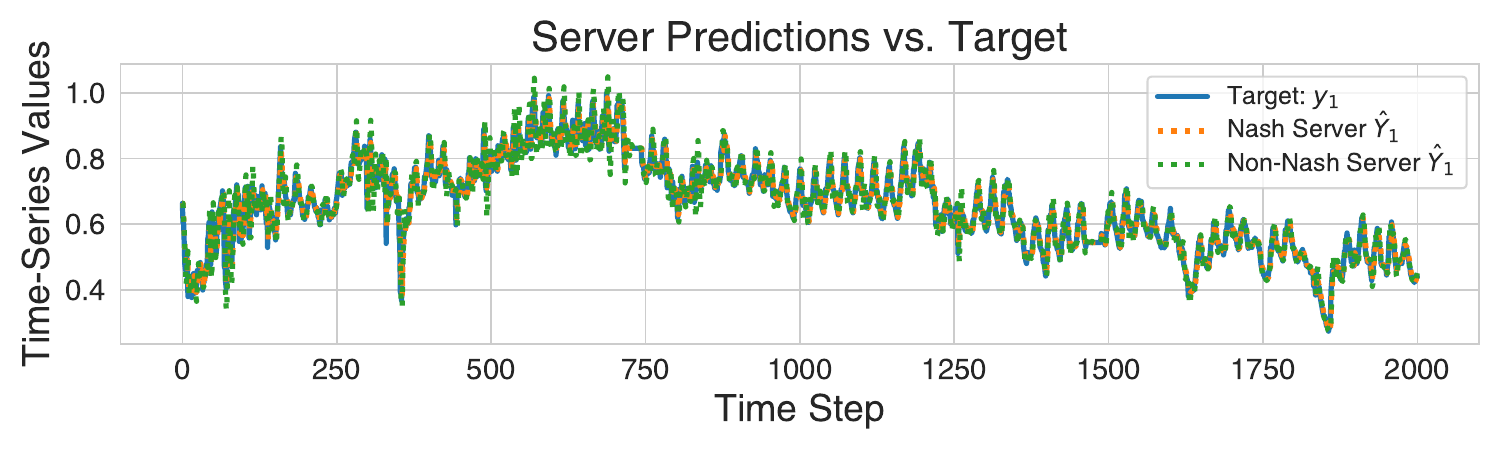}
    \end{minipage}
    \label{fig:transformer_ett_overlaid_predictions}
    \end{subfigure}
    \caption{Server predictions compared with ground truth targets for mixtures of encoder models with and without Nash-game synchronization for the Periodic (\subref{fig:transformer_periodic_overlaid_predictions}), Logistic (\subref{fig:transformer_logistic_overlaid_predictions}), Concept Drift (\subref{fig:transformer_concept_overlaid_predictions}), BoC Exchange Rates (\subref{fig:transformer_boc_overlaid_predictions}), and ETT (\subref{fig:transformer_ett_overlaid_predictions}) time series.}
    \label{fig:transfomer_overlaid}
\end{figure}

\subsubsection{Transformer Experimental Results}
\label{s:Examples__ss:PreTrained__sss:Results}

As seen in Figure \ref{fig:transfomer_overlaid}, predictions leveraging the Nash game produce significantly better modelling than those without it. Qualitatively, this improvement is especially notable in high-frequency segments of the time series and following sudden changes in the trajectory of the signals. Figure \ref{fig:transfomer_relative} shows the squared-error ratio of the non-game setting to that of the Nash game, as described in Section \ref{s:Examples}, for transformer agents over different signals. Points above the red line indicate time steps for which the Nash game is outperforming the baseline. A high density of such points above the red line reflects the performance advantages of Nash-game synchronization.

In the Periodic and Logistic time series, predictions produced by the transformer models without Nash synchronization tend towards the overall series mean. In contrast, models with Nash synchronization more effectively capture the amplitude of the underlying oscillations. While the non-game predictions for the Periodic signal may appear accurate at the beginning and end of the series, the ratio of squared errors in Figure \ref{fig:transfomer_relative}\subref{fig:transformer_periodic_squared_errors} shows that the error remains consistently and significantly higher than that of the Nash game predictions across the entire time span.
The Logistic Map dataset is quite challenging, especially with only a single lagged input from which to make predictions. This series poses difficulties for the transformer models in both settings due to its oscillatory and occasionally unstable nature. However, as shown in Figures \ref{fig:transfomer_overlaid}\subref{fig:transformer_logistic_overlaid_predictions} and \ref{fig:transfomer_relative}\subref{fig:transformer_logistic_squared_errors}, predictions with Nash-game synchronization capture the series' behavior significantly better. Interestingly, for some time steps, the non-game settings tendency to follow the mean can reduce the risk of sharp deviations that lead to larger errors in the game-based model during changes in dynamics.

The Concept Drift dataset incorporates multi-dimensional inputs and targets. In this series, a smooth shift in the relationship between the input and target variables occurs halfway through. As shown in Figure \ref{fig:transfomer_overlaid}\subref{fig:transformer_concept_overlaid_predictions}, the Nash game predictions more cleanly adapt to this change. In contrast, predictions without Nash synchronization struggle to adjust to the transition, resulting in a delayed alignment with the target signal and degraded fidelity. 

\begin{figure}[ht!]
    \centering
    \begin{subfigure}[b]{0.99\linewidth}
    \centering
    \begin{minipage}[c]{0.05\linewidth}
    \caption{}
    \end{minipage}
    \begin{minipage}[c]{0.7\linewidth}
        \centering
        \includegraphics[width=\linewidth]{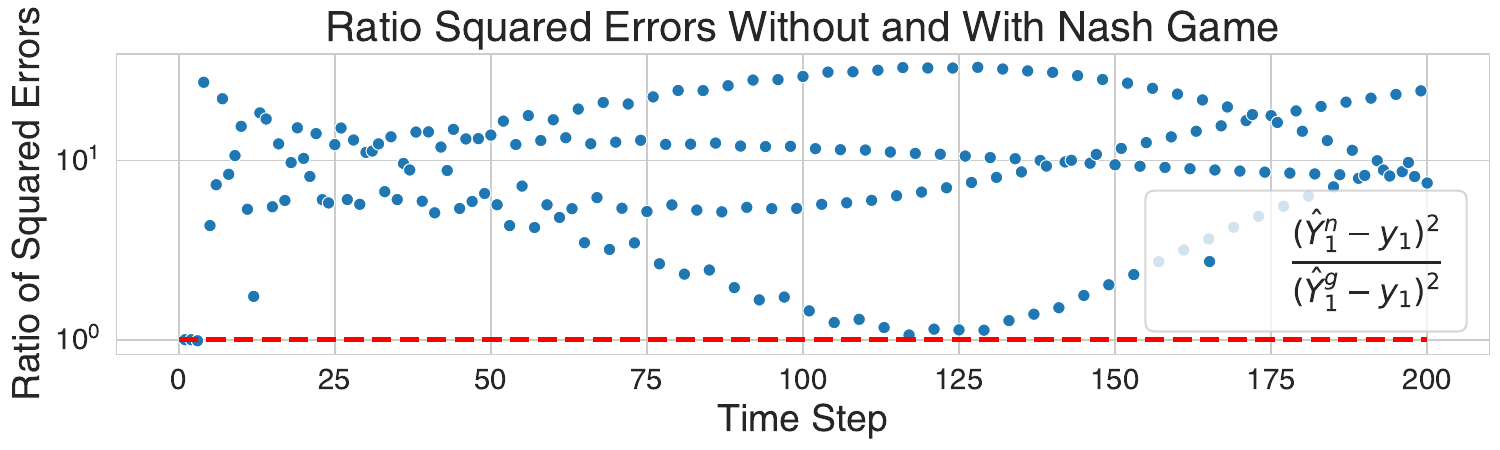}
    \end{minipage}
    \label{fig:transformer_periodic_squared_errors}
    \end{subfigure}

    \begin{subfigure}[b]{0.99\linewidth}
    \centering
    \begin{minipage}[c]{0.05\linewidth}
    \caption{}
    \end{minipage}
    \begin{minipage}[c]{0.7\linewidth}
        \centering
        \includegraphics[width=\linewidth]{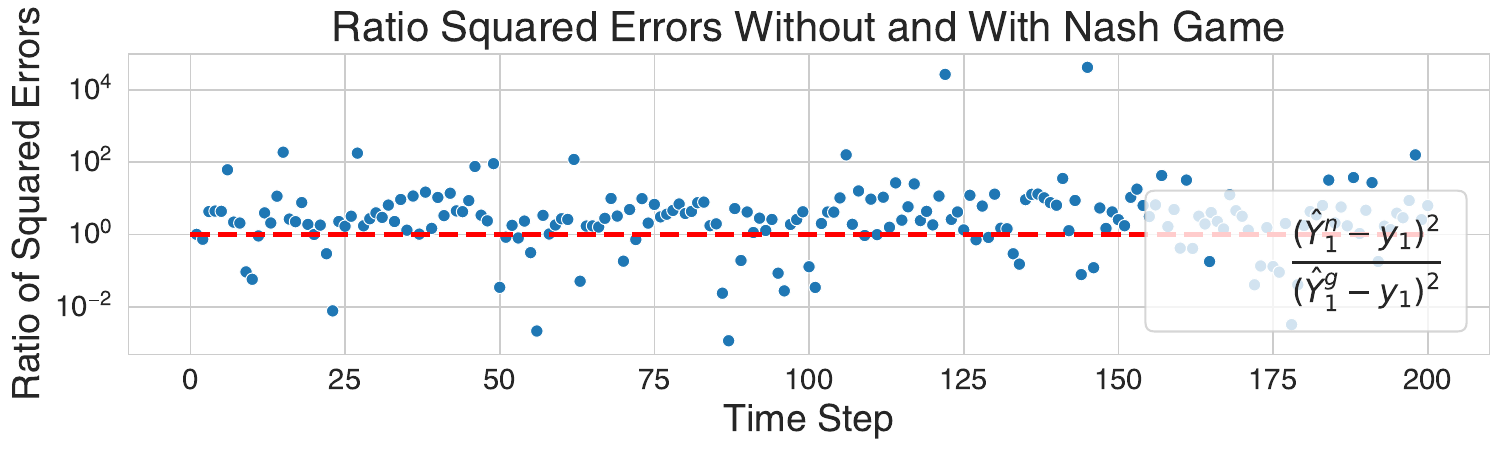}
    \end{minipage}
    \label{fig:transformer_logistic_squared_errors}
    \end{subfigure}
    
    \begin{subfigure}[b]{0.99\linewidth}
    \centering
    \begin{minipage}[c]{0.05\linewidth}
    \caption{}
    \end{minipage}
    \begin{minipage}[c]{0.7\linewidth}
        \centering
        \includegraphics[width=\linewidth]{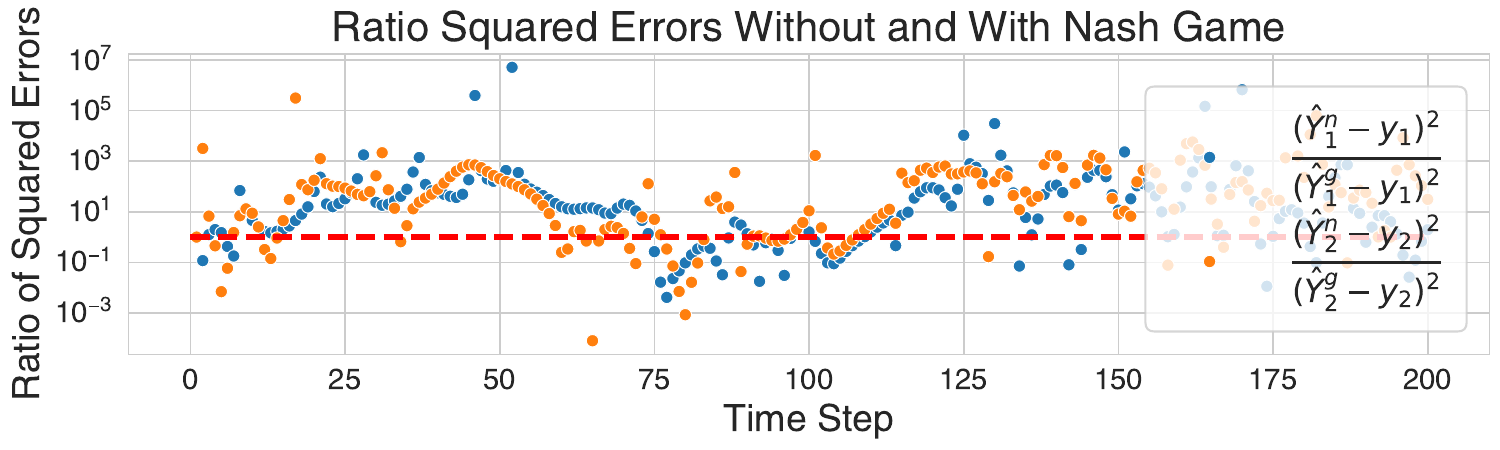}
    \end{minipage}
    \label{fig:transformer_concept_squared_errors}
    \end{subfigure}

    \begin{subfigure}[b]{0.99\linewidth}
    \centering
    \begin{minipage}[c]{0.05\linewidth}
    \caption{}
    \end{minipage}
    \begin{minipage}[c]{0.7\linewidth}
        \centering
        \includegraphics[width=\linewidth]{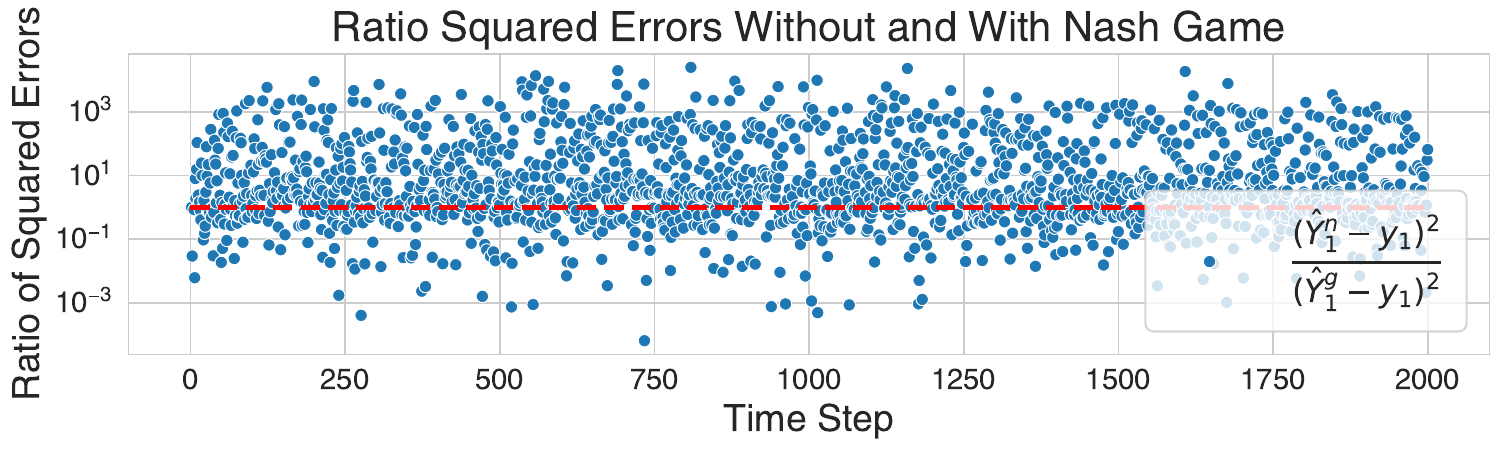}
    \end{minipage}
    \label{fig:transformer_boc_squared_errors}
    \end{subfigure}

    \begin{subfigure}[b]{0.99\linewidth}
    \centering
    \begin{minipage}[c]{0.05\linewidth}
    \caption{}
    \end{minipage}
    \begin{minipage}[c]{0.7\linewidth}
        \centering
        \includegraphics[width=\linewidth]{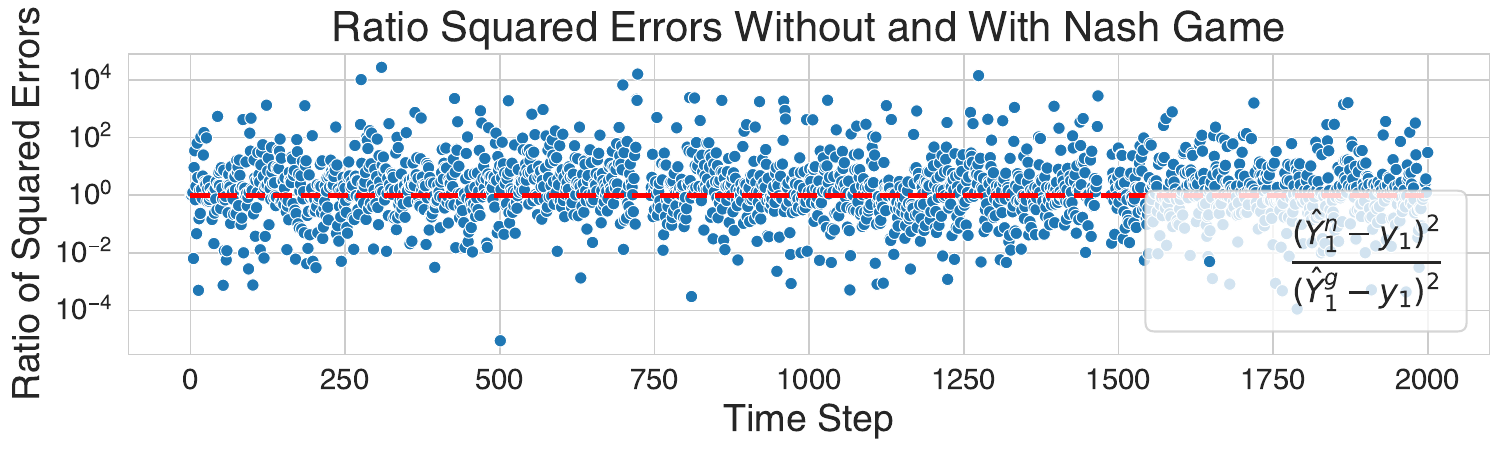}
    \end{minipage}
    \label{fig:transformer_ett_squared_errors}
    \end{subfigure} 

    \caption{Comparison of relative squared errors for encoder model predictions with and without Nash-game synchronization for the Periodic (\subref{fig:transformer_periodic_squared_errors}), Logistic (\subref{fig:transformer_logistic_squared_errors}), Concept Drift (\subref{fig:transformer_concept_squared_errors}), BoC Exchange Rates (\subref{fig:transformer_boc_squared_errors}), and ETT (\subref{fig:transformer_ett_squared_errors}) time series.}
    \label{fig:transfomer_relative}
\end{figure}

Predictions for the BoC Exchange Rate and ETT datasets are generally of higher quality likely due, in part, to presence of more informative features as detailed in Appendices \ref{s:ExperimentDetials__ss::datasets__sss:::BoC} and \ref{s:ExperimentDetials__ss::datasets__sss:::ETT}. They are also much longer time series than the preceding three datasets. The BoC series incorporates a number of lagged covariate exchange rates. Similarly, the ETT series leverages several informative input variables and lagged values of the target oil temperature. With these additional inputs, the server predictions of the Nash and non-Nash settings closely follow the targets for both datasets. However, as seen in Figures \ref{fig:transfomer_overlaid}\subref{fig:transformer_boc_overlaid_predictions} and \ref{fig:transfomer_overlaid}\subref{fig:transformer_ett_overlaid_predictions}, the non-game model predictions still fail to match the target signal in several parts of the time series. These shortcomings are effectively addressed by incorporating the Nash game. This improvement is clearly evident in Figures \ref{fig:transfomer_relative}\subref{fig:transformer_boc_squared_errors} and \ref{fig:transfomer_relative}\subref{fig:transformer_ett_squared_errors}, where the bulk of error ratios favor predictions using Nash synchronization.

\subsection{Generative Random Feature Networks (RFNs)}
\label{s:Analysis__ss:ELN}

In this section, the feature encoders of Equation \eqref{eq:ENCODER} are expanded to incorporate re-sampled randomness at every step. This places such encoders within the realm of extreme learning machines such as RFNs \citep{huang2006universal,rahimi2008weighted}, which have attracted significant attention recently due to their theoretical amenability; especially via NTK theory \citep{jacot2018neural,cheng2024characterizing}, or the wide neural network literature; see e.g. \citep{mei2022generalization,gonon2023approximation,dirksen2022separation,frei2023random,simchoni2023integrating,hanin2024random,parhi2025random,maillard2025injectivity}.  

Focusing on the case where the encoder is a random shallow ReLU neural network, we are able to obtain explicit expressions for all matrix-valued processes. Corollary~\ref{cor:Random_NeuralNetwork} presents the results for the case $d_y=1$, while the general case with $d_y>1$ is detailed in Corollary~\ref{cor:Random_NN_dy>1}, located in Appendix~\ref{s:ProofCorollaries}.
\begin{corollary}[Random Feature Network]
\label{cor:Random_NeuralNetwork}
Assume the setting of Theorem \ref{thm:NashEqmExistence} with $d_y=1$. For each $i=1,\dots,N$, fix random matrices, $A^{i} \in \mathbb{R}^{d_y \times d_x}$ and $b^{i} \in \mathbb{R}^{d_y \times d_z}$. Define
\begin{equation}
\hat{Y}^i_{t+1} = \hat{Y}^i_t +
            \big[ \operatorname{ReLU}\bullet\big(
                A^{i} [x_t, \dots, x_t ] + b^{i} + \sigma_t^i \, W_t^i
            \big) \big]  \beta^{i}_{t} , 
\notag 
\end{equation}
where $W_t^i\sim \mathcal{N}_{1\times d_z}(0,1)$ are independent of $A^i$ and $b^i$, 
$[x_t, \dots, x_t] \in \mathbb{R}^{d_x\times d_z}$, $\beta^i_t\in \mathbb{R}^{d_z \times 1}$, $\sigma^i_t \in \mathbb{R}^{d_y \times 1}$, and $\bullet$ denotes the element-wise application of the $\operatorname{ReLU}$ function. 
Here, $\varepsilon_t^i\eqdef \sigma^i_t \, W_t^i$, 
$\varphi^{i}(x_{[0:t]}^i,Z_{t-1}^i, \varepsilon_t^i)\eqdef \operatorname{ReLU}\bullet\big(A^{i} [x_t, \dots, x_t] + b^{i} + \sigma^i_t \, W_t^i\big) = Z_{t}^i$.  
We have the block matrices 
\begin{align*}
& \mathbf{A}(t)
 = 
\big[ A^{(i,j)}(t)\big]_{i,j=1}^N , 
\hspace{0.5cm} 
\widehat{\mathbf{A}}(t) = \big[ \widehat{A}^{(i,j)}(t) \big]_{i, j = 1}^N , 
\notag \\  
 & \mathbf{B}(t)  = 
 \begin{bmatrix} 
  P_1(t+1) \mathbf{e}_1 a(t, 1), \dots, P_N(t+1) \mathbf{e}_N a(t, N) \\ 
 \end{bmatrix}^\top  ,  
 \notag \\ 
 & \mathbf{C}(t)  = 
 \begin{bmatrix} 
  S_1(t+1)^\top \mathbf{e}_1 a(t, 1), \dots, S_N(t+1)^\top  \mathbf{e}_N a(t, N) \\ 
 \end{bmatrix}^\top , 
  \notag \\   
& \mathbf{D}(t) 
  = 
 \begin{bmatrix} 
 \mathbf{e}_1 a(t,1), \dots  , \mathbf{e}_N a(t, N) 
 \end{bmatrix} , 
 \notag \\ 
 & \mathbf{D}_i(t) = \big[ D_i^{(j,k)}(t) \big]_{j,k = 1}^N , \quad i = 1, \dots, N . \notag 
\end{align*}
For $t=1,\dots,N$ and $t\in \mathbb{N}_+$, each $a(t,i)$ is defined as follows
\begin{align} 
a^i_t\eqdef A^{i}[x_t, \dots, x_t] + b^{i} 
\mbox{ and }
a(t,i)
= 
   a^i_t 
        \odot
        \left(
            \bar{1}_{d_z} 
            - 
            \Phi\bullet\left(
                \frac{-a^i_t}{\sigma^i_t}
            \right)
        \right)
        +
            \frac{\sigma_t^i}{\sqrt{2\pi}}
        \,
        \exp\bullet\left(
            \frac{- ( a^i_t \odot a^i_t )}{2\, (\sigma^i_t)^2 }
        \right)  , \notag 
\end{align} 
and $\odot$ and $\Phi$ are as in Lemma \ref{lem:Mean_Cov_ResModel}. 
The sub-matrices of $\mathbf{A}(t)$, $\widehat{\mathbf{A}}(t)$, 
and $\{ \mathbf{D}_i(t) \}_{i=1}^N$, 
are given by
\begin{align} 
& A^{(i,j)}(t) =
\begin{cases} 
  \mathbf{e}_i^\top P_i(t+1) \mathbf{e}_i \mathbb{E} \big[  Z^{i\top}_t  Z^i_t  \big] , \quad & \mbox{if} \,\, i = j , \\
  \mathbf{e}_i^\top P_i(t+1) \mathbf{e}_j a(t, i)^\top a(t, j) , \quad & \mbox{if} \,\, i \neq j  , 
 \end{cases} 
 \notag \\ 
 & \widehat{A}^{(i,j)}(t) = 
 \begin{cases} 
 w^i_t w^i_t  \mathbb{E} \big[ Z^{i\top}_t Z^i_t \big]  , \quad & \mbox{if} \,\, i = j , \\ 
 w^i_t w^j_t a(t, i)^\top   a(t, j) , \quad & \mbox{if} \,\, i \neq j   , 
\end{cases} 
\notag \\ 
& D_i^{(j,k)}(t) = 
\begin{cases}
  \mathbf{e}_j^\top P_i(t+1) \mathbf{e}_j \mathbb{E} \big[  Z^{j\top}_t  Z^j_t  \big]  , \quad & \mbox{if} \,\, j=k , \\ 
 \mathbf{e}_j^\top P_i(t+1) \mathbf{e}_k a(t,j)^\top  a(t, k) , \quad & \mbox{if} \,\, j\neq k , 
\notag 
\end{cases} 
\end{align} 
where the matrices $\mathbb{E}\big[ Z^{i\top}_t Z^i_t \big] = \big( \mathbb{E}\big[ Z^{i\top}_t Z^i_t \big]_{jk}  \big)_{j,k=1}^{d_z}$, $i=1, \dots, N$ are given by 
\begin{align}  
& \mathbb{E} \big[ Z^{i\top}_t Z^i_t \big]_{jk}  
 =  
    \begin{cases}
            \Big( (a^i_t \odot a^i_t )_j + (\sigma^i_t)^2 \Big)
            \cdot
            \left(
                1 
                - 
                \Phi\left(
                    \frac{-(a^i_t)_j}{\sigma_t^i}
                \right)
            \right) 
            +
            (a^i_t)_j\,
                \frac{\sigma_t^i}{\sqrt{2\pi}}
            \,
            \exp\Big( 
                \frac{ - (a^i_t \odot a^i_t )_j }{2\,(\sigma_t^i)^2 }
            \Big) ,
    & 
        \mbox{ if } j=k , 
\\ 
      a(t, i)_j a(t,i)_k , 
    & \mbox{ if }   j \neq k . 
    \end{cases} 
    \notag 
\end{align} 
\end{corollary} 

\begin{figure}[ht!]
    \centering
    \begin{subfigure}[b]{0.99\linewidth}
    \centering
    \begin{minipage}[c]{0.05\linewidth}
    \caption{}
    \end{minipage}
    \begin{minipage}[c]{0.8\linewidth}
        \centering
        \includegraphics[width=\linewidth]{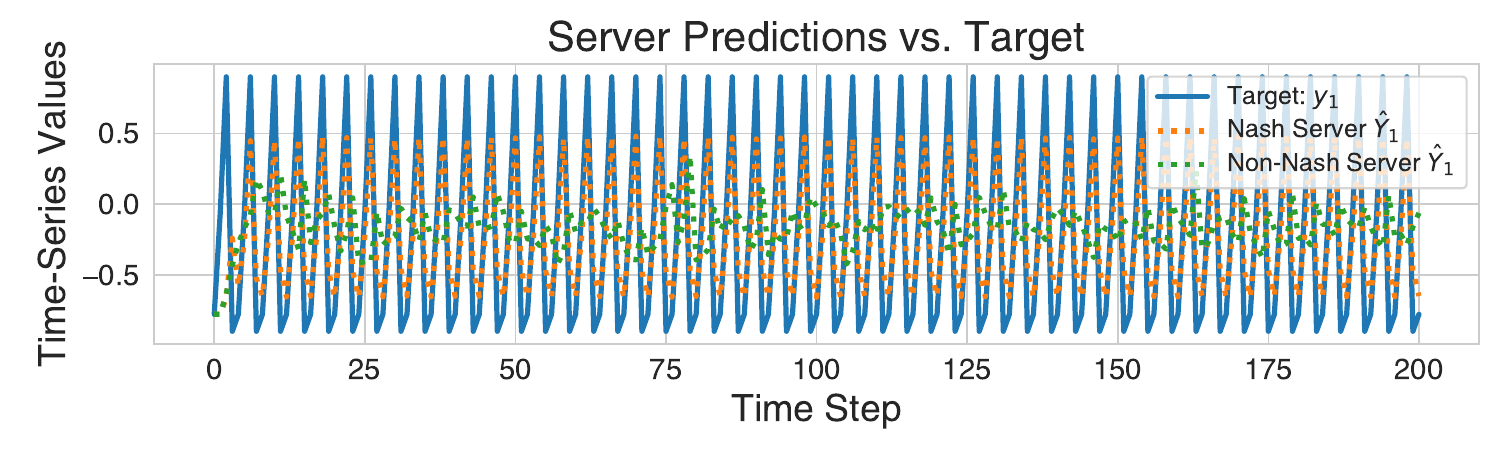}
    \end{minipage}
    \label{fig:rfn_periodic_overlaid_server_predictions}
    \end{subfigure}
    
    \begin{subfigure}[b]{0.99\linewidth}
    \centering
    \begin{minipage}[c]{0.05\linewidth}
    \caption{}
    \end{minipage}
    \begin{minipage}[c]{0.8\linewidth}
        \centering
        \includegraphics[width=\linewidth]{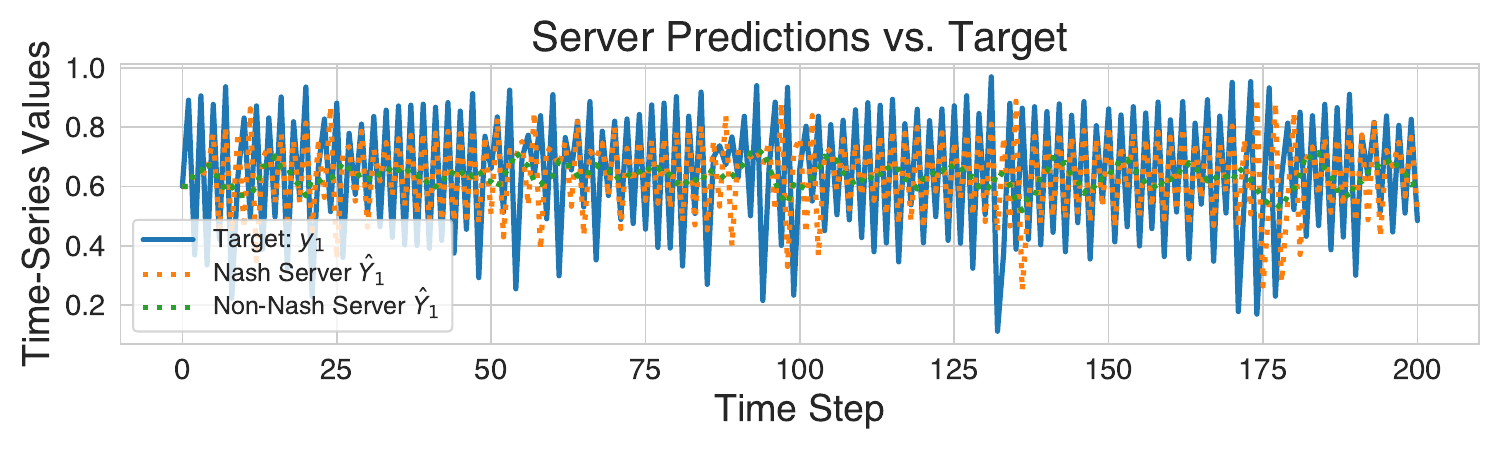}
    \end{minipage}
    \label{fig:rfn_logistic_overlaid_server_predictions}
    \end{subfigure}

    \begin{subfigure}[b]{0.99\linewidth}
    \centering
    \begin{minipage}[c]{0.05\linewidth}
    \caption{}
    \end{minipage}
    \begin{minipage}[c]{0.8\linewidth}
        \centering
        \includegraphics[width=\linewidth]{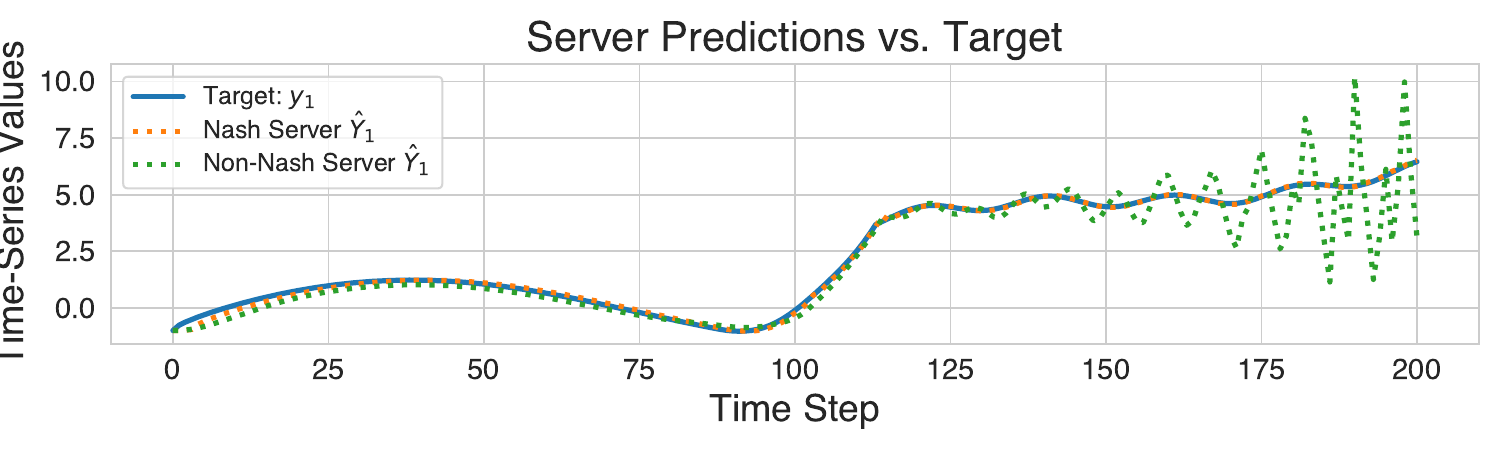}
    \end{minipage}
    \label{fig:rfn_concept_overlaid_server_predictions}
    \end{subfigure}

    \begin{subfigure}[b]{0.99\linewidth}
    \centering
    \begin{minipage}[c]{0.05\linewidth}
    \caption{}
    \end{minipage}
    \begin{minipage}[c]{0.8\linewidth}
        \centering
        \includegraphics[width=\linewidth]{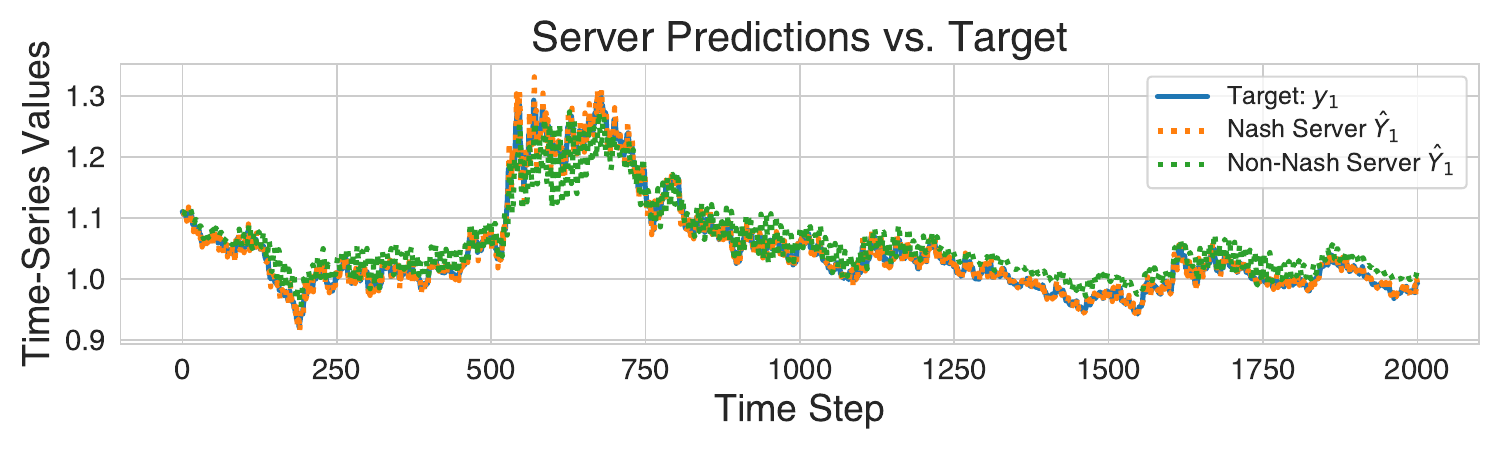}
    \end{minipage}
    \label{fig:rfn_boc_overlaid_server_predictions}
    \end{subfigure}

    \begin{subfigure}[b]{0.99\linewidth}
    \centering
    \begin{minipage}[c]{0.05\linewidth}
    \caption{}
    \end{minipage}
    \begin{minipage}[c]{0.8\linewidth}
        \centering
        \includegraphics[width=\linewidth]{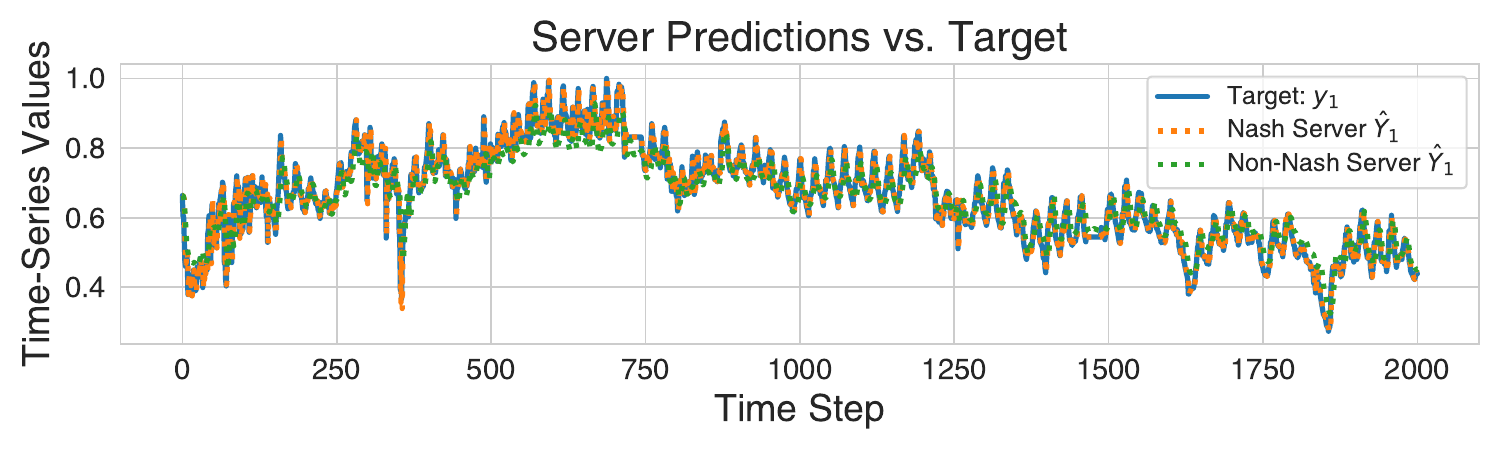}
    \end{minipage}
    \label{fig:rfn_ett_overlaid_predictions}
    \end{subfigure}
    \caption{Server predictions compared with ground truth targets for mixtures of RFN models with and without Nash-game synchronization for the Periodic (\subref{fig:rfn_periodic_overlaid_server_predictions}), Logistic (\subref{fig:rfn_logistic_overlaid_server_predictions}), Concept Drift (\subref{fig:rfn_concept_overlaid_server_predictions}), BoC Exchange Rates (\subref{fig:rfn_boc_overlaid_server_predictions}), and ETT (\subref{fig:rfn_ett_overlaid_predictions}) time series.}
    \label{fig:rfn_overlaid_prediction}
\end{figure}

\subsubsection{RFN Experimental Results}
\label{s:Examples__ss:RFN__sss:Results}

Across all datasets, the predictions incorporating the Nash game remain significantly better. This is qualitatively evident in the various trajectories shown in Figure \ref{fig:rfn_overlaid_prediction}. The RFN models are simpler than the transformers models of Section \ref{s:Analysis__ss:PreTrained} in terms of expressiveness. As such, the predictions in the non-game setting tend to be of somewhat lower quality. This is quantitatively supported by the numerical results in Table \ref{tab:summary_results}. Nonetheless, Figure \ref{fig:rfn_squared_error} illustrates that, for all signals, over the majority of time steps, the Nash game achieves lower squared errors compared to the non-game approach. Moreover, use of Nash synchronization brings the predictions of such models into closer competition with those of the transformers. Note that the best hyperparameter choices are recorded in Table \ref{tab:rfn_hyperparams} of Appendix \ref{s:OptimalParamterChoices}.

\begin{figure}[ht!]
    \centering
    \begin{subfigure}[b]{0.99\linewidth}
    \centering
    \begin{minipage}[c]{0.05\linewidth}
    \caption{}
    \end{minipage}
    \begin{minipage}[c]{0.7\linewidth}
        \centering
        \includegraphics[width=\linewidth]{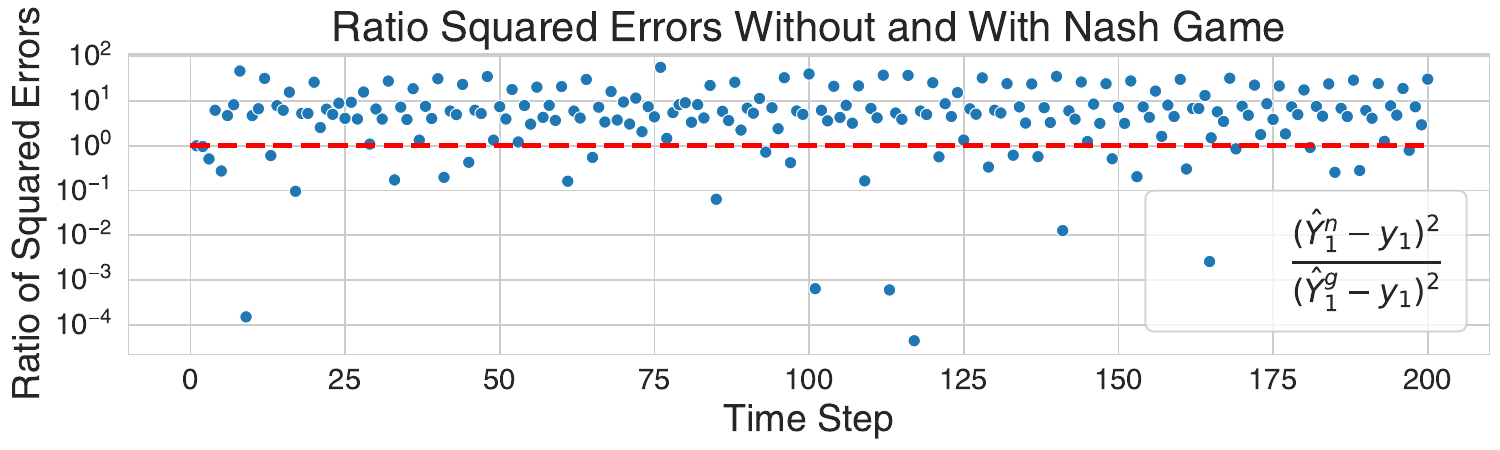}
    \end{minipage}
    \label{fig:rfn_periodic_relative_squared_errors}
    \end{subfigure}

    \begin{subfigure}[b]{0.99\linewidth}
    \centering
    \begin{minipage}[c]{0.05\linewidth}
    \caption{}
    \end{minipage}
    \begin{minipage}[c]{0.7\linewidth}
        \centering
        \includegraphics[width=\linewidth]{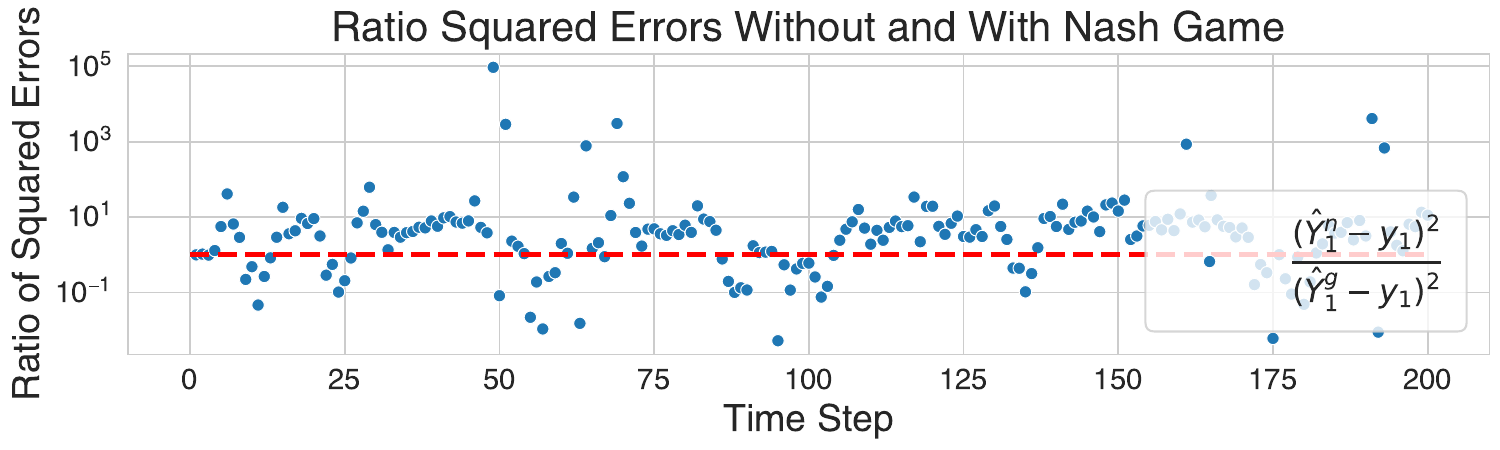}
    \end{minipage}
    \label{fig:rfn_logistic_relative_squared_errors}
    \end{subfigure}

    \begin{subfigure}[b]{0.99\linewidth}
    \centering
    \begin{minipage}[c]{0.05\linewidth}
    \caption{}
    \end{minipage}
    \begin{minipage}[c]{0.7\linewidth}
        \centering
        \includegraphics[width=\linewidth]{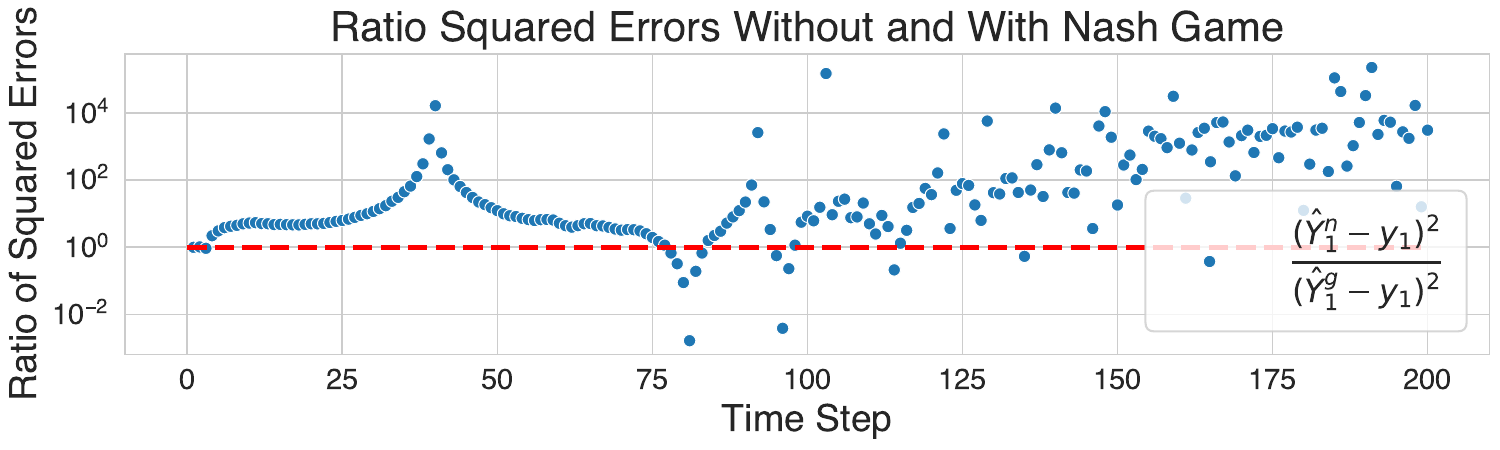}
    \end{minipage}
    \label{fig:rfn_concept_relative_squared_errors}
    \end{subfigure}

    \begin{subfigure}[b]{0.99\linewidth}
    \centering
    \begin{minipage}[c]{0.05\linewidth}
    \caption{}
    \end{minipage}
    \begin{minipage}[c]{0.7\linewidth}
        \centering
        \includegraphics[width=\linewidth]{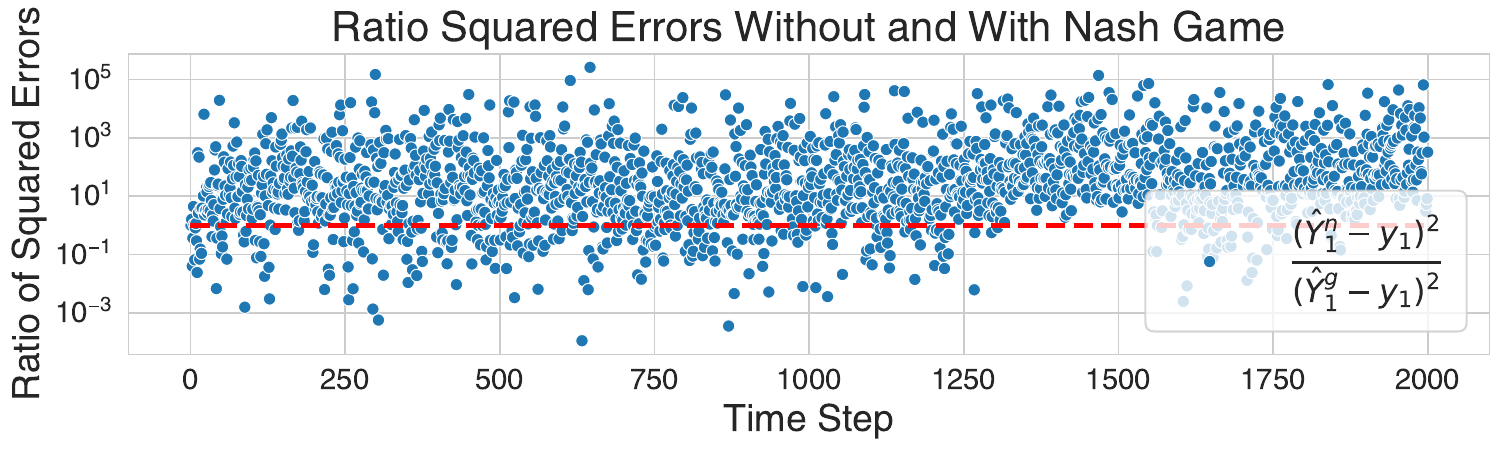}
    \end{minipage}
    \label{fig:rfn_boc_relative_squared_server_rrrors}
    \end{subfigure}

    \begin{subfigure}[b]{0.99\linewidth}
    \centering
    \begin{minipage}[c]{0.05\linewidth}
    \caption{}
    \end{minipage}
    \begin{minipage}[c]{0.7\linewidth}
        \centering
        \includegraphics[width=\linewidth]{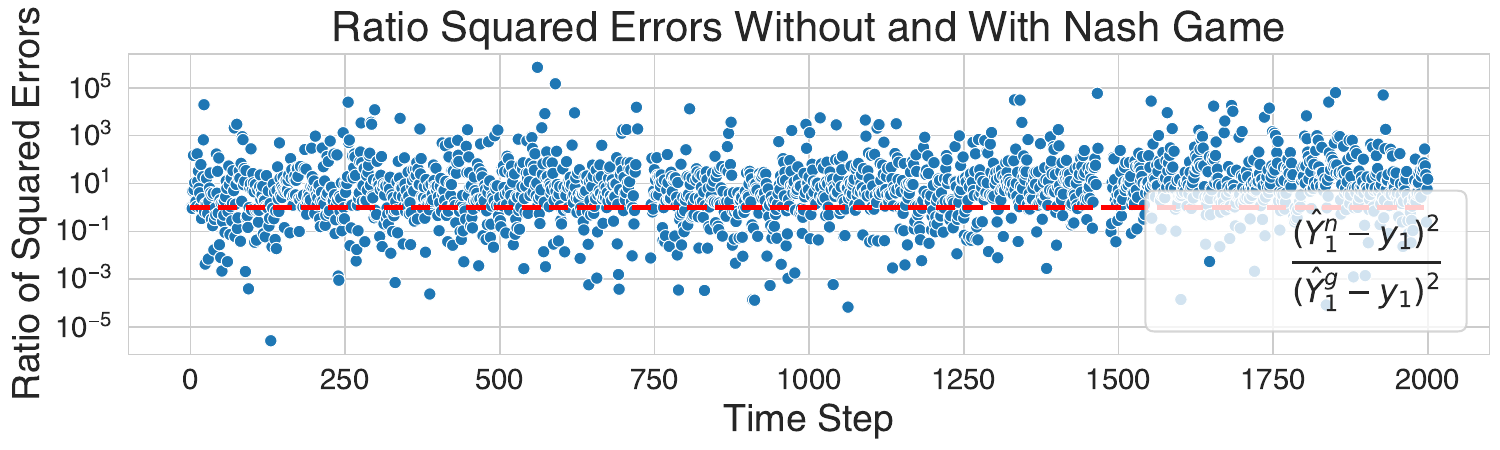}
    \end{minipage}
    \label{fig:rfn_ett_rlative_squared_error}
    \end{subfigure}
    \caption{Comparison of relative squared errors for RFN model predictions with and without Nash-game synchronization for the Periodic (\subref{fig:rfn_periodic_relative_squared_errors}), Logistic (\subref{fig:rfn_logistic_relative_squared_errors}), Concept Drift (\subref{fig:rfn_concept_relative_squared_errors}), BoC Exchange Rates (\subref{fig:rfn_boc_relative_squared_server_rrrors}), and ETT (\subref{fig:rfn_ett_rlative_squared_error}) time series.}
    \label{fig:rfn_squared_error}
\end{figure}

Similar to the transformer results, predictions without Nash synchronization settle near the mean of the Periodic and Logistic series with smaller modulations in the prediction magnitude. During regions of more unstable behavior, occurring, for example, between $t=45$ to $t=65$ in the Logistic time series, this approach can inadvertently produce a better fit due to its conservative nature, as seen in Figure  \ref{fig:rfn_squared_error}\subref{fig:rfn_logistic_relative_squared_errors}. However, over the course of the series, the Nash-game synchronization process provides much more accurate predictions for both datasets.

For simplicity in the RFN experiments, the Concept Drift time series target dimensionality is reduced such that $d_y = 1$. In this case, the $y_2$ series is retained. Predictions made with Nash-game synchronization outperform those without by a large margin. As seen in Figures  \ref{fig:rfn_overlaid_prediction}\subref{fig:rfn_concept_overlaid_server_predictions} and \ref{fig:rfn_squared_error}\subref{fig:rfn_concept_relative_squared_errors}, at the beginning of the time series, both approaches perform quite well. Predictions with the Nash game are of relatively higher quality, but the non-Nash game setting is competitive. However, after the concept transition, and especially later in the series, prediction errors appear to compound for the non-Nash setting causing heavy oscillations and poor predictions. Given that the series transitions from a lower to higher frequency pattern in its oscillations, it's possible that the hyperparameters for the non-Nash setting are well-fitted to the lower frequency setting but less so for the higher regime, and it fails to adapt.

The RFN models produce fairly accurate predictions for the BoC Exchange Rate series. Considering the squared errors in Figure \ref{fig:rfn_squared_error}\subref{fig:rfn_boc_relative_squared_server_rrrors}, the vast majority of predictions throughout the time series are above the threshold of $1.0$. This is especially true in the later stages of the series. For many of the time steps, predictions based on the Nash game are three to five orders of magnitude better fits. Likewise, the RFN models produce relatively good predictions for the ETT dataset. For the server predictions, displayed in Figure \ref{fig:rfn_overlaid_prediction}\subref{fig:rfn_ett_overlaid_predictions}, those with the Nash-game synchronization qualitatively do a better job matching the magnitudes of the time series. This is also borne out in the ratio of squared errors plots of Figure \ref{fig:rfn_squared_error}\subref{fig:rfn_ett_rlative_squared_error}. The bulk of predictions errors favor the Nash-game setting. Further, it appears that the predictions for the RFN models are more heavily distributed in favor of the Nash synchronization process than in the transformer model settings, see Figure \ref{fig:transfomer_overlaid}\subref{fig:transformer_ett_squared_errors}.

\subsection{Echo-State Agents - The Classical Reservoir Computers}
\label{s:Examples_ReservoirComputer}
Lastly, the framework is applied to reservoir computers. The setting of Corollary \ref{cor:Random_NeuralNetwork} is augmented by considering an additional random matrix, $B^i \in \mathbb{R}^{d_y \times d_y}$, independent of the $W_t^i$, for all $t$ and $i$.  
If we allow for the hidden state, $Z_t$, generated at time $t$ to be remembered, then we obtain an ESN, which is a special case of a general reservoir computer, with structure
\begin{equation}
\label{eq:reservoir}
\begin{aligned}
Z_{t}^i\eqdef & \text{Hard Sigmoid} \bullet \big(
    A^i [x_t, \dots , x_t] + B^i Z_{t-1}^i + b^i + \sigma^i \, W_t^i
\big) , 
\\
\hat{Y}^i_{t+1} \eqdef & \hat{Y}^i_t +  Z_t^i \beta^{i}_t
.
\end{aligned}
\end{equation}
There are two subtle differences between Equation \eqref{eq:reservoir} and standard ESN formulations. First, the final layer of the ESN, $\beta$, is allowed to be dynamically updated, which is not usually the case in reservoir computing. Second, the residual updates in Equation \eqref{eq:reservoir} differ from standard ESNs. Unlike the previous types of agents, namely pre-trained encoders and RFNs, no simple closed forms are available for the expectation matrices in Theorem \ref{thm:NashEqmExistence}.  Rather, these quantities are estimated in the experiments using standard Monte-Carlo sampling. While the approach can be computationally heavy, the process is also quite effective, as seen in the results to follow.

\begin{figure}[ht!]
    \centering
    \begin{subfigure}[b]{0.99\linewidth}
    \centering
    \begin{minipage}[c]{0.05\linewidth}
    \caption{}
    \end{minipage}
    \begin{minipage}[c]{0.8\linewidth}
        \centering
        \includegraphics[width=\linewidth]{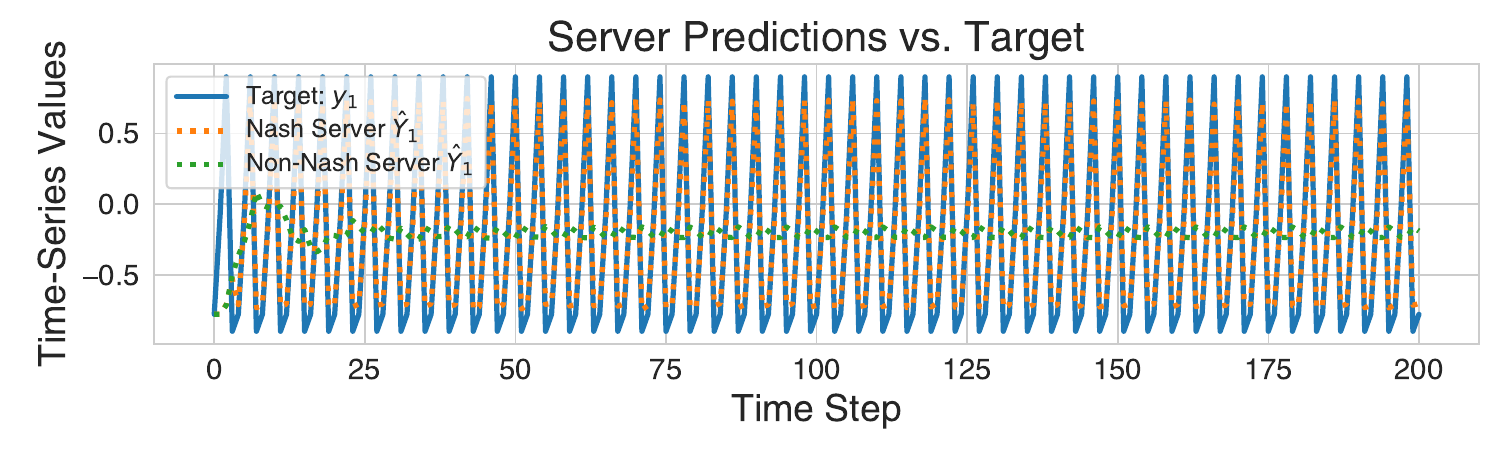}
    \end{minipage}
    \label{fig:esn_periodic_overlaid_server_predictions}
    \end{subfigure}

    \begin{subfigure}[b]{0.99\linewidth}
    \centering
    \begin{minipage}[c]{0.05\linewidth}
    \caption{}
    \end{minipage}
    \begin{minipage}[c]{0.8\linewidth}
        \centering
        \includegraphics[width=\linewidth]{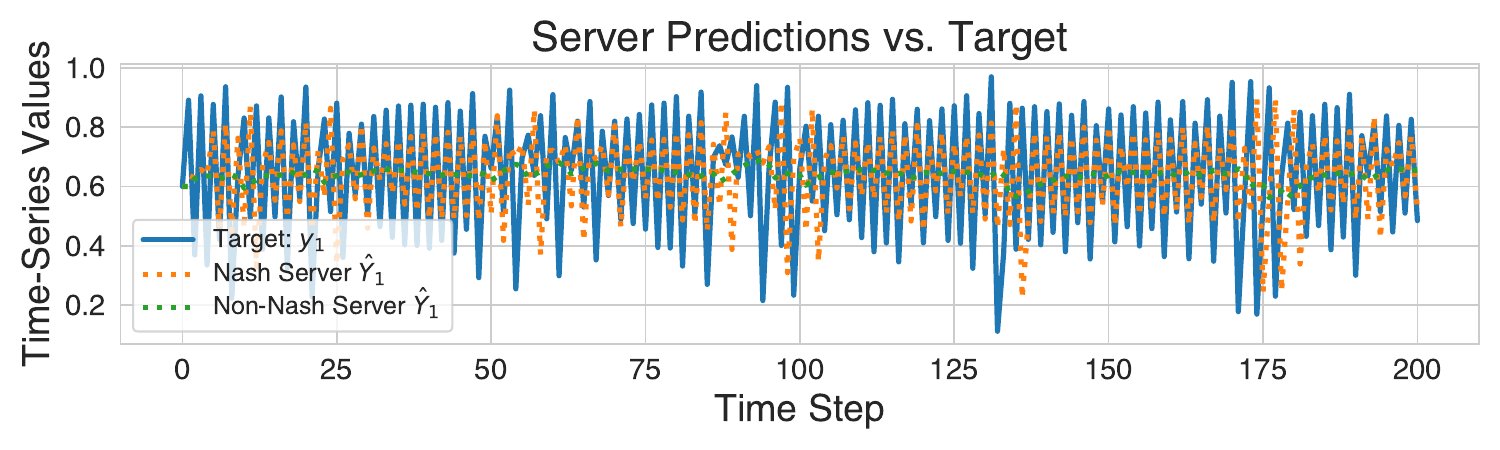}
    \end{minipage}
    \label{fig:esn_logistic_overlaid_server_predictions}
    \end{subfigure} 

    \begin{subfigure}[b]{0.99\linewidth}
    \centering
    \begin{minipage}[c]{0.05\linewidth}
    \caption{}
    \end{minipage}
    \begin{minipage}[c]{0.8\linewidth}
        \centering
        \includegraphics[width=\linewidth]{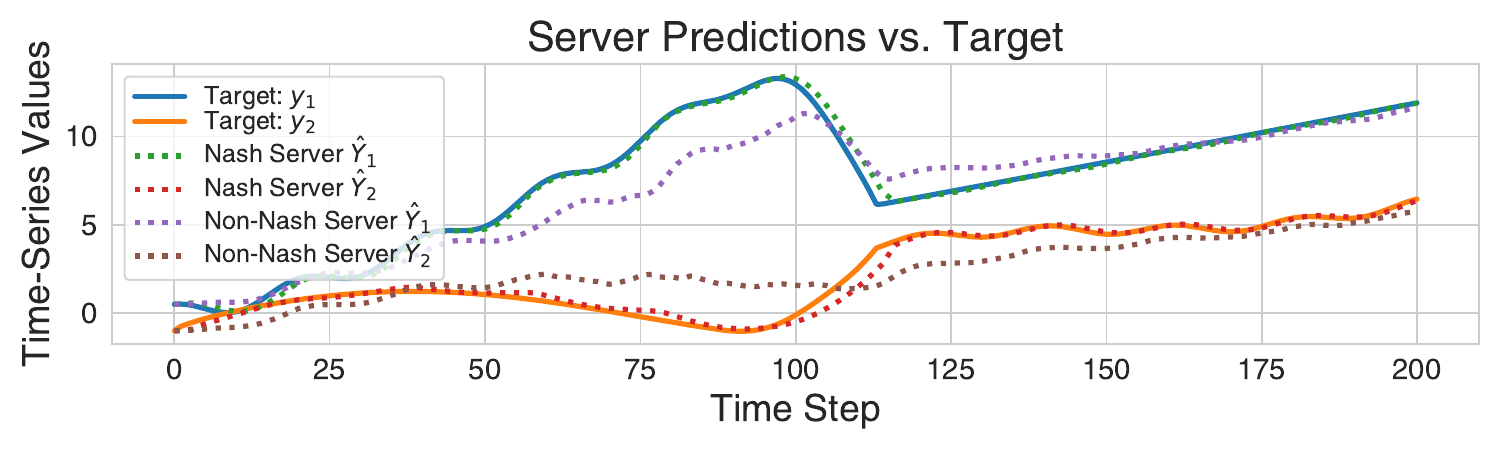}
    \end{minipage}
    \label{fig:esn_concept_overlaid_server_predictions}
    \end{subfigure}

    \begin{subfigure}[b]{0.99\linewidth}
    \centering
    \begin{minipage}[c]{0.05\linewidth}
    \caption{}
    \end{minipage}
    \begin{minipage}[c]{0.8\linewidth}
        \centering
        \includegraphics[width=\linewidth]{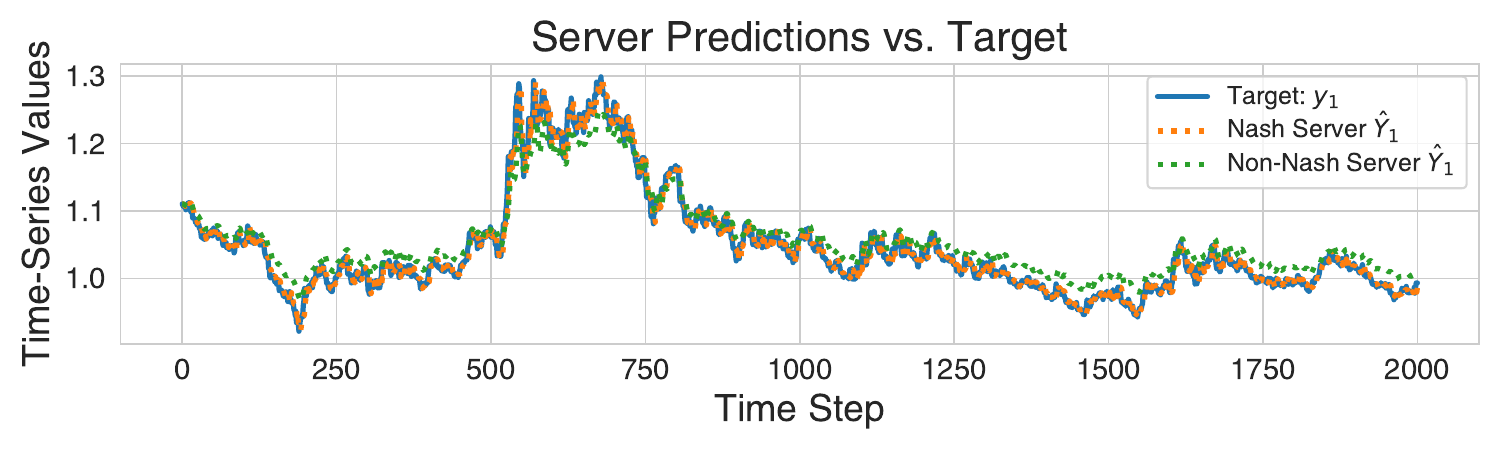}
    \end{minipage}
    \label{fig:esn_boc_overlaid_server_predictions}
    \end{subfigure}

    \begin{subfigure}[b]{0.99\linewidth}
    \centering
    \begin{minipage}[c]{0.05\linewidth}
    \caption{}
    \end{minipage}
    \begin{minipage}[c]{0.8\linewidth}
        \centering
        \includegraphics[width=\linewidth]{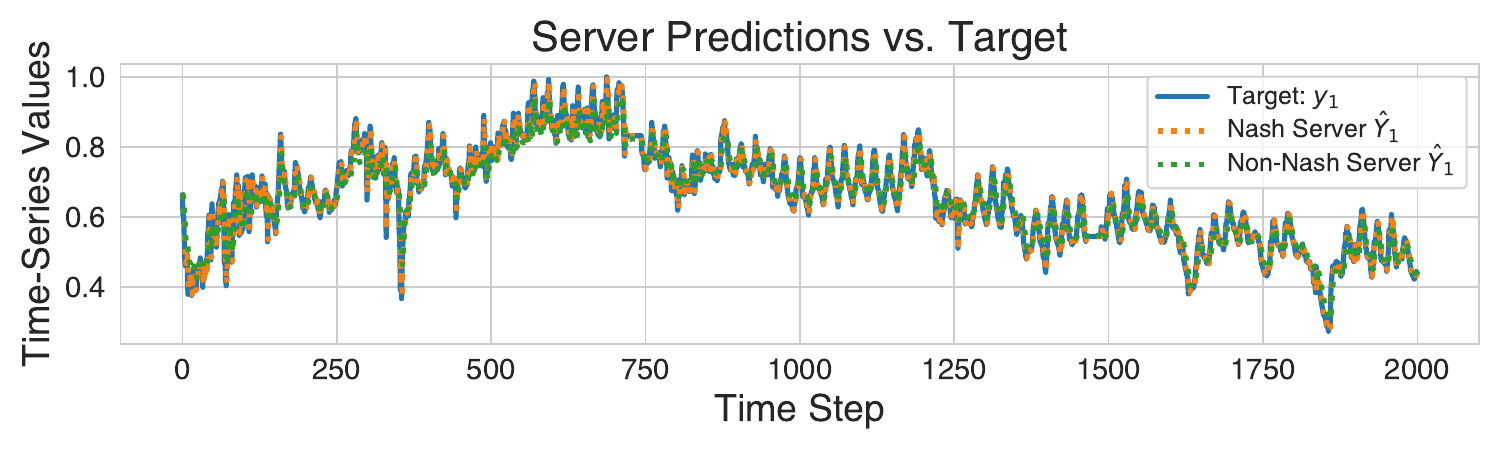}
    \end{minipage}
    \label{fig:esn_ett_overlaid_server_predictions}
    \end{subfigure}
    \caption{Server predictions compared with ground truth targets for mixtures of ESN models with and without Nash-game synchronization for the Periodic (\subref{fig:esn_periodic_overlaid_server_predictions}), Logistic (\subref{fig:esn_logistic_overlaid_server_predictions}), Concept Drift (\subref{fig:esn_concept_overlaid_server_predictions}), BoC Exchange Rates (\subref{fig:esn_boc_overlaid_server_predictions}), and ETT (\subref{fig:esn_ett_overlaid_server_predictions}) time series.}
    \label{fig:esn_overlaid_prediction_comparison}
\end{figure}

\subsubsection{ESN Experimental Results}
\label{s:Examples__ss:ESN__sss:Results}

As in previous experiments, predictions incorporating the Nash synchronization procedure significantly outperform those without. This is seen both qualitatively, in Figure \ref{fig:esn_overlaid_prediction_comparison}, and quantitatively in Figure \ref{fig:esn_relative_squared_errors}. Consistent with previous results, predictions without the Nash game are largely damped towards the average of both the Periodic and Logistic time series. On the other hand, while not quite reaching the full amplitude, predictions incorporating the game match provide much better replication of the series dynamics. However, during periods of instability in the Logistic Map series, the damped predictions produced without Nash synchronization temporarily outperform those including such synchronization. This phenomenon can be seen, for example, between $t=50$ and $t=65$ or $t=85$ to $t=105$. In spite of this, Nash-game synchronization produces notable reductions in prediction error. In Figures \ref{fig:esn_relative_squared_errors}\subref{fig:esn_periodic_relative_squared_errors} and \ref{fig:esn_relative_squared_errors}\subref{fig:esn_logistic_relative_squared_errors}, the majority error ratios are above $1.0$.

For the Concept Drift dataset, non-Nash predictions actually struggle to accurately model the target values in both phases of the series, as shown in Figure \ref{fig:esn_overlaid_prediction_comparison}\subref{fig:esn_concept_overlaid_server_predictions}. The fit produced by the ESN model in this setting for the second stage is generally more accurate than seen in preceding experiments, while predictions for the first half are worse. Nonetheless, as in previous results, the predictions produced by the server when leveraging Nash synchronization are much more tightly bound to the target curves. This is even more evident when examining the squared-error ratios in Figure \ref{fig:esn_relative_squared_errors}\subref{fig:esn_concept_relative_squared_error}. 

Similar to the results from other model types, both Nash and non-Nash game predictions follow the target signal fairly closely in the BoC Exchange Rate and ETT datasets, due in part to the higher quality of input features. In the BoC Exchange Rate dataset, the non-game server prediction fails to match the target signal during several time periods, such as between $t=600$ to $t=700$. This is, overall, not the case for Nash-game predictions, which demonstrate better alignment with the target. For the ETT dataset, predictions incorporating Nash synchronization more accurately capture the magnitudes of various spikes in the oil temperature throughout the series. As shown in Figures \ref{fig:esn_relative_squared_errors}\subref{fig:esn_boc_relative_squared_errors} and \ref{fig:esn_relative_squared_errors}\subref{fig:esn_ett_rlative_squared_errors}, the bulk of time steps have squared-error ratios above $1.0$, indicating sustained performance improvements from predictions leveraging Nash-game synchronization than those without.

\begin{figure}[ht!]
    \centering
    \begin{subfigure}[b]{0.99\linewidth}
    \centering
    \begin{minipage}[c]{0.05\linewidth}
    \caption{}
    \end{minipage}
    \begin{minipage}[c]{0.7\linewidth}
        \centering
        \includegraphics[width=\linewidth]{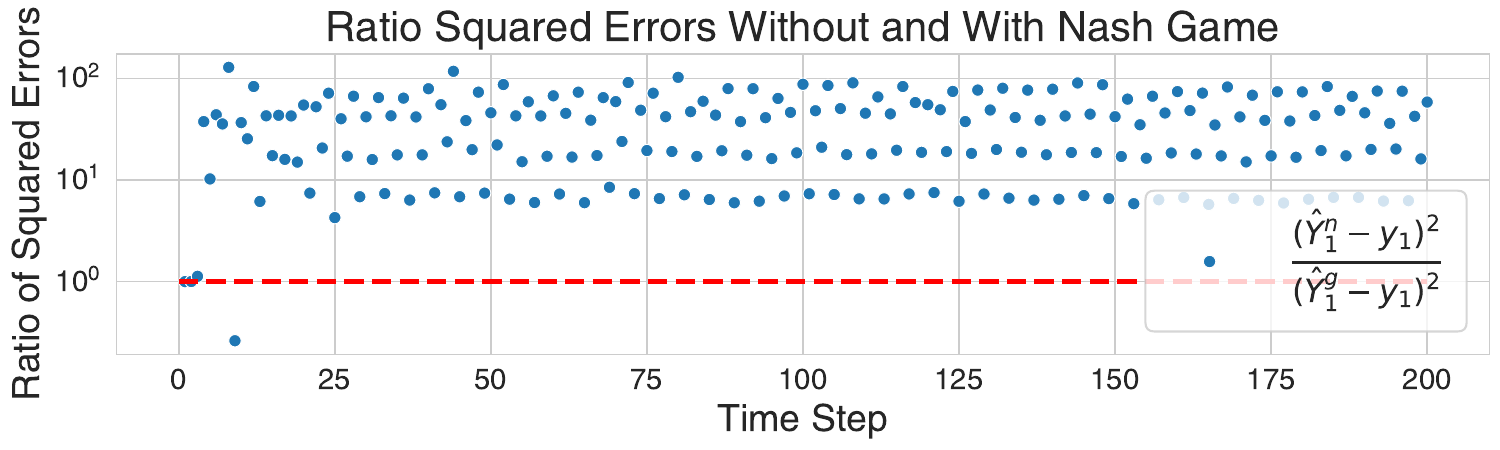}
    \end{minipage}
    \label{fig:esn_periodic_relative_squared_errors}
    \end{subfigure}

    \begin{subfigure}[b]{0.99\linewidth}
    \centering
    \begin{minipage}[c]{0.05\linewidth}
    \caption{}
    \end{minipage}
    \begin{minipage}[c]{0.7\linewidth}
        \centering
        \includegraphics[width=\linewidth]{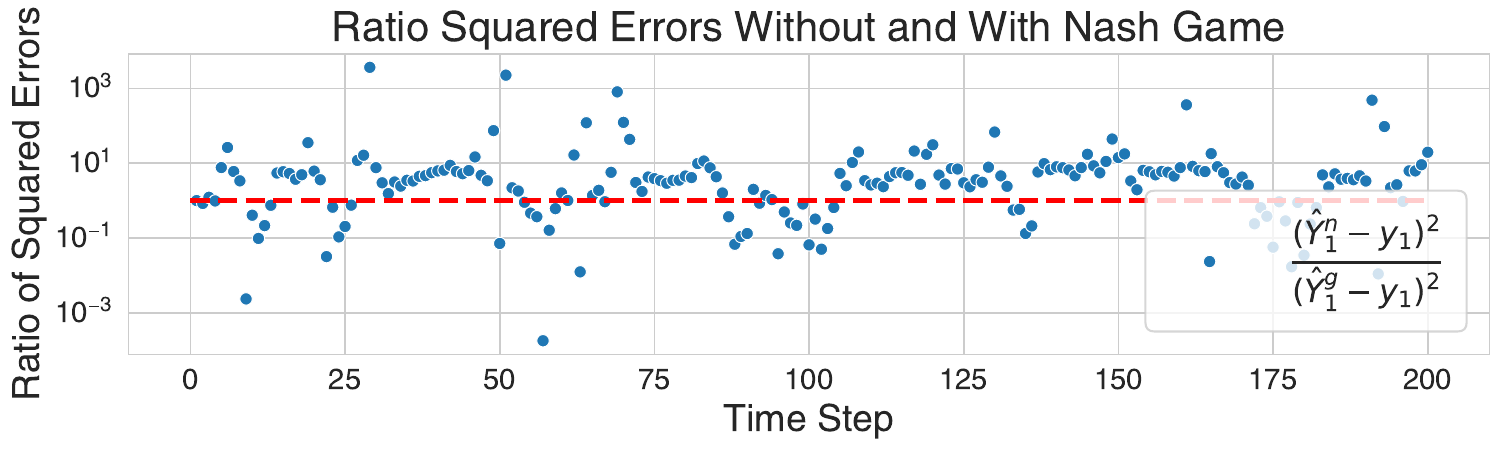}
    \end{minipage}
    \label{fig:esn_logistic_relative_squared_errors}
    \end{subfigure}

    \begin{subfigure}[b]{0.99\linewidth}
    \centering
    \begin{minipage}[c]{0.05\linewidth}
    \caption{}
    \end{minipage}
    \begin{minipage}[c]{0.7\linewidth}
        \centering
        \includegraphics[width=\linewidth]{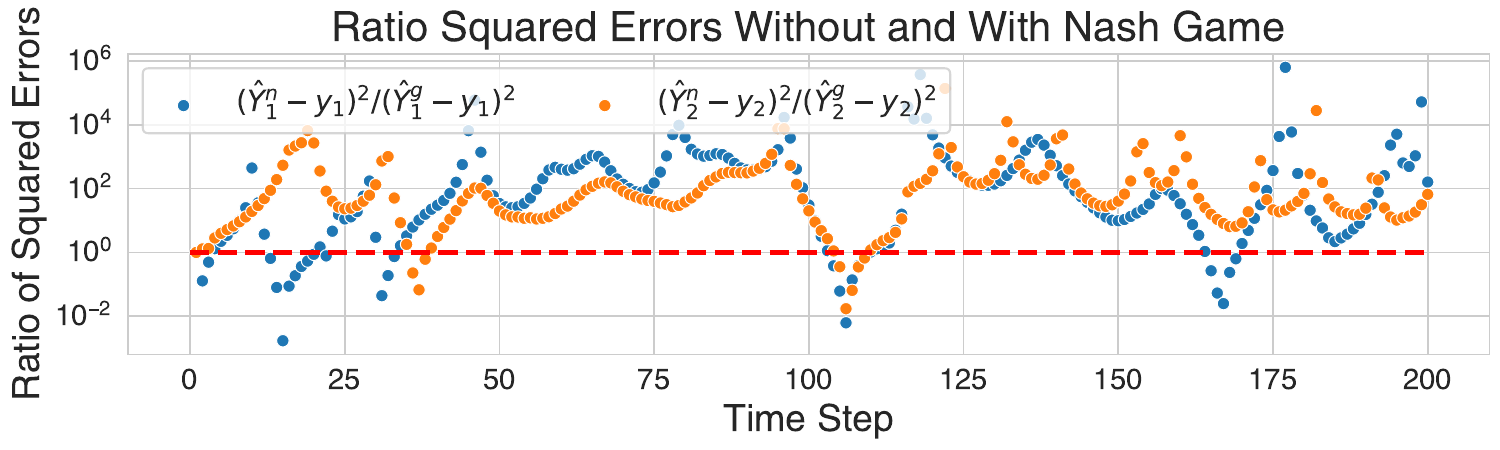}
    \end{minipage}
    \label{fig:esn_concept_relative_squared_error}
    \end{subfigure}

    \begin{subfigure}[b]{0.99\linewidth}
    \centering
    \begin{minipage}[c]{0.05\linewidth}
    \caption{}
    \end{minipage}
    \begin{minipage}[c]{0.7\linewidth}
        \centering
        \includegraphics[width=\linewidth]{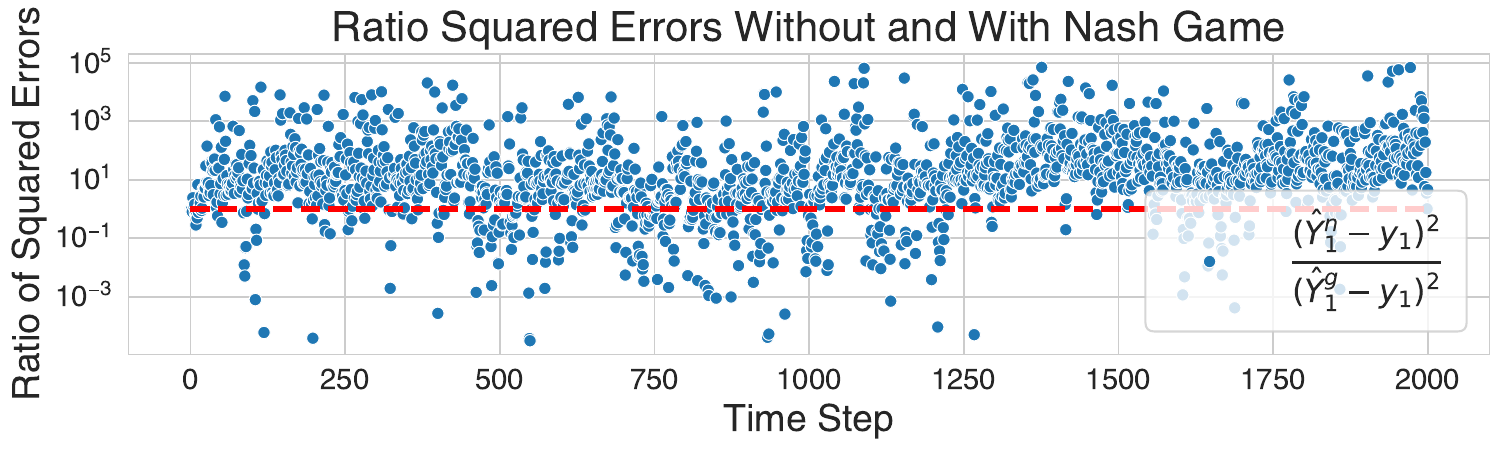}
    \end{minipage}
    \label{fig:esn_boc_relative_squared_errors}
    \end{subfigure}

    \begin{subfigure}[b]{0.99\linewidth}
    \centering
    \begin{minipage}[c]{0.05\linewidth}
    \caption{}
    \end{minipage}
    \begin{minipage}[c]{0.7\linewidth}
        \centering
        \includegraphics[width=\linewidth]{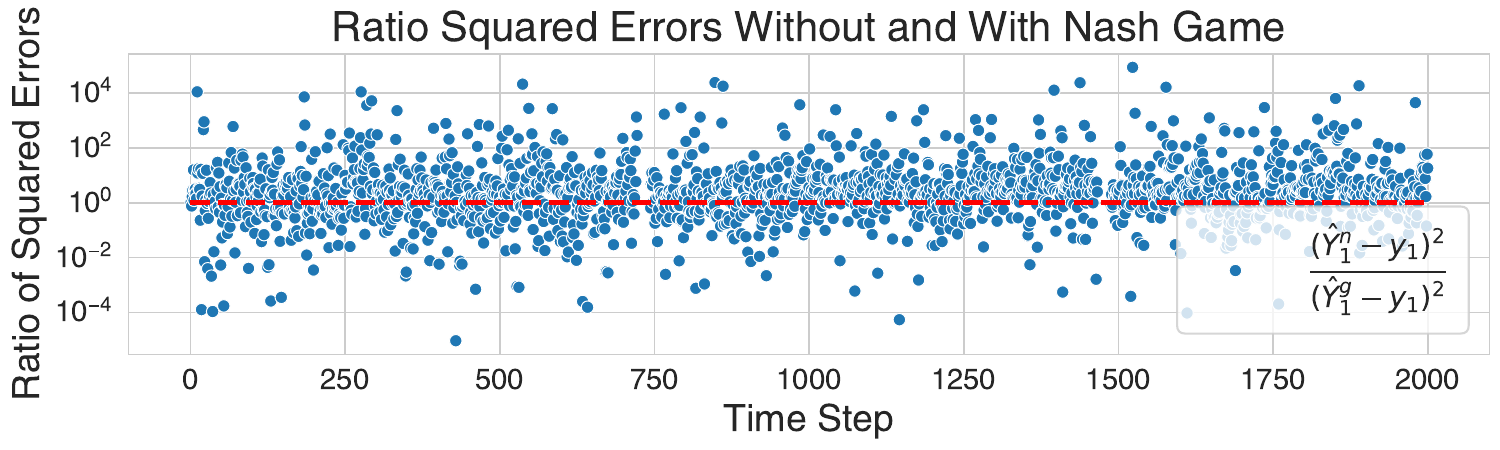}
    \end{minipage}
    \label{fig:esn_ett_rlative_squared_errors}
    \end{subfigure}
    \caption{Comparison of relative squared errors for ESN model predictions with and without Nash-game synchronization for the Periodic (\subref{fig:esn_periodic_relative_squared_errors}), Logistic (\subref{fig:esn_logistic_relative_squared_errors}), Concept Drift (\subref{fig:esn_concept_relative_squared_error}), BoC Exchange Rates (\subref{fig:esn_boc_relative_squared_errors}), and ETT (\subref{fig:esn_ett_rlative_squared_errors}) time series.}
    \label{fig:esn_relative_squared_errors}
\end{figure}

\subsection{Ablation Studies}
\label{s:Experiments__ss:Ablation}
In this section, the effects of adjusting the principal knobs and hyperparameters driving the behaviour of the proposed FL algorithm are examined. Specifically, we analyze the impact of the Nash game on the mixture weights optimized by the server, the look-back parameter ($T$) in the Nash games, and the influence of varying the frequency with which Nash synchronization occurs ($\tau$), with special focus on the trade-off between runtime of the FL algorithm and its predictive power.

\subsubsection{Mixture weights Analysis}
\label{s:Experiments__ss:Ablation___sss:MixtureWeights}

The Concept Drift dataset goes through an abrupt, but smooth, transition in the relationship between the input and output variables approximately halfway through the series. More details are provided in Appendix \ref{concept_drift_description}. This shift occurs between time steps $88$ and $112$. The onset of this transition is evident in the mixture weights computations of both methods in Figure \ref{fig:mixture_weights_concept_drift}. During this period, the dynamics of these weights changes significantly. The transition also appears to briefly impact the efficacy of Nash synchronization, evidenced by a degradation and eventual recovery of the ratio of errors in Figure \ref{fig:esn_relative_squared_errors}\subref{fig:esn_concept_relative_squared_error} between the non-game and game predictions during this time-frame. Of note for this dataset, there are clearly agents in the ensemble that struggle to produce helpful predictions. For example, two of the clients in the Nash game setting hover around zero weight in the ensemble, while three experts in the non-game setting are assigned weights near zero at various points in the series. 

\begin{figure}[ht!]
    \centering
    \begin{tabular}{cc}
        \begin{subfigure}[b]{0.45\linewidth}
            \centering
            \includegraphics[width=\linewidth]{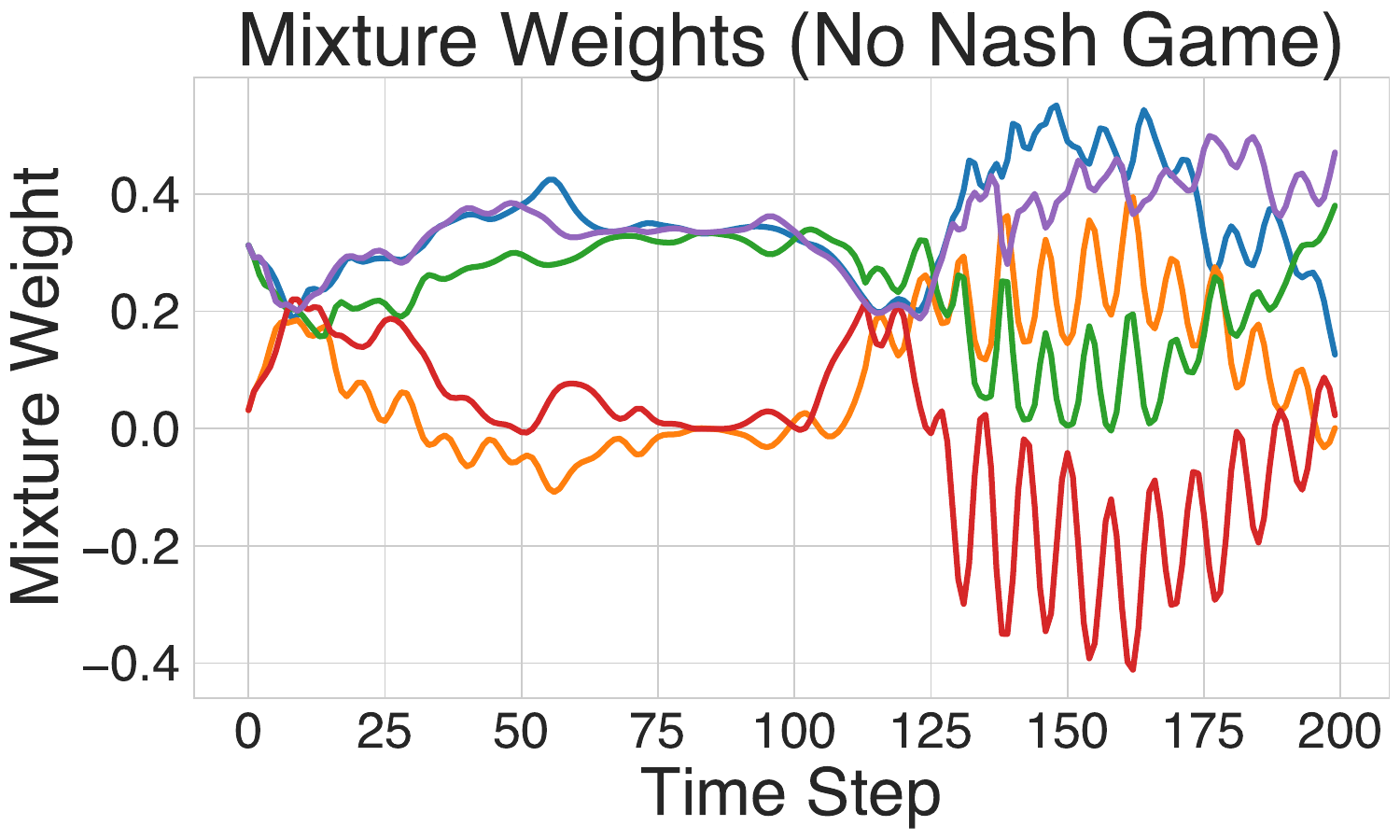}
            \caption{Mixture weights with transformer models.}
            \label{transformer_concept_mixtures_no_nash}
        \end{subfigure} &
        \begin{subfigure}[b]{0.45\linewidth}
            \centering
            \includegraphics[width=\linewidth]{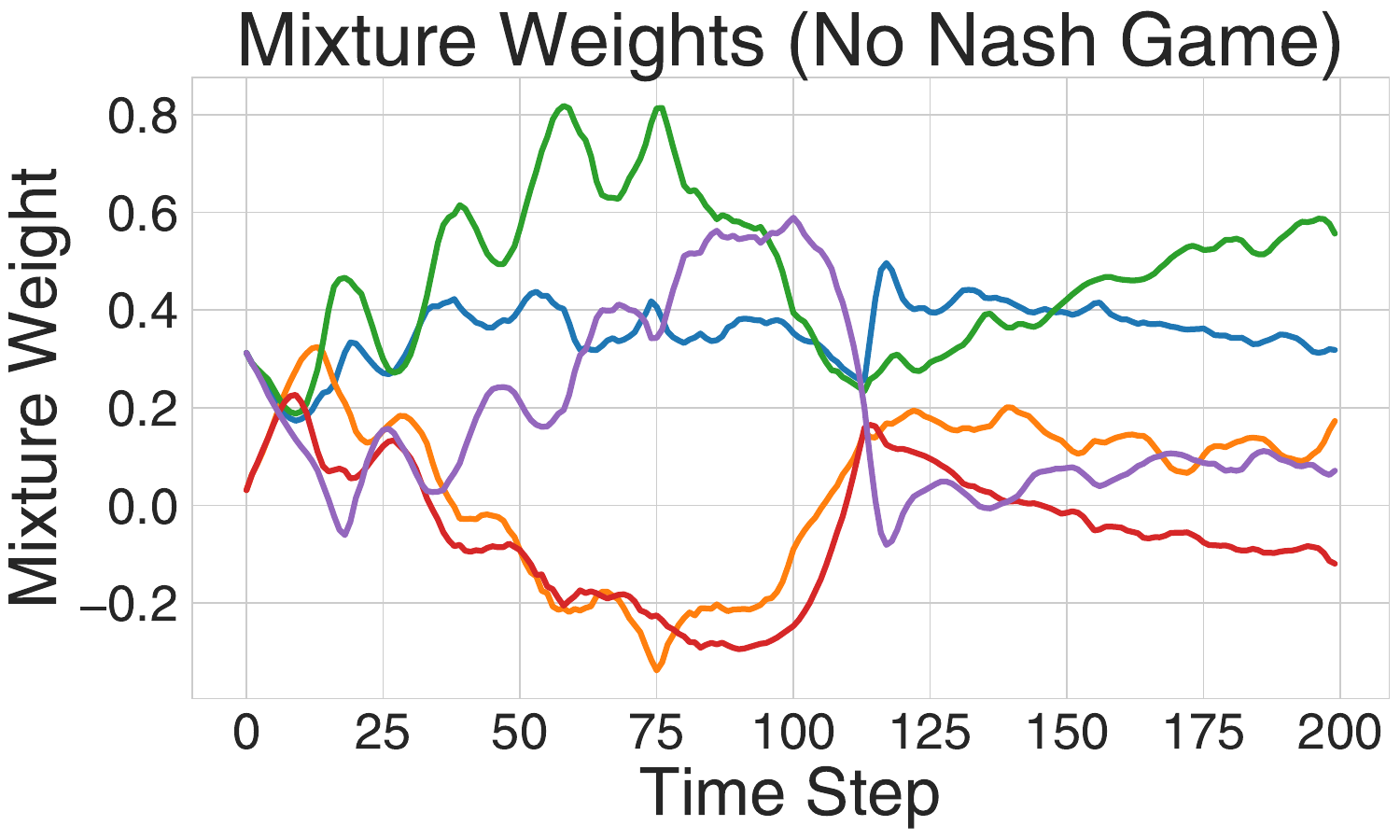}
            \caption{Mixture weights with ESN models.}
        \end{subfigure} \\
        \begin{subfigure}[b]{0.45\linewidth}
            \centering
            \includegraphics[width=\linewidth]{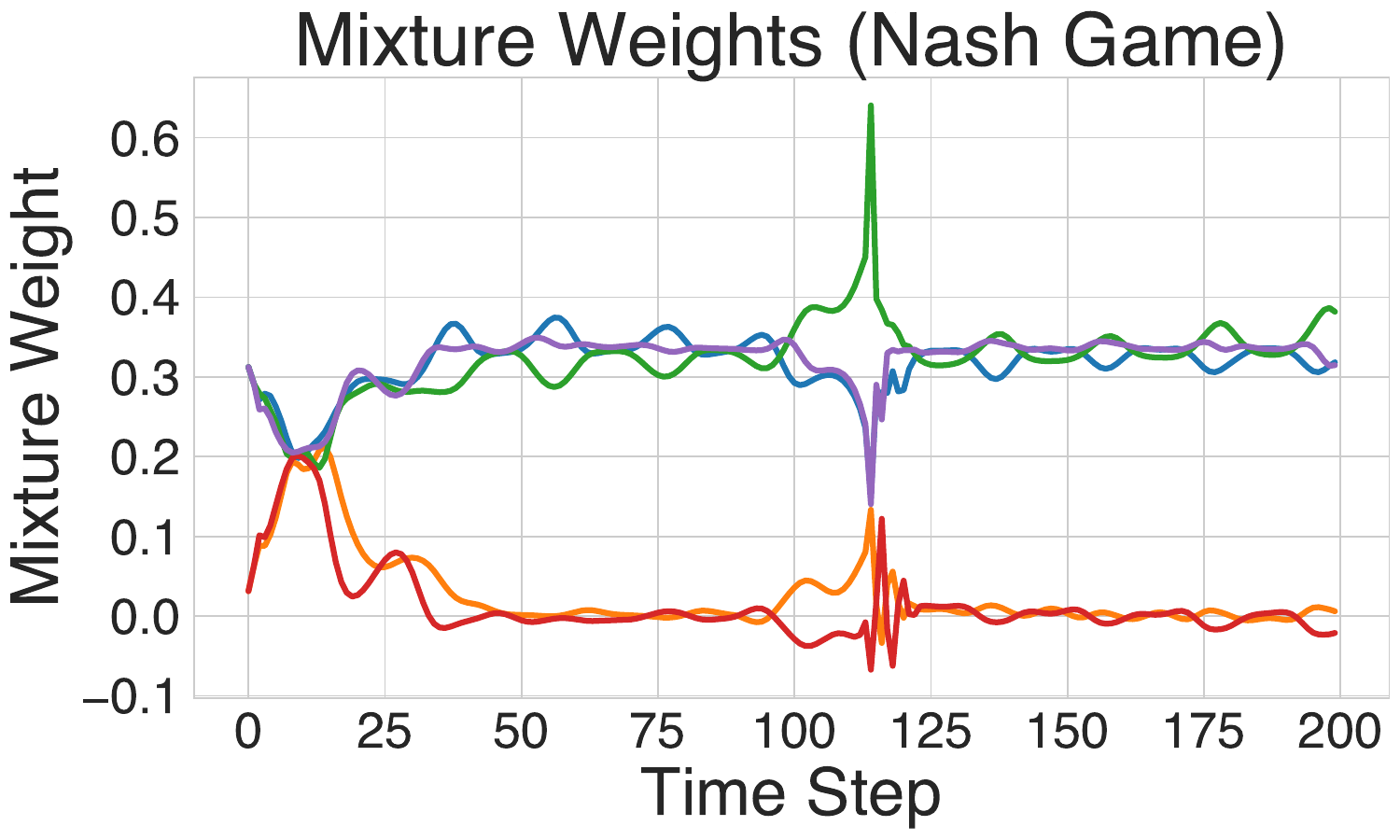}
            \caption{Mixture weights with transformer models.}
            \label{transformer_concept_mixtures_nash}
        \end{subfigure} &
        \begin{subfigure}[b]{0.45\linewidth}
            \centering
            \includegraphics[width=\linewidth]{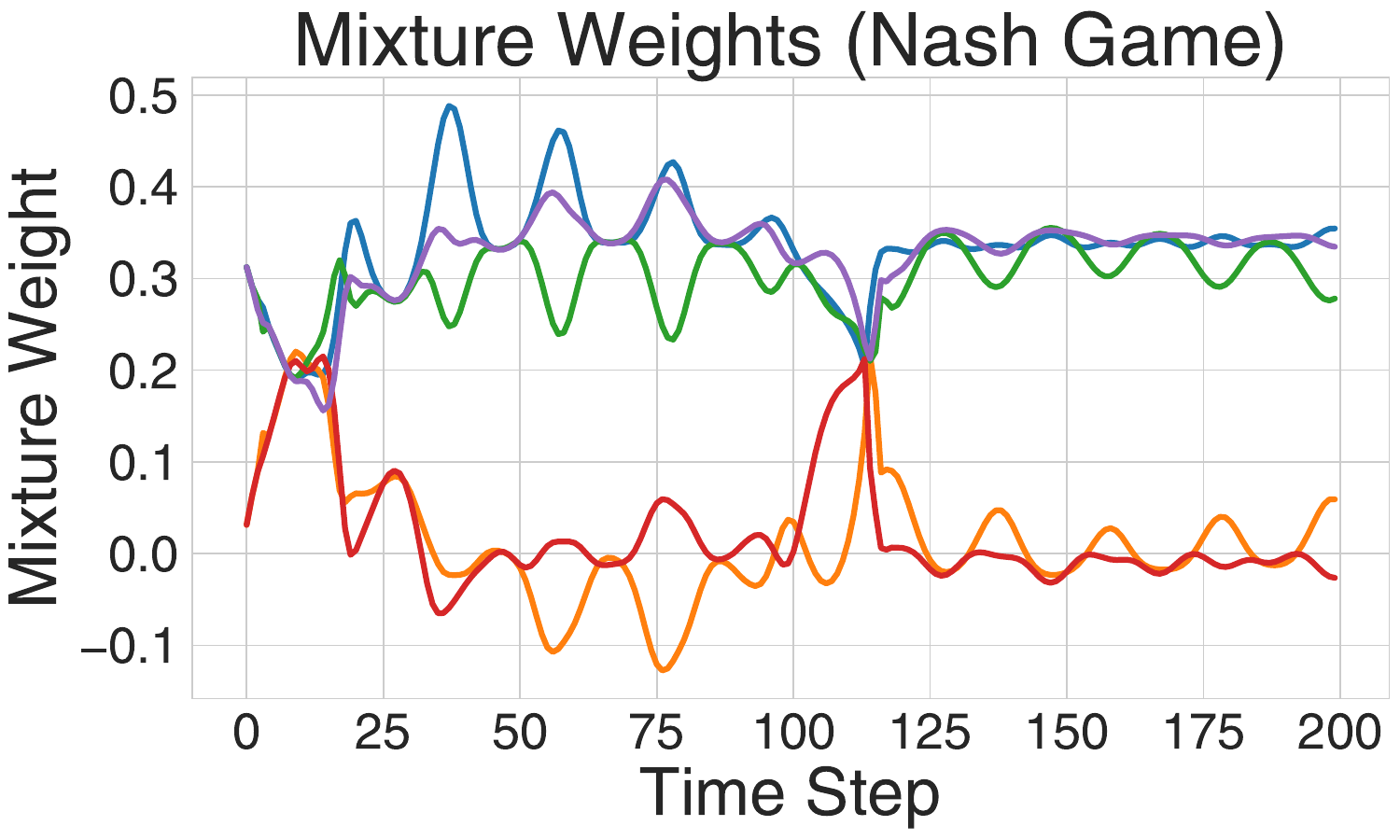}
            \caption{Mixture weights with ESN models.}
        \end{subfigure}
    \end{tabular}
\caption{Concept Drift Dataset Mixture Weights}
\label{fig:mixture_weights_concept_drift}
\end{figure}

As seen in the mixture weights graphs of Figure \ref{fig:mixture_weights_periodic_signal} for the periodic signal, the non-game ESN server settles into a stable and repeating set of weights after a warm-up period. On the same dataset, the non-game transformer server assigns highly volatile mixture weights with rapid changes between positive and negative weights assigned to a single expert. Meanwhile, the weights with the Nash game are somewhat more adaptive while maintaining stability for both agent types. 

For both datasets, application of Nash synchronization appears to yield more stable mixture weight trajectories and shorter periods of instability. For example, comparing the weights of Figures \ref{fig:mixture_weights_concept_drift}\subref{transformer_concept_mixtures_no_nash} and \subref{transformer_concept_mixtures_nash}, those associated with the Nash game maintain steady states during periods of stable relationships between the input and target variables, while those without the Nash game exhibit larger fluctuations. Similarly, between Figures \ref{fig:mixture_weights_concept_drift}\subref{fig:mixture_weights_periodic_esn_non_game} and \subref{fig:mixture_weights_periodic_esn_nash_game}, the Nash game weights reach their periodic state more quickly and with lower oscillation amplitude than those without Nash synchronization.

\begin{figure}[ht!]
    \centering
    \begin{tabular}{cc}
        \begin{subfigure}[b]{0.45\linewidth}
            \centering
            \includegraphics[width=\linewidth]{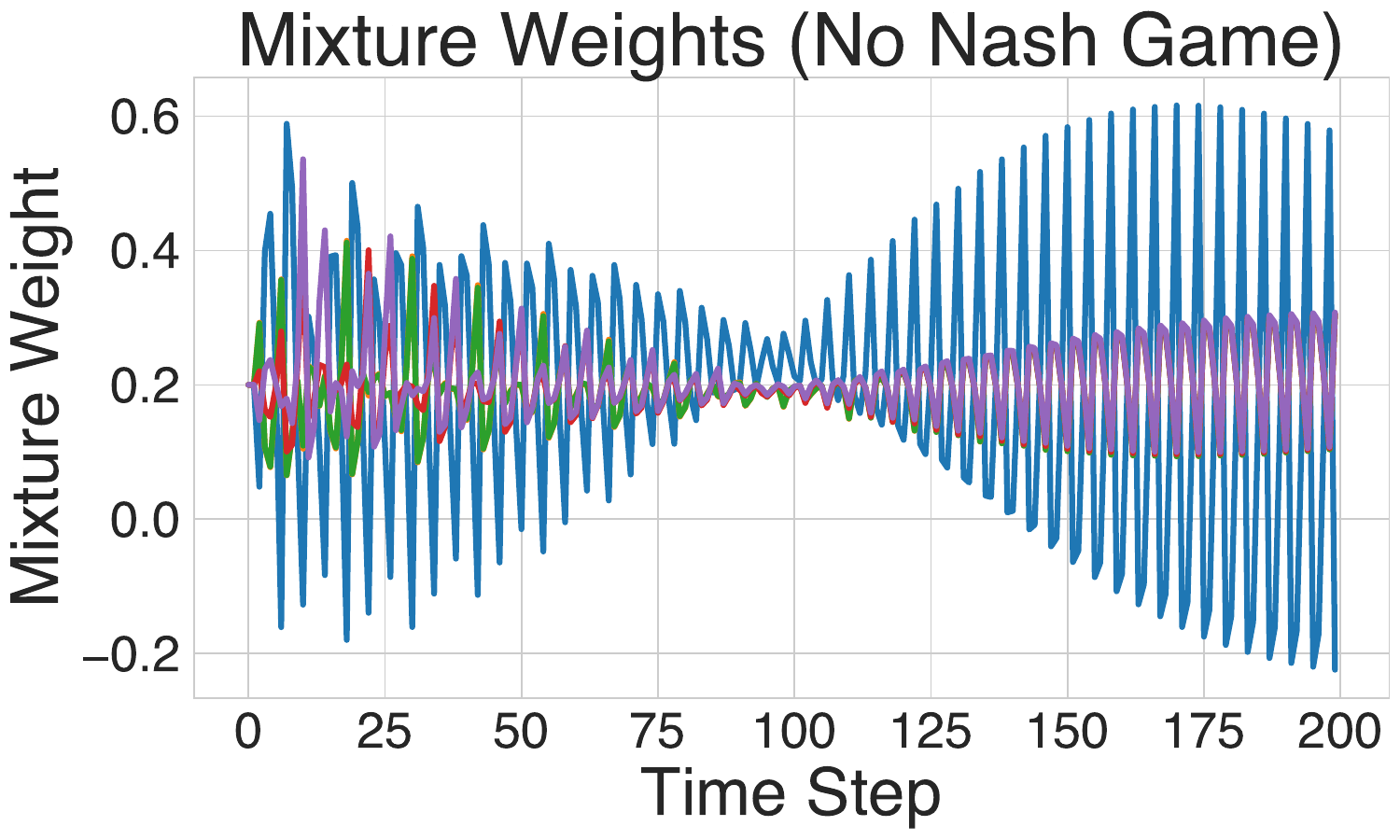}
            \caption{Mixture weights with transformer models.}
        \end{subfigure} &
        \begin{subfigure}[b]{0.45\linewidth}
            \centering
            \includegraphics[width=\linewidth]{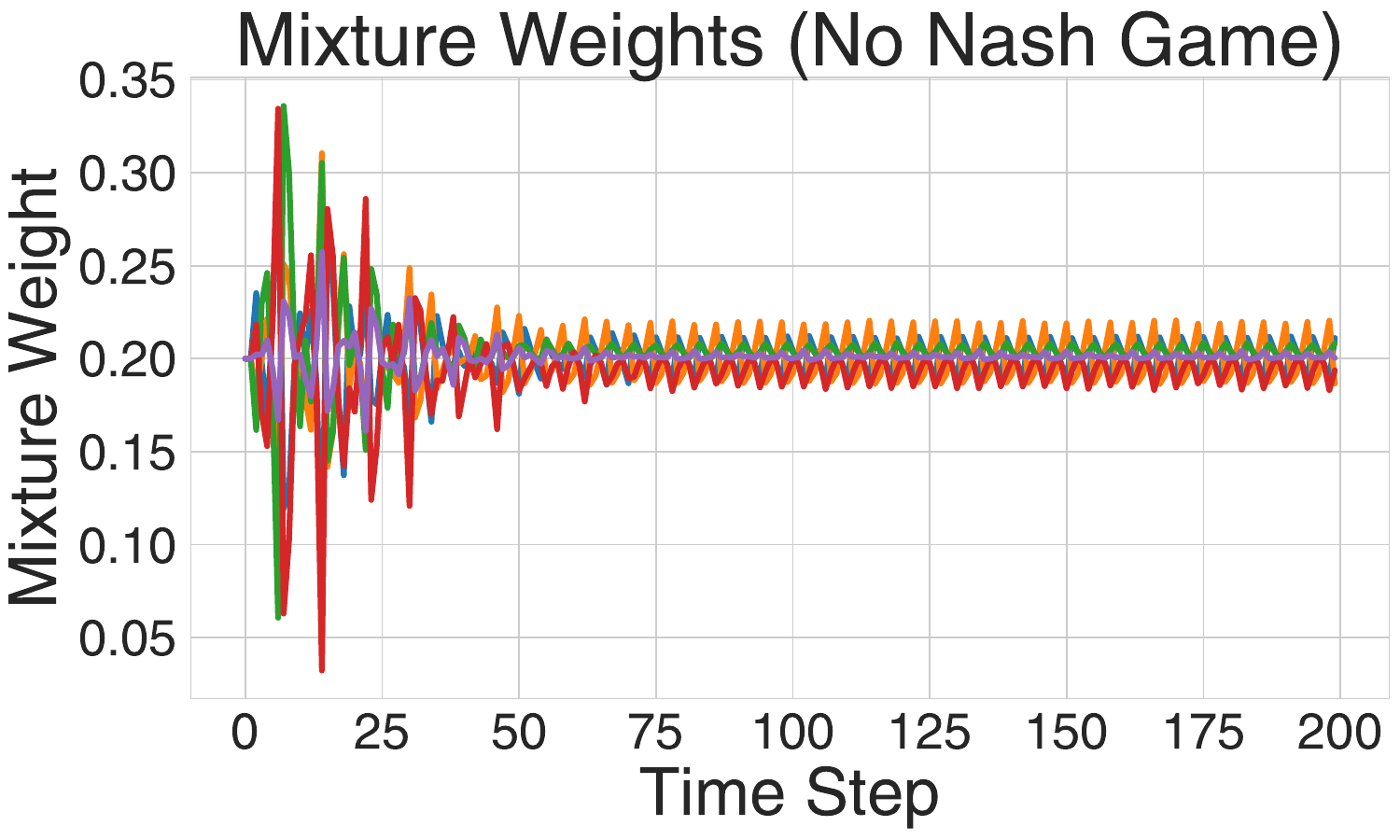}
            \caption{Mixture weights with ESN models.}
            \label{fig:mixture_weights_periodic_esn_non_game}
        \end{subfigure} \\
        \begin{subfigure}[b]{0.45\linewidth}
            \centering
            \includegraphics[width=\linewidth]{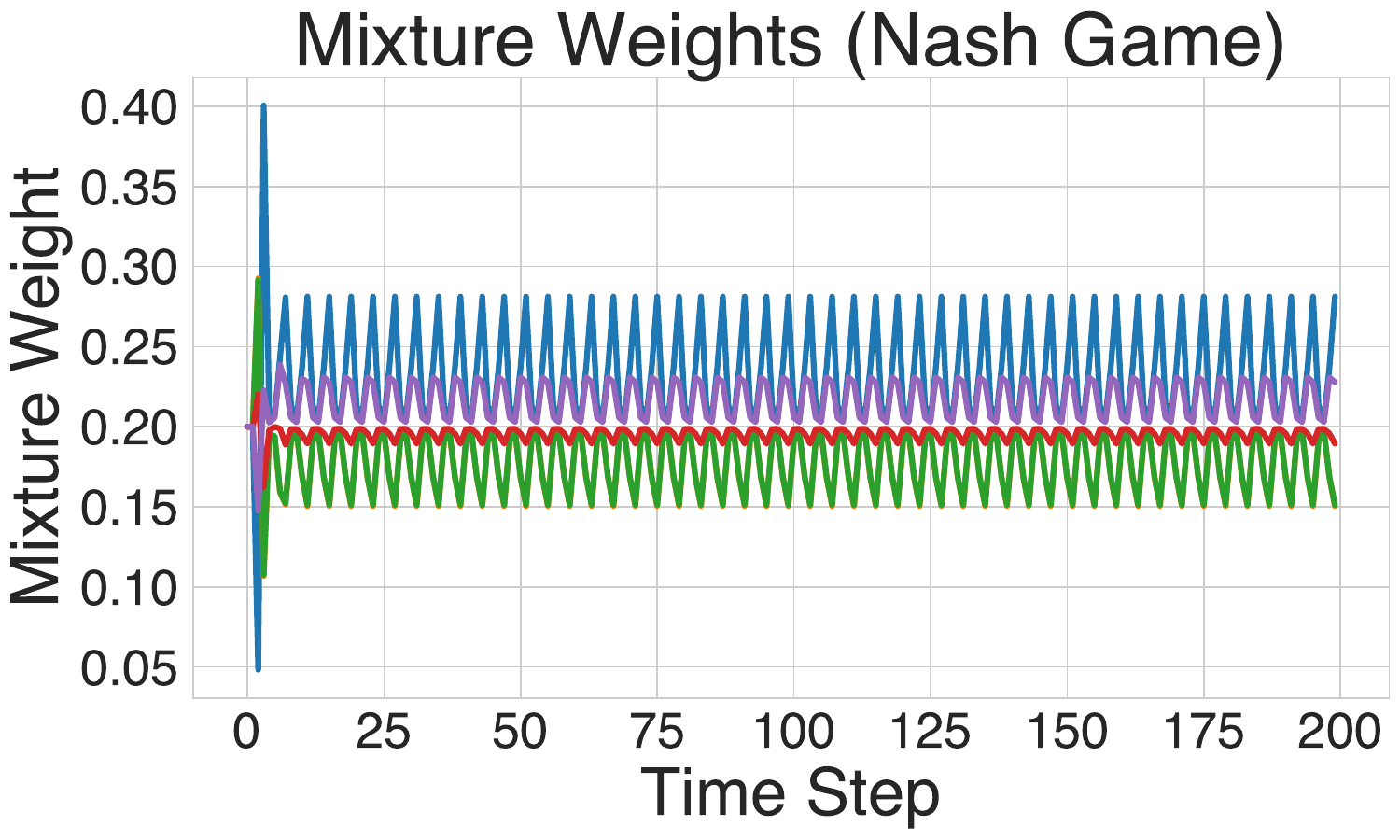}
            \caption{Mixture weights with transformer models.}
        \end{subfigure} &
        \begin{subfigure}[b]{0.45\linewidth}
            \centering
            \includegraphics[width=\linewidth]{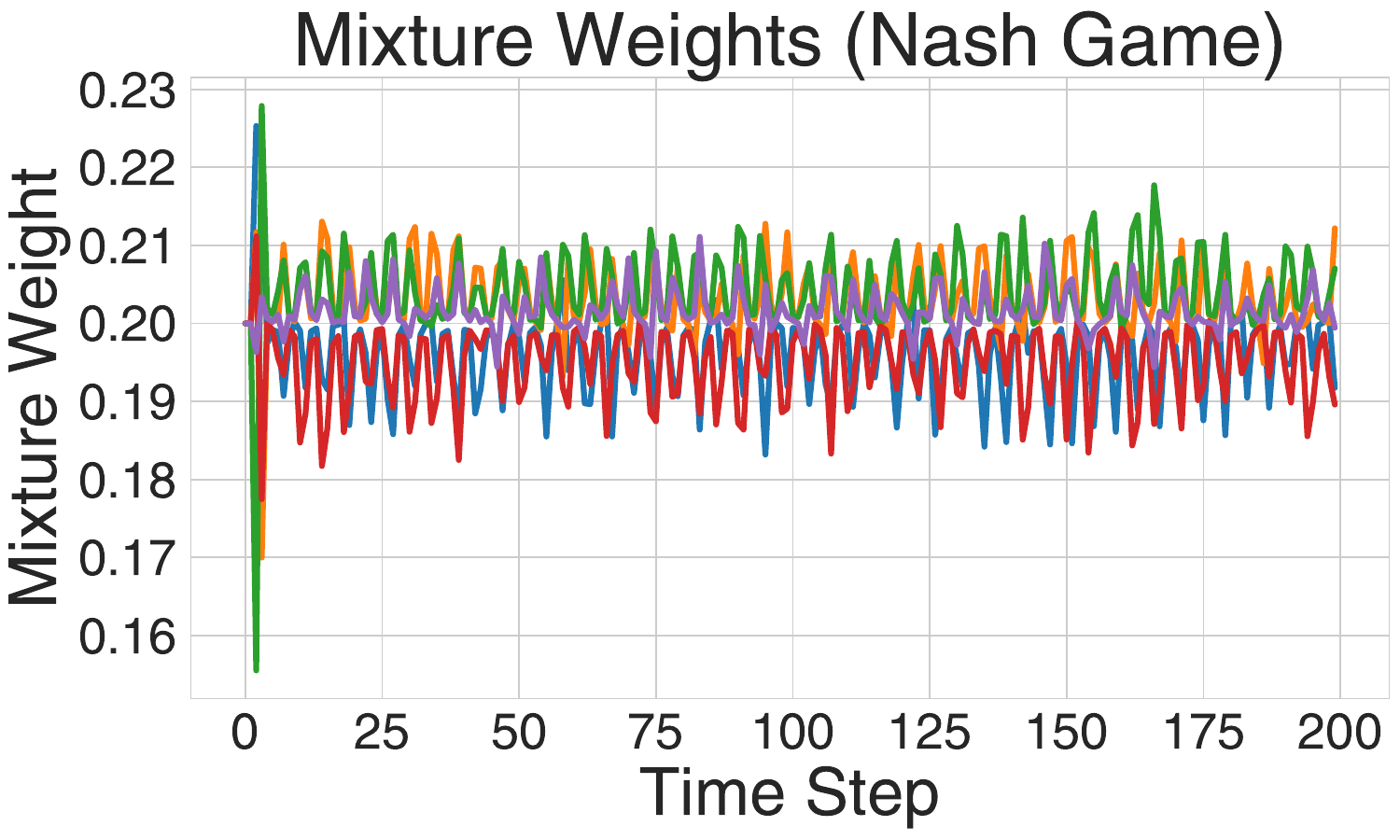}
            \caption{Mixture weights with ESN models.}
            \label{fig:mixture_weights_periodic_esn_nash_game}
        \end{subfigure}
    \end{tabular}
\caption{Periodic Signal Mixture Weights}
\label{fig:mixture_weights_periodic_signal}
\end{figure}

\subsubsection{Effect of Nash game look-back Length}
\label{s:Experiments__ss:Ablation___sss:look-back}

The ``look-back'' length, or number of steps backwards in time considered in the Nash game calculations, is an important hyperparameter, denoted by $T$. This length not only affects prediction accuracy but also the computational cost, as the longer the look-back, the more calculations are necessary in the backward and forward processes of the Nash system. In this section, the effect of increasing look-back length is examined using RFN models applied to the BoC Exchange Rate and ETT datasets in the 5-client setting. The optimal hyperparameters discussed in Appendix \ref{s:OptimalParamterChoices} are applied with the exception of variation in the Game $T$, which takes values in the set $\{2, 3, 4, 5, 6, 7, 8, 9, 10, 15, 25\}$.

\begin{figure}[ht!]
    \centering
    \includegraphics[width=0.49\linewidth]{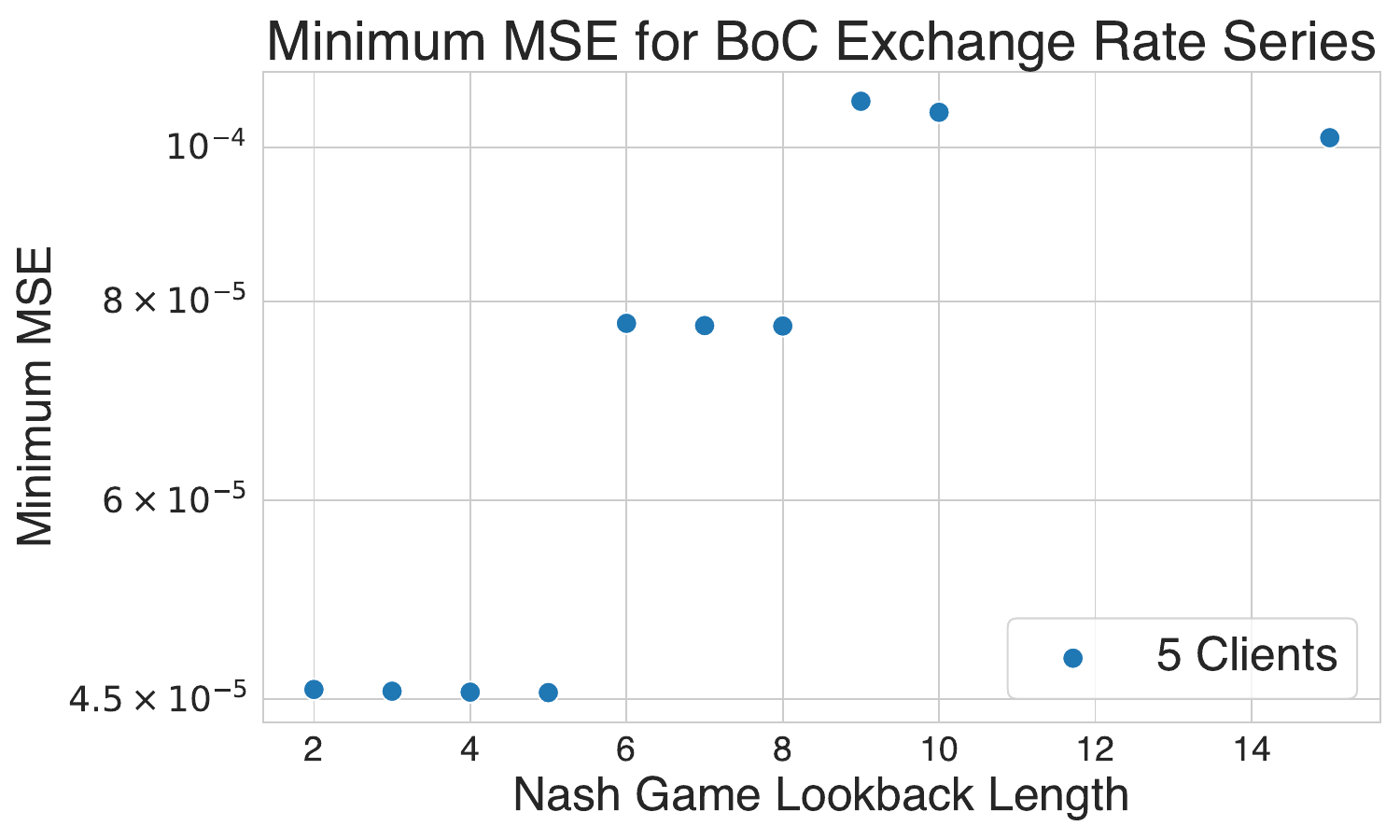}
    \includegraphics[width=0.49\linewidth]{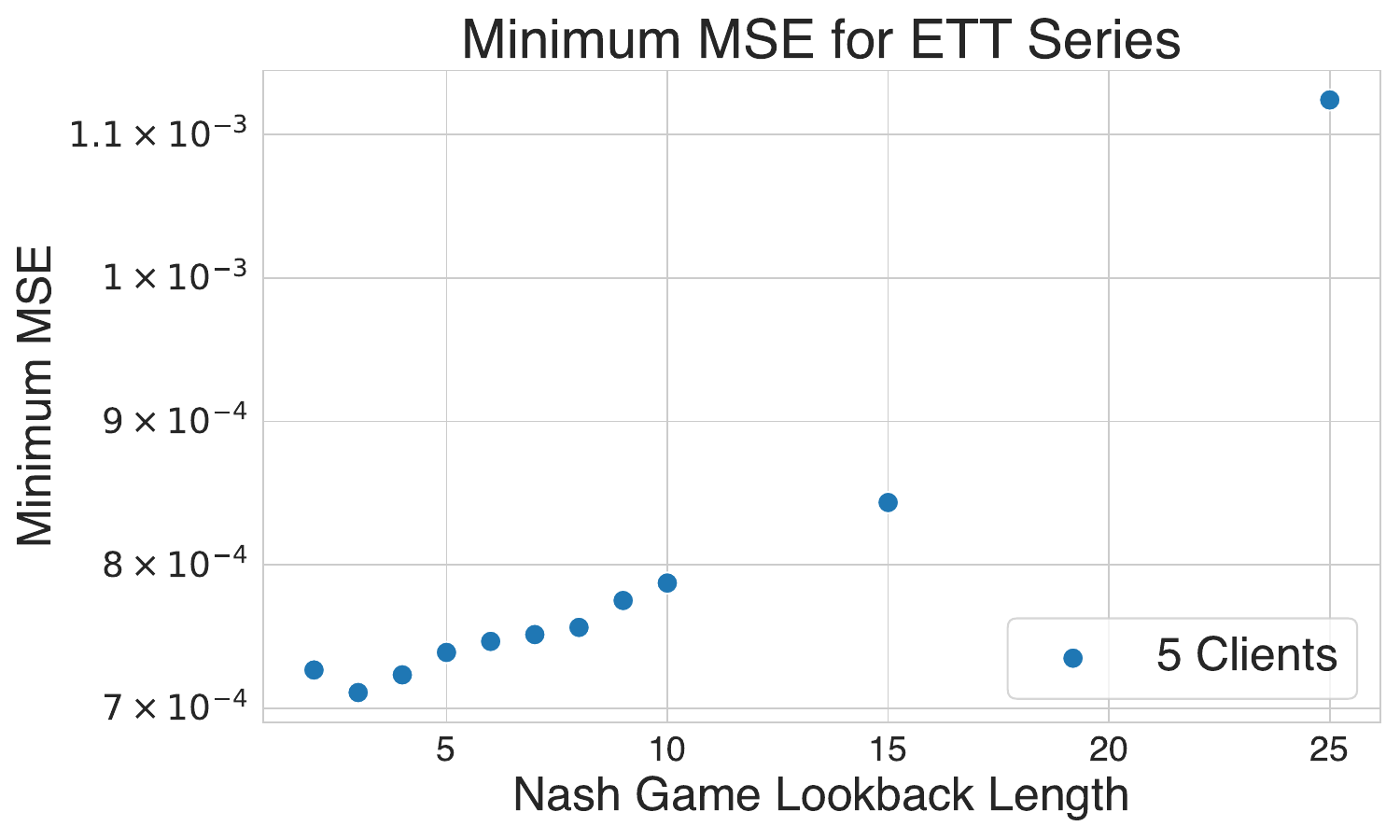}
    \caption{Variations in Minimum MSE for RFN models with varying Nash game look-back lengths for the BoC Exchange Rate (left) and ETT (right) time series using 5 experts.}
    \label{fig:look-back_results}
\end{figure}

Similar to the Client $T$ value for local ridge regression, the optimal look-back length is dataset specific, dependent on the decay parameter $\alpha$, and non-monotonic in nature. Results of the experiments are displayed in Figure \ref{fig:look-back_results}. Departing slightly from previous results, the figure reports the minimum MSE seen across three separate RFN models with different seeds. As look-back length increases, it is likely that the optimal hyperparameters differ significantly from those reported in Appendix \ref{s:OptimalParamterChoices}. As such, the time-series predictions sometimes become unstable with longer look-back lengths. If none of the three runs produced an MSE less than $1.0$, they are excluded from the figure. The optimal look-back length for the BoC Exchange Rate and ETT time series are $5$ and $3$, respectively, with the minimum MSE growing steadily thereafter. Encouragingly, there are other look-back values near the optimum with good performance.

\subsubsection{Nash game Frequency and Runtime/Performance Trade-off}
\label{s:Experiments__ss:Ablation___sss:Runtime}

Given the computational overhead associated with computing the Nash equilibrium, increasing the frequency with which the Nash computations are performed increases prediction latency. In this section, the performance drop induced by playing the Nash game less frequently is investigated, along with the runtime overhead imposed by each synchronization round. As discussed above, synchronizing predictions at every step is optimal in the experiments, but this comes at a cost of increased runtime. Runtime measurements for each round of Nash synchronization with five agents are shown in Table \ref{tab:runtime}, as well as total runtime.

MSE results when running Nash synchronization every $\tau$ steps are shown in Figure \ref{fig:nash_frequency} with a fixed set of hyperparameters matching those in Appendix \ref{s:OptimalParamterChoices}. Such hyperparameters are likely no longer optimal for different $\tau$ steps, where $\tau \in \{1, 2, 3, 4, 5, 6, 7, 8, 9, 10, 15, 25\}$. As such they no longer outperform the hyperparameter optimized non-Nash predictions. Nonetheless, the relationship between $\tau$ and performance is informative, even in this regime. The divergence from optimal hyperparameters is acutely felt in the 5-client setting for the BoC Exchange Rate dataset. When synchronization happens less often than every four steps, predictions tend towards instability. As such, the results for these experiments are not reported in the figure. A potential way to reduce the reliance of frequent synchronization is to replace previous client predictions with the ideal trajectory computed by the Nash game. Investigation of this improvement is left to future work.

\begin{figure}[ht!]
    \centering
    \includegraphics[width=0.49\linewidth]{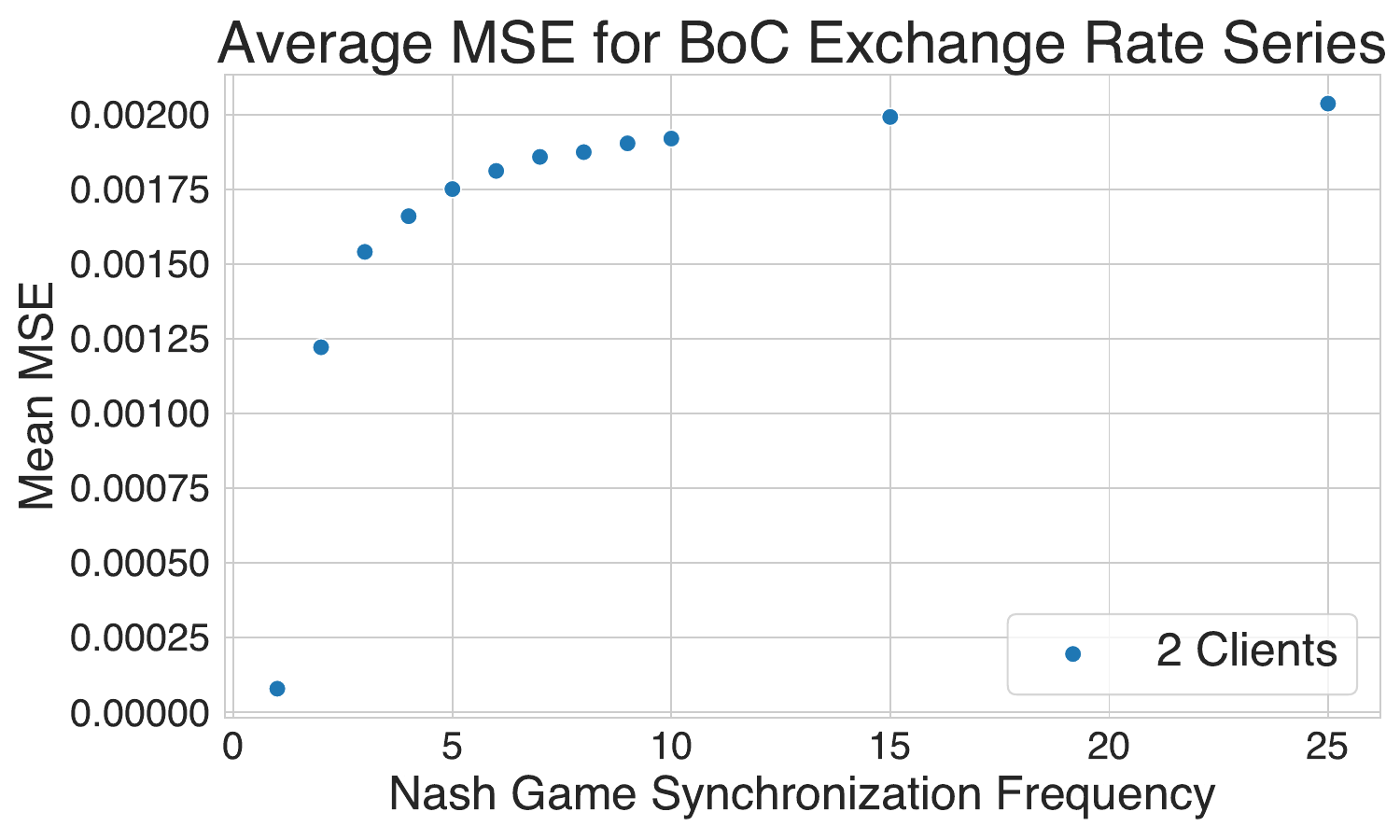}
    \includegraphics[width=0.49\linewidth]{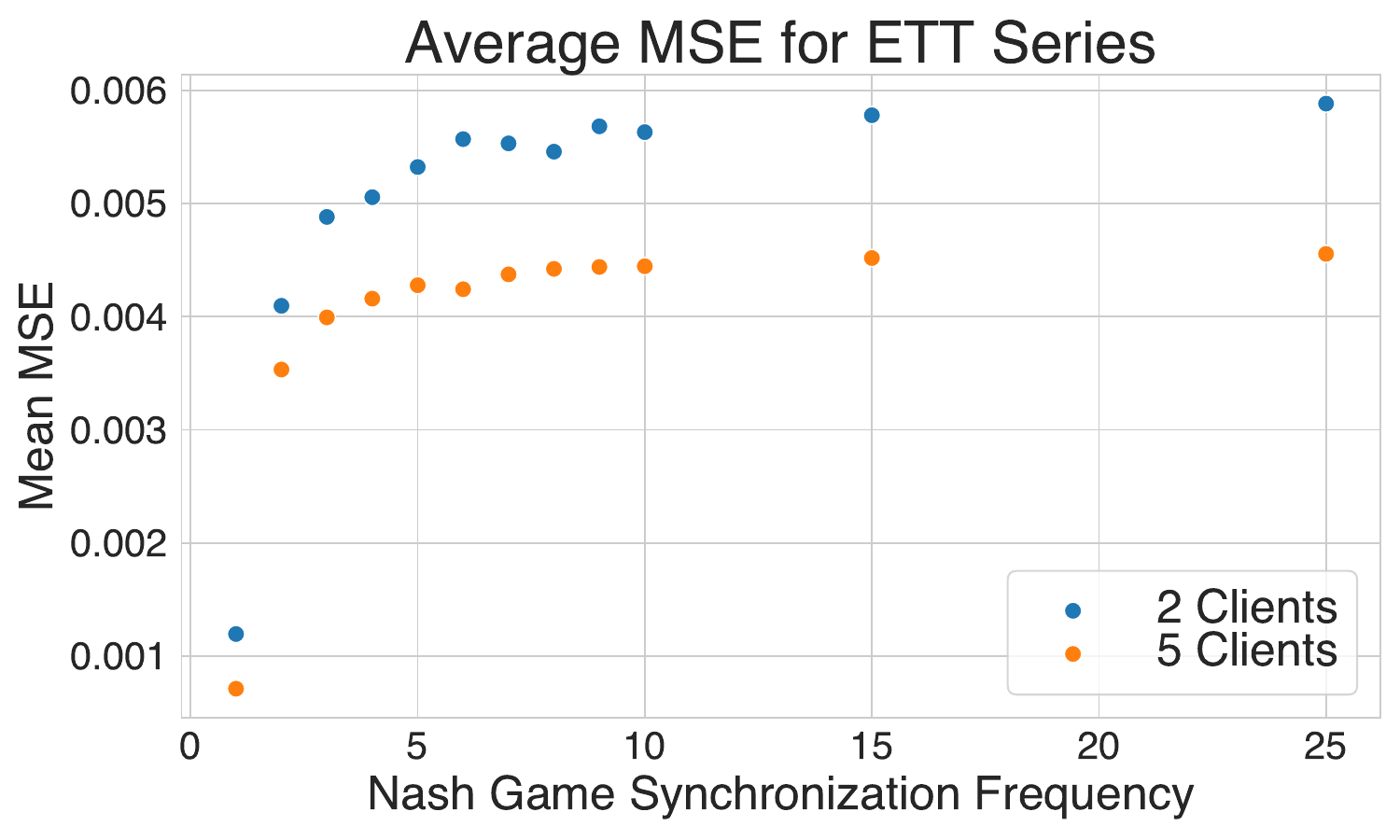}
    \caption{Variations in Mean MSE for RFN models with increasing $\tau$ for the BoC Exchange Rate (left) and ETT (right) time series. Note that the the 5-client setting for the BoC Exchange Rate dataset produces significant instabilities for the chosen hyperparameters when synchronization happens less than every four steps. Thus, it is excluded.}
    \label{fig:nash_frequency}
\end{figure}
 
Comparing the runtime and performance differences between the Nash game and non-game settings provides an estimation of the extra time imposed by the algorithm per round of playing the game. ESN models are expected to incur the largest overhead, as solutions to the Nash system require the approximation of a number of expectations. These approximations are computed with Monte-Carlo simulations using $100$ samples for each expectation when performing synchronization. The computational cost is considerably less for RFN models due to the analytical expressions for such expectations. Pretrained models impose the least runtime overhead when playing the Nash game. Table \ref{tab:runtime} reports the mean runtimes over three random seeds for the Periodic and BoC Exchange Rate time series using ESN, RFN, and transformer models with and without Nash synchronization. 

Total time for RFN model predictions is shorter than that of ESN models by a factor of approximately $40$. Subsequently, transformer models are faster than RFNs by roughly a factor of $2$, making these models a good choice when prediction latency is important. On the other hand, without the Nash game, transformers have the longest runtime compared to other models due to their computational overhead. It should be noted that significant computation optimizations have not been pursued in solving the Nash systems. For example, the Monte Carlo simulations are done sequentially but are embarrassingly parallelizable in practice.

\begin{table}[ht!]
    \centering
    \caption{Runtime measurements in seconds for simulations with five experts. Nash synchronization occurs each step when the technique is applied, and all optimal hyperparameters are used.}
    \begin{tabular}{llcccc}
    \toprule
     & & \multicolumn{2}{c}{Nash game} & \multicolumn{2}{c}{No Nash game} \\
    Model & Dataset & Total (s) & Per Step (s) & Total (s) & Per Step (s) \\
    \midrule
     \multirow{2}{*}{ESN} & Periodic & $1.88\mathrm{e}{\text{+}3}$s & $9.41\mathrm{e}{\text{-}0}$s & $4.24\mathrm{e}{\text{-}1}$s & $ 2.12\mathrm{e}{\text{-}3}$s\\
     & BoC & $2.00\mathrm{e}{\text{+}4}$s & $1.00\mathrm{e}{\text{-}1}$s & $4.24\mathrm{e}{\text{-}0}$s & $2.11\mathrm{e}{\text{-}3}$s\\
     \midrule
     \multirow{2}{*}{RFN} & Periodic & $5.28\mathrm{e}{\text{+}1}$s & $2.64\mathrm{e}{\text{-}1}$s & $3.68\mathrm{e}{\text{-}1}$s & $1.84\mathrm{e}{\text{-}3}$s\\
     & BoC & $5.08\mathrm{e}{\text{+}2}$s & $2.54\mathrm{e}{\text{-}1}$s & $4.07\mathrm{e}{\text{-}0}$s & $2.03\mathrm{e}{\text{-}3}$s \\
     \midrule
     \multirow{2}{*}{Transformer} & Periodic & $1.96\mathrm{e}{\text{+}1}$s & $9.80\mathrm{e}{\text{-}2}$s & $1.67\mathrm{e}{\text{-}0}$s & $8.34\mathrm{e}{\text{-}3}$s\\
     & BoC & $1.59\mathrm{e}{\text{+}2}$s & $7.93\mathrm{e}{\text{-}2}$s & $1.81\mathrm{e}{\text{+}1}$s & $9.03\mathrm{e}{\text{-}3}$s \\
     \bottomrule
    \end{tabular}
    \label{tab:runtime}
\end{table}

It is important to note that a considerable performance advantage when synchronizing with the Nash game may be preserved by performing such synchronizations less frequently if using optimal parameters. Consider the BoC Exchange Rate dataset and an RFN model. With hyperparameters tuned for the specific synchronization frequency, if Nash synchronization occurs every $10$ steps, rather than every step, the runtime is reduced by a factor of more than $7.5$, and the resulting MSE remains well below that of the non-game setting. Specifically, a synchronization frequency of $10$ results in an average MSE of $5.82080\mathrm{e}{\text{-}4}$ over $67.027$s compared with the non-game mean MSE of $3.15237\mathrm{e}\text{-}3$ and average runtime of $4.069$s.

\section{Conclusion and Future Work}

In this work, the proprietary federated learning problem is viewed through a non-cooperative game-theoretic lens (Theorem \ref{thm:NashEqmExistence})
and an efficient and decentralized algorithm (Algorithm \ref{alg:sync_subroutine}) enabling collaboration among black-box agents without revealing their internal structure is derived.  
The method is successfully applied to transformers (Corollary \ref{cor:PreTrainedTransformer}), random feature networks (Corollary \ref{cor:Random_NeuralNetwork}), and echo-state networks (Equation \eqref{eq:reservoir}). Numerical experiments are conducted on challenging real-world and synthetic time series, demonstrating significant improvements in predictive accuracy over existing baselines. 

This work opens as many possible future research questions as it answers. Two such questions which we emphasize are how to solve the static and non-linear settings, beyond the kernelized setup we currently have where the agents only control their linear readout layer in Equation \eqref{eq:DECODER_RESIDUALS}.  We also mention a possible asymptotically large-population version of this problems which we imagine can be approached using the mean-field game toolbox, in a vague analogy with the neural tangent kernel \citep{jacot2018neural}.

\bibliography{Refs}

\appendix
\section{Proofs of Propositions and Theorems}
\label{s:Proofs}

This section contains the proofs of our main results.

\subsection{{Proof of Proposition \ref{prop:localgreedyupdates}}}
\begin{proof}
Observe that the loss function in Equation \eqref{eq:ridge_local} can be re-written as
\allowdisplaybreaks
\begin{align}
\label{eq:redux_ridge}
\begin{aligned}
        & \sum_{s=0}^t\,
            e^{ -\alpha_i(t-s)}
            \,\big\|(y_s-\hat{Y}_s^i) - Z_s^i \beta_t^i \big\|^2
            +
            \gamma_i\big\|\beta_t^i\big\|^2
    \\
    & =
        \sum_{s=0}^t\,
            \,\big\| e^{-\alpha_i(t-s)/2}
                (y_s-\hat{Y}_s^i)
             - (e^{-\alpha_i(t-s)/2} Z_s^i) \beta^i_t \big\|^2
     +
            \gamma_i\big\|\beta_t^i\big\|^2  \\ 
    & = \big( \bar{y}^i_t - X^i_t \beta^i_t \big)^\top \big( \bar{y}^i_t - X^i_t \beta^i_t \big) 
     + \gamma_i \beta^{i\top}_t \beta^i_t,
\end{aligned}
\end{align}
and the last line of Equation \eqref{eq:redux_ridge} is simply a standard ridge-regression problem.  Since this is a strictly convex problem, there is a unique $\beta_t^i\in \mathbb{R}^{d_z}$ minimizing the last line of Equation\eqref{eq:redux_ridge}.  The desired closed-form expression is given in \citep[Equation (3.44)]{HastieTibshiraniFriedman_2001ESLBook}.
\end{proof}

\subsection{{Proof of Theorem \ref{thm:OptimalWeights}}}

\begin{proof}
Define $A\eqdef 2\,
 ( \hat{\mathbf{Y}}^\top_t \hat{\mathbf{Y}}_t + \kappa I_N)$, where $I_N$ is the $N\times N$ identity matrix, and set $b\eqdef 2  \big(
y^\top_t  \hat{\mathbf{Y}}_t\big)^{\top}$. 
The matrix $\hat{\mathbf{Y}}^\top_t \hat{\mathbf{Y}}_t$ is positive semi-definite.  Since $\kappa>0$ then $\det(\kappa I_N)=\kappa^N>0$ and thus, the matrix $A$ is positive definite.  

Let $B\eqdef \mathbf{1}_N\in \mathbb{R}^{N\times 1}$, where for $i=1,\dots,N$ we have $B_i=1$.  
Note that $B$ has rank $1\le N$.

Observe that the loss function $\Big\| y_t  - \sum_{j=1}^N w^j_t \hat{Y}^j_t \Big\|^2 + \sum_{j=1}^N \kappa\, | w^j_t|^2 $ can be written as
\allowdisplaybreaks
\begin{align}
    \Big\| y_t  - \sum_{j=1}^N w^j_t \hat{Y}^j_t \Big\|^2 
    + 
    \sum_{j=1}^N \kappa\, | w^j_t|^2 
= &
    \frac1{2}\, 
    w^\top A w 
  - b^{\top} w 
   + y_t^\top y_t
\eqdef Q(w)
.
\end{align}
Therefore, the server's optimization problem in Equation \eqref{eq:server} can be written as a quadratic program with linear constraint 
\begin{equation}
\label{eq:server__QPLC}
\begin{aligned}
    \min_{w\in \mathbb{R}^N}\,
    &
    \frac1{2} \,
    w^{\top} A w - b^{\top} w
\\
\mbox{s.t.} & \,\,
B^{\top} w = \eta .
\end{aligned}
\end{equation}
Since $A$ is positive definite and since $\operatorname{rank}(B)=1\le N$ then the quadratic program in Equation \eqref{eq:server__QPLC} is a non-degenerate quadratic program with a feasible linear constraint set.
By \citep[Proposition 6.3]{gallier2020linear}, there exists a unique solution to the constrained quadratic program and it is given by the Karush-Kuhn-Tucker (KKT) system
\begin{align}
\label{eq:KKT_1}
A\,w + B\lambda & = b
\\
\label{eq:KKT_2}
B^{\top} w & = \eta .
\end{align}
Since $A$ is positive definite, it is invertible.  Therefore, Equation \eqref{eq:KKT_1} allows us to isolate $w$
\begin{align}
\label{eq:KKT_1__manip1}
w  = A^{-1}\,(b - B\lambda)
.
\end{align}
Plugging the expression for $w$, which we just obtained in Equation \eqref{eq:KKT_1__manip1}, back into Equation \eqref{eq:KKT_2} yields
\allowdisplaybreaks
\begin{align}
\nonumber
\eta
= & B^{\top} \big(A^{-1}\,(b - B\lambda)\big)
\\
\nonumber
\therefore \,
\eta
= & B^{\top} \big(A^{-1}\,b - A^{-1}B\lambda \big) 
= B^{\top}A^{-1}\,b - \lambda B^{\top}A^{-1}B
\\
\nonumber
\therefore \,
    \lambda B^{\top}A^{-1}B
= & 
    B^{\top}A^{-1}\,b-
    \eta
\\
\label{eq:KKT_2__manip1}
\therefore \,
    \lambda 
= & 
    \frac{
        B^{\top}A^{-1}\,b-
        \eta
    }{
        B^{\top}A^{-1}B
    }
.
\end{align}
Plugging the value for $\lambda$, just obtained in Equation \eqref{eq:KKT_2__manip1}, back into Equation \eqref{eq:KKT_1__manip1} yields the optimizer $w^{\star}\in \mathbb{R}^N$ to Equation \eqref{eq:server__QPLC}
\allowdisplaybreaks
\begin{align}
\nonumber
w^{\star}   & = A^{-1}\,\big(b - B\lambda\big)
\\
\nonumber
& = 
 A^{-1}\,b -  \lambda\, A^{-1}B
\\
\nonumber
& = 
 A^{-1}\,b -  
 \biggl(
    \frac{
        B^{\top}A^{-1}\,b-
        \eta
    }{
        B^{\top}A^{-1}B
    }
 \biggr)
 \, A^{-1}B
\\
\label{eq:KKT_1__manip2}
\therefore \,
w^{\star} & = 
 A^{-1}\,
 \Biggl(
     b -  
     \biggl(
        \frac{
            B^{\top}A^{-1}\,b-
            \eta
        }{
            B^{\top}A^{-1}B
        }
     \biggr)B
 \Biggr)
.
\end{align}
Plugging in our definitions for $A$, $b$, and $B$ back into Equation \eqref{eq:KKT_1__manip2} yields the conclusion.
\end{proof}

\subsection{{Proof of Theorem \ref{thm:NashEqmExistence}}}
\label{s:Proofs__ss:thm:NashEqmExistence}
\begin{proof} 
We employ dynamic programming to establish the result via backward induction~\cite[Corollary 6.1]{BO1998}.  
Let $\hat{Y}_t \eqdef [\hat{Y}^{1\top}_t, \dots, \hat{Y}^{N\top}_t]^\top$ denote the concatenated prediction vector of the $N$ agents. 
Define $V_i(t,\hat{\mathbf{y}})$ as the value function of agent $i$ at time $t$, given $\hat{Y}_t=\hat{\mathbf{y}}$, corresponding to the Nash game~\eqref{eq:ENCODER}, 
\eqref{eq:DECODER_RESIDUALS}, and~\eqref{eq:Objective_nth_expert}. 
From \eqref{eq:Objective_nth_expert}, the value functions at time $t=T$ with $\hat{Y}_T=\hat{\mathbf{y}}$ satisfy 
$V_i(T, \hat{\mathbf{y}}) = 0$ and thus take the form 
\begin{align} 
 V_i(T, \hat{\mathbf{y}}) = \hat{\mathbf{y}}^\top  P_i(T) \hat{\mathbf{y}}  
  + 2 S_i(T)^\top \hat{\mathbf{y}} + r_i(T) , \quad i = 1, \dots, N , 
\end{align} 
with $P_i(T)=0$ and $S_i(T)=0$ determined by \eqref{eqn:Pi} and \eqref{eqn:Si}.

Assume by induction that for some $t\in \{ 1, \dots, T-1\}$ and for all $s\in\{t +1, \dots , T\}$, the value functions $\{V_i(s, \hat{\mathbf{y}})\}_{i=1}^N$ at time $s$ with $\hat{Y}_s = \hat{\mathbf{y}}$ take the affine-quadratic form 
\begin{align} 
  V_i(s, \hat{\mathbf{y}} ) =  \hat{\mathbf{y}}^\top  P_i(s) \hat{\mathbf{y}}  
  + 2 S_i(s)^\top \hat{\mathbf{y}} + r_i(s) , \quad i = 1, \dots, N , \notag  
 \end{align} 
with $P_i(s)$ and $S_i(s)$ determined by \eqref{eqn:Pi} and \eqref{eqn:Si} on the time span $s\in \{t+1, \dots, T\}$. 
We show that at time $t$ with $\hat{Y}_t=\hat{\mathbf{y}}$, the equilibrium strategy is given by $\boldsymbol{\beta}_t=[\beta^{1\top}_t, ..., \beta^{N\top}_t ]^\top=\mathbf{G}(t)\hat{\mathbf{y}} + \mathbf{H}(t)$, 
and the value functions $\{V_i(t, \hat{\mathbf{y}})\}_{i=1}^N$ take the affine-quadratic form 
\begin{align} 
  V_i(t, \hat{\mathbf{y}} ) =  \hat{\mathbf{y}}^\top  P_i(t) \hat{\mathbf{y}}  
  + 2 S_i(t)^\top \hat{\mathbf{y}} + r_i(t) , \quad i = 1, \dots, N , 
  \label{vi(t,y)ansatz}
 \end{align} 
with $P_i(t)$ and $S_i(t)$ determined by \eqref{eqn:Pi} and \eqref{eqn:Si}, respectively.

By the dynamic programming principle and the induction hypothesis, 
$\{V_i(t, \hat{\mathbf{y}})\}_{i=1}^N$ satisfy the system of dynamic programming equations:  
\allowdisplaybreaks
\begin{align} 
V_i(t, \hat{\mathbf{y}}) 
= & \min_{\beta^i} \mathbb{E} \Big\{  e^{-\alpha_i (T-1-t)} 
 \Big[ \Big\| y_{t+1} - \mathbf{w}_t^\top \big( \hat{\mathbf{y}} + \sum_{j=1}^N \mathbf{e}_j Z^j_t \beta^j \big)  \Big\|^2 
  + \gamma_i \| \beta^i \|^2   \Big]   \notag \\ 
  & \hspace{5.5cm} 
  +   V_i(t+1, \hat{\mathbf{y}} + \sum_{j=1}^N  \mathbf{e}_j Z^j_t \beta^j ) 
   \Big| \hat{Y}_t = \mathbf{\hat{y}}  \Big\} \notag \\ 
  = & \min_{\beta^i} \mathbb{E} 
  \Big\{   e^{-\alpha_i (T-1-t)} 
 \Big[ \Big\| y_{t+1} - \mathbf{w}_t^\top \big( \hat{\mathbf{y}} + \sum_{j=1}^N \mathbf{e}_j Z^j_t \beta^j \big)  \Big\|^2 
  + \gamma_i \| \beta^i \|^2   \Big]    \notag \\ 
 & \hspace{2cm}  +   \Big[ \hat{\mathbf{y}} + \sum_{j=1}^N \mathbf{e}_j Z^j_t  \beta^j  \Big]^\top P_i(t+1) 
  \Big[ \hat{\mathbf{y}} + \sum_{j=1}^N \mathbf{e}_j Z^j_t  \beta^j  \Big] \notag \\ 
 &   \hspace{3.5cm}  + 2    S_i(t+1)^\top \Big[ \hat{\mathbf{y}} + \sum_{j=1}^N \mathbf{e}_j Z^j_t  \beta^j  \Big] 
  +   r_i(t+1)      \Big\} , \quad 
   i = 1, \dots, N , 
  \label{DPeqnVi(t)}  
\end{align} 
where the condition $\hat{Y}_{t}=\hat{\mathbf{y}}$ is removed from the expectation because $\{Z^j_{t}\}_{j=1}^N$ are independent of $\hat{Y}_{t}$. 
Since the objective~\eqref{DPeqnVi(t)} to be minimized is strictly convex in $\boldsymbol{\beta}=[\beta^{1 \top}, \dots, \beta^{N \top} ]^\top$, the first-order necessary conditions for minimization are also sufficient 
and therefore we have the unique set of equations  
\begin{align} 
& \mathbb{E} \Big\{   - e^{-\alpha_i (T-1-t)} Z^{i\top}_t \mathbf{e}_i^\top \mathbf{w}_t 
 \Big[ y_{t+1} - \mathbf{w}_t^\top \big( \hat{\mathbf{y}} + \sum_{j=1}^N \mathbf{e}_j Z^j_t \beta^j \big)   \Big] 
 + e^{-\alpha_i (T-1-t)} \gamma_i \beta^i    \notag \\  
&  \hspace{3cm} +  Z^{i\top}_t \mathbf{e}_i^\top P_i(t+1) \Big( \hat{\mathbf{y}} + \sum_{j=1}^N  \mathbf{e}_j Z^j_t \beta^j \Big) 
  + Z^{i\top}_t \mathbf{e}_i^\top S_i(t+1) \Big\}  =0 , 
  \quad i = 1, \dots, N , \notag 
\end{align} 
which can be written in the compact form  
\begin{align} 
 &  \big[ e^{-\boldsymbol{\alpha}(T -1 - t)}  \boldsymbol{\Gamma}  
  + e^{-\boldsymbol{\alpha}(T -1 - t)} \widehat{\mathbf{A}}(t) 
   + \mathbf{A}(t) \big] \boldsymbol{\beta}_t  \nonumber \\ 
   & \hspace{10ex}+ \mathbf{B} (t) \hat{\mathbf{y}} + \mathbf{C}(t) 
  - e^{-\boldsymbol{\alpha}(T-1-t)} \mathbf{D}^\top(t) \mathbf{w}_t (y_{t+1} - \mathbf{w}_t^\top  \hat{\mathbf{y}} ) 
 = 0. 
\end{align} 
Since \eqref{eqn:Pi} admits a solution on $\{t+1, \dots, T\}$ by the hypothesis of the theorem, it follows from the definitions of $\boldsymbol{\Gamma}$, $\mathbf{A}(\cdot)$ and $\widehat{\mathbf{A}}(\cdot)$ in Table \ref{tab:matricesEqnPiSi} that 
$ e^{-\boldsymbol{\alpha} (T-1-t) } \boldsymbol{\Gamma} 
  + e^{-\boldsymbol{\alpha}(T -1 - t)} \widehat{\mathbf{A}}(t)  + \mathbf{A}(t)$ is positive definite. 
Consequently, the equilibrium strategy $\boldsymbol{\beta}_t=[\beta^{1 \top}_t, \dots, \beta^{N \top}_t ]^\top$ at time $t$ is then given by 
\begin{align} 
 \boldsymbol{\beta}_t = &  \big[ e^{-\boldsymbol{\alpha} (T-1-t) } \Gamma 
  + e^{-\boldsymbol{\alpha}(T -1 - t)} \widehat{\mathbf{A}}(t)  + \mathbf{A}(t) \big]^{-1} \cdot \notag \\  
& \qquad \big[  e^{-\boldsymbol{\alpha}(T-1-t)} \mathbf{D}^\top(t) \mathbf{w}_t (y_{t+1} - \mathbf{w}_t^\top  \hat{\mathbf{y}} ) 
- \mathbf{B}(t) \hat{\mathbf{y}} - \mathbf{C}(t) 
  \big]  \notag \\ 
= & \mathbf{G}(t) \hat{\mathbf{y}} + \mathbf{H}(t).   
\label{eqmbetat}   
\end{align} 
By substituting \eqref{eqmbetat} into \eqref{DPeqnVi(t)}, we obtain that $\{ V_i(t, \hat{\mathbf{y}}) \}_{i=1}^N$ take the affine-quadratic form \eqref{vi(t,y)ansatz}; that is, for $i=1,\dots,N$:
\begin{align} 
 & \hat{\mathbf{y}}^\top  P_i(t) \hat{\mathbf{y}}  
  + 2 S_i(t)^\top \hat{\mathbf{y}} + r_i(t) \notag \\ 
  \qquad&= \big[ \mathbf{G}(t) \hat{\mathbf{y}} + \mathbf{H}(t) \big]^\top  
 \big[ \widehat{\mathbf{A}}(t) + \mathbf{D}_i(t) + e^{-\alpha_i(T-1-t)} 
  \gamma_i \widehat{\mathbf{e}}_i \widehat{\mathbf{e}}_i^\top \big]  \big[ \mathbf{G}(t) \hat{\mathbf{y}} + \mathbf{H}(t) \big]  
 \notag \\ 
 & \qquad+ 2 \big[  - e^{-\alpha_i(T-1-t)} (y_{t+1} - \mathbf{w}_t^\top \hat{\mathbf{y}} )^\top \mathbf{w}_t^\top  
  + \hat{\mathbf{y}}^\top P_i(t+1)  + S_i^\top(t+1)  \big] 
  \mathbf{D}(t) 
   \big[ \mathbf{G}(t) \hat{\mathbf{y}} + \mathbf{H}(t) \big] 
 \notag \\ 
 & \qquad+ e^{-\alpha_i(T-1-t)} \big\| y_{t+1} - \mathbf{w}_t^\top \hat{\mathbf{y}} \big\|^2 + \hat{\mathbf{y}}^\top P_i(t+1) \hat{\mathbf{y}} 
 + 2 S_i^\top(t+1) \hat{\mathbf{y}} + r_i(t+1) . 
 \notag 
\end{align} 
Since the above equations hold for all possible $\hat{\mathbf{y}}$, 
by matching the coefficients of the quadratic and linear terms of $\hat{\mathbf{y}}$, we obtain $P_i(t)$ and $S_i(t)$ as determined by \eqref{eqn:Pi} and \eqref{eqn:Si}. 
By induction, we have shown that the equilibrium strategy takes the form \eqref{eqmvecbeta} with the matrix-valued coefficients determined by \eqref{eqn:Pi} and \eqref{eqn:Si}.  
\end{proof} 

\subsection{Proofs of Corollaries} 
\label{s:ProofCorollaries}

\textbf{Corollary \ref{cor:Random_NeuralNetwork}}

\begin{proof}
From Theorem \ref{thm:NashEqmExistence}, the sub-matrices of $\mathbf{A}(t)=\big[ A^{(i,j)}(t) \big]_{1\leq i,j \leq N}$ 
are 
\begin{align} 
A^{(i,j)}(t) 
=
 & \mathbb{E} \big[ Z^{i\top}_t \mathbf{e}_i^\top P_i(t+1) \mathbf{e}_j Z^j_t \big] 
 \notag \\ 
 = & 
 \mathbf{e}_i^\top P_i(t+1) \mathbf{e}_j \mathbb{E} \big[ Z^{i\top}_t  Z^j_t \big] ,
 \quad 1 \leq i, j \leq N .  
  \notag 
\end{align} 
The diagonal sub-matrices 
$A^{(i,i)}(t) = \big[ ( A^{(i, i)}(t) )_{jk} \big]_{1\leq j, k\leq d_z}$ are given by  
\begin{align} 
 ( A^{(i,i)}(t) )_{jk}  
 = &
  \mathbf{e}_i^\top  P_i(t+1) \mathbf{e}_i
  \mathbb{E} \big[ (Z_{t+1}^{i\top} Z_{t+1}^i)_{jk} \big]  
    \notag \\ 
= &
  \mathbf{e}_i^\top  P_i(t+1) \mathbf{e}_i
  \mathbb{E} \big[ (Z_{t+1}^i)_j (Z_{t+1}^i)_k \big]  , 
    \notag 
\end{align} 
where $\mathbb{E}\big[ (Z_{t+1}^i)_j (Z_{t+1}^i)_k \big]$ are given by \eqref{E[(Zij)b(Zik)]}. 

Since the white noises $\{W^i_t\}_{i, t}$, are independent,  
$\{ Z^i_t \}_{i, t}$ are independent. 
The off-diagonal sub-matrices of $\mathbf{A}(t)$ are written as   
\begin{align} 
A^{(i,j)}(t) 
 =&  \mathbb{E} \big[ Z^{i\top}_t \mathbf{e}_i^\top P_i(t+1) \mathbf{e}_j Z^j_t \big]  \notag \\ 
 = & \mathbb{E} \big[  Z^{i\top}_t \big]  \mathbf{e}_i^\top P_i(t+1) \mathbf{e}_j  \mathbb{E} \big[ Z^j_t \big] , 
\notag  
\end{align} 
where the explicit form of $\mathbb{E} \big[ Z^{i}_t \big]$ is given by \eqref{E[Zi]}. 
The explicit forms of the matrices $\widehat{\mathbf{A}}(t)=\big[ \widehat{A}^{(i,j)}(t) \big]_{i,j =1}^N$ 
and $\mathbf{D}_i(t) = \big[ D_i^{(j,k)}(t) \big]_{j,k = 1}^N$,  $i = 1, ..., N$ are determined in a similar manner. 
The explicit forms of $\mathbf{B}(t)$, $\mathbf{C}(t)$ and $\mathbf{D}(t)$ are determined by \eqref{E[Zi]} and 
$a(t, i) = \mathbb{E}[Z^i_{t+1}]$. 
\end{proof}

The following lemma gives enough information to compute the relevant matrices in Theorem \ref{thm:NashEqmExistence}.
\begin{lemma}[Mean and Covariance of $Z_{\cdot}^i$ in Corollary \ref{cor:Random_NeuralNetwork} for $d_y=1$]
\label{lem:Mean_Cov_ResModel} 
Under the conditions \\specified in Corollary \ref{cor:Random_NeuralNetwork}, and for each $t \in \mathbb{N}$, $i = 1, \dots, N$, and $j, k = 1, \dots, d_z$, the following hold:
\begin{align} 
 &       \mathbb{E}[Z_{t+1}^i] 
    = 
        a^i_t 
        \odot
        \Big(
            \bar{1}_{d_z} 
            - 
            \Phi\bullet \big(
                \frac{-a^i_t}{\sigma_t^i}
            \big)
        \Big)
        +
            \frac{\sigma_t^i}{\sqrt{2\pi}}
        \,
        \exp\bullet\Big(
            \frac{-(a^i_t \odot a^i_t )}{2\, (\sigma_t^i)^2 }
        \Big) , 
        \label{E[Zi]}
\\ 
&
\begin{aligned} 
     \mathbb{E}\big[(Z_{t+1}^i)_j  (Z_{t+1}^i)_k\big] 
    = 
    \begin{cases}
             ( (a^i_t \odot a^i_t )_j + (\sigma^i_t)^2 )
            \cdot
            \Big(
                1 
                - 
                \Phi(
                    \frac{-(a^i_t)_j}{\sigma_t^i}
                )
            \Big) 
            \\
            \hspace{2cm}
            +
            (a^i_t)_j\,
                \frac{\sigma_t^i}{\sqrt{2\pi}}
            \,
            \exp\Big( 
                \frac{ - (a^i_t \odot a^i_t )_j }{2\,(\sigma_t^i)^2 }
            \Big) ,   
        \quad & \mbox{ if } j=k , 
\\
   \mathbb{E}\big[ (Z_{t+1}^i)_j \big] 
    \mathbb{E} \big[ (Z_{t+1}^i)_k \big]  , 
\quad & \mbox{ if } j\neq k, 
\end{cases}
\end{aligned} 
\label{E[(Zij)b(Zik)]}
\end{align} 
where $\odot$ denotes the Hadamard product,\footnote{The Hadamard product of any pair of vectors $a,b\in \mathbb{R}^K$ and any $K\in \mathbb{N}$ is given by
$a\odot b\eqdef (a_kb_k)_{k=1}^K$.}
and $\Phi$ is the standard normal CDF. 
\end{lemma}
\begin{proof}
In the setting of Corollary \ref{cor:Random_NeuralNetwork}, for each $t\in \mathbb{N}$, every $i=1,\dots,N$, and $j=1,\dots,d_z$, we have that
$(Z_{t+1}^i)_j$ is a normal random variable with mean $
(a_t)_j$ and variance $(\sigma_t^i)_j$; where we have used the fact that the components of $W_t^i$ are independent whenever their indices differ. 
Consequentially, the computation of the mean and variance of the rectified Gaussian distribution, computed in \citep[page 90 and Appendix A]{Beauchamp_Distribution_RectifiedGaussian2018}, yields the conclusion. 
\end{proof}

The multi-dimensional version of Corollary \ref{cor:Random_NeuralNetwork} is proven here. Overall it is of a similar form with some additional technicalities.
\begin{corollary}[Random Feature Network with $d_y>1$] 
\label{cor:Random_NN_dy>1}
In the setting of Corollary \ref{cor:Random_NeuralNetwork} with $d_y>1$, we have the block matrices
\begin{align*}
& \mathbf{A}(t)
 = 
\big[ A^{(i,j)}(t) \big]_{i,j=1}^N , 
\hspace{0.5cm} 
\widehat{\mathbf{A}}(t) = \big[ \widehat{A}^{(i,j)}(t) \big]_{i, j = 1}^N , 
\notag \\  
 & \mathbf{B}(t)  = 
 \begin{bmatrix} 
  P_1(t+1) \mathbf{e}_1 a(t, 1), \dots, P_N(t+1) \mathbf{e}_N a(t, N) \\ 
 \end{bmatrix}^\top  ,  
 \notag \\ 
 & \mathbf{C}(t)  = 
 \begin{bmatrix} 
  S_1(t+1)^\top \mathbf{e}_1 a(t, 1), \dots, S_N(t+1)^\top  \mathbf{e}_N a(t, N) \\ 
 \end{bmatrix}^\top , 
  \notag \\   
& \mathbf{D}(t) 
  = 
 \begin{bmatrix} 
 \mathbf{e}_1 a(t,1), \dots  , \mathbf{e}_N a(t, N) 
 \end{bmatrix} , 
 \quad 
  \mathbf{D}_i(t) = [ D_i^{(j,k)}(t) ]_{j,k = 1}^N , \quad i = 1, \dots, N . \notag 
\end{align*}
For $t=1,\dots,N$ and $t\in \mathbb{N}_+$, each $a(t,i)$ is defined as follows
\begin{align} 
& a^i_t\eqdef A^{i} [x_t,\dots, x_t ] + b^{i} = \big[ (a^i_t)_1^\top , \dots, (a^i_t)_{d_y}^\top \big]^\top \in \mathbb{R}^{d_y \times d_z} , \notag \\ 
& a(t,i)
\eqdef 
   \big[ a(t, i)_1^\top , \dots, a(t, i)_{d_y}^\top \big]^\top
   \in \mathbb{R}^{d_y \times d_z} 
   , \notag \\ 
& a(t, i)_k = (a^i_t)_k 
        \odot
        \left(
            \bar{1}_{d_z} 
            - 
            \Phi\bullet \left(
                \frac{- (a^i_t)_k}{(\sigma_t^i)_k}
            \right)
        \right)
    \notag \\     
    & \hspace{4cm}+
            \frac{(\sigma_t^i)_k}{\sqrt{2\pi}}
        \,
        \exp\bullet\left(
            \frac{-( (a^i_t)_k \odot (a^i_t)_k )}{2\, ( (\sigma_t^i)_k )^2 }
        \right)  , \quad 
        k = 1, \dots, d_y , 
        \label{a(t,i)k}
\end{align} 
and $\odot$ and $\Phi$ are defined in Lemma \ref{lem:Mean_Cov_ResModel}. 
The sub-matrices of $\mathbf{A}(t)$, $\widehat{\mathbf{A}}(t)$, 
and $\{ \mathbf{D}_i(t) \}_{i=1}^N$,   
are given by 
\begin{align} 
 & A^{(i,j)}(t) =
\begin{cases} 
\sum_{k, l = 1 }^{d_y} 
( \mathbf{e}_i^\top P_i(t+1) \mathbf{e}_i )_{kl}
\mathbb{E} \big[  (Z^i_t)_k^\top  (Z^i_t)_l \big] , \quad & \mbox{if} \,\, i = j , \\
  a(t, i)^\top \mathbf{e}_i^\top  P_i(t+1) \mathbf{e}_j a(t, j) , \quad & \mbox{if} \,\, i \neq j  , 
 \end{cases} 
 \notag \\ 
 & \widehat{A}^{(i,j)}(t) = 
 \begin{cases} 
 \sum_{k, l = 1 }^{d_y} w^i_t w^i_t \mathbb{E} \big[  (Z^i_t)_k^\top 
   (Z^i_t)_l \big]    , \quad & \mbox{if} \,\, 
  i = j , \\ 
 a(t, i)^\top  w^i_t w^j_t a(t, j) , \quad & \mbox{if} \,\, i \neq j   , 
\end{cases} 
\notag \\ 
& D_i^{(r,s)}(t) = 
\begin{cases}
 \sum_{k,l = 1}^{d_y} ( \mathbf{e}_r^\top P_i(t+1) \mathbf{e}_r )_{kl} 
 \mathbb{E} \big[ (Z^r_t)_k^\top 
  (Z^r_t)_l \big]  , \quad & \mbox{if} \,\, r=s , \\ 
a(t,r)^\top \mathbf{e}_r^\top P_i(t+1) \mathbf{e}_s a(t, s) , \quad & \mbox{if} 
\,\, r \neq s .  
\notag 
\end{cases}
\end{align} 
The matrices 
$\mathbb{E} \big[ (Z^i_t)_k^\top (Z^i_t)_l \big]
= \big( \mathbb{E} \big[ (Z^i_t)_k^\top (Z^i_t)_l \big]_{pq} \big)_{p,q=1}^{d_z}$ for  
$k$, $l=1, \dots, d_y$ are given as follows: 
if $k$ and $l$ are distinct then
\begin{align} 
  \mathbb{E} \big[  (Z^i_t)_k^\top  (Z^i_t)_l \big]_{pq}   
 =  
 a(t, i)_{kp}  
  a(t, i)_{lq} ;   
\end{align} 
otherwise $\mathbb{E} \big[ (Z^i_t)_k^\top (Z^i_t)_k \big]_{pq} 
 = \mathbb{E} \big[ (Z^i_t)_{kp} (Z^i_t)_{kq} \big]$ and 
\begin{align} 
 \mathbb{E} \big[ (Z^i_t)_{kp} (Z^i_t)_{kq} \big] 
= 
\begin{cases}
            \big[ ( (a^i_t)_k \odot (a^i_t)_k )_p + (\sigma^i_t)_k^2 \big]  
            \cdot
            \Big(
                1 
                - 
                \Phi(
                    \frac{-(a^i_t)_{kp}}{ (\sigma_t^i)_k }
                )
            \Big) 
             \\
            \hspace{2cm}
            +
            (a^i_t)_{kp}   \,
                \frac{ (\sigma_t^i)_k }{\sqrt{2\pi}}
            \,
            \exp\Big(
                \frac{ - ( (a^i_t)_k \odot (a^i_t)_k )_p}{2\, (\sigma_t^i)_k^2 }
            \Big) , 
            \quad 
        & \mbox{ if } p = q , 
\\
    \mathbb{E}[(Z_{t+1}^i)_{kp}]
    \mathbb{E}[(Z_{t+1}^i)_{kq}] 
     = a(t, i)_{kp} a(t,i)_{kq} ,  
    \quad 
     & \mbox{ if }   p \neq q . 
    \end{cases}   
 \label{E[(Zi)kp(Zi)kq]} 
\end{align} 
\end{corollary}

\begin{proof}
If $d_y>1$, we have for each $i=1, \dots, N$ that   
\begin{align} 
 Z^i_t = \operatorname{ReLU}  \bullet \big( A^i [x_t, \dots, x_t] + b^i + \sigma^i_t W^i_t \big)  
 = \big[ (Z^{i}_t)_1^\top , \dots , (Z^{i}_t)_{d_y}^\top \big]^\top 
 \in \mathbb{R}^{d_y \times d_z} 
 \label{Zi=[Zi1T,...,ZidyT]T}
\end{align} 
with $(Z^{i}_t)_k = (a^i_t)_k + (\sigma^i_t)_k W^i_t$, $k=1,\dots, d_y$. 

Denote
\begin{align}  
& \mathbb{E} \big[ Z^{i}_{t} \big] =  a(t, i)  
  =  \big[ a(t, i)_1^\top , \cdots, a(t, i)_{d_y}^\top \big]^\top \in \mathbb{R}^{d_y \times d_z} . 
  \notag 
\end{align} 
By \eqref{E[Zi]}, we obtain \eqref{a(t,i)k} for each $\mathbb{E} \big[ (Z^i_{t})_k \big] = a(t, i)_k$.

For $i\neq j$, $Z^i_t$ and $Z^j_t$ are independent since the random noises $W^i_t$ and $W^j_t$ are independent, 
then $A^{(i,j)}(t)$ can be represented as 
\begin{align} 
  A^{(i,j)}(t) = &
 \mathbb{E} [ Z^{i\top}_t \mathbf{e}_i^\top P_i(t+1) \mathbf{e}_j Z^j_t ] \notag \\ 
 =& \mathbb{E} [ Z^{i\top}_t ] \mathbf{e}_i^\top P_i(t+1) \mathbf{e}_j \mathbb{E} [ Z^j_t ] \notag \\ 
 =& a(t, i)^\top \mathbf{e}_i^\top P_i(t+1) \mathbf{e}_j a(t,j) . 
 \notag 
\end{align} 

For $i=j$, by \eqref{Zi=[Zi1T,...,ZidyT]T} the diagonal sub-matrices $A^{(i,i)}(t)$ can be writen as  
\begin{align} 
 A^{(i,i)}(t) = &
 \mathbb{E} \big[ Z^{i\top}_t \mathbf{e}_i^\top P_i(t+1) \mathbf{e}_i Z^i_t \big]    \notag  \\ 
  = & \sum_{k, l = 1 }^{d_y} 
  ( \mathbf{e}_i^\top P_i(t+1) \mathbf{e}_i )_{kl} 
  \mathbb{E} \big[  (Z^i_t)_k^\top  (Z^i_t)_l \big] . 
 \notag 
\end{align}

For $k=l$, since $(Z^{i}_t)_k = (a^i_t)_k + (\sigma^i_t)_k W^i_t$ is a $d_z$-dimensional row vector, 
we can use \eqref{E[(Zij)b(Zik)]} to determine \eqref{E[(Zi)kp(Zi)kq]} for the matrix 
$\mathbb{E} \big[ (Z^i_t)_k^\top (Z^i_t)_l \big]
= \big( \mathbb{E} \big[ (Z^i_t)_k^\top (Z^i_t)_l \big]_{pq} \big)_{p,q=1}^{d_z}$.  

For $k\neq l$, we use the fact that $W^i_t$ is a vector of independent random variables to determine the matrix 
$\mathbb{E} \big[  (Z^i_t)_k^\top ( \mathbf{e}_i^\top P_i(t+1) \mathbf{e}_i )_{kl} (Z^i_t)_l \big]$ such that
\begin{align} 
  \mathbb{E} \big[  (Z^i_t)_k^\top ( \mathbf{e}_i^\top P_i(t+1) \mathbf{e}_i )_{kl} (Z^i_t)_l \big]_{pq}   
 = & 
 \mathbb{E} \big[  (Z^i_t)_{kp} ( \mathbf{e}_i^\top P_i(t+1) \mathbf{e}_i )_{kl} (Z^i_t)_{lq} \big] 
 \notag \\ 
 = & 
 \mathbb{E} \big[  (Z^i_t)_{kp} \big] ( \mathbf{e}_i^\top P_i(t+1) \mathbf{e}_i )_{kl} 
  \mathbb{E} \big[ (Z^i_t)_{lq} \big]  
 \notag \\ 
 = & 
 a(t, i)_{kp} ( \mathbf{e}_i^\top P_i(t+1) \mathbf{e}_i )_{kl} 
  a(t, i)_{lq}  . 
  \notag 
\end{align} 
The explicit forms of the matrices $\widehat{\mathbf{A}}(t)=\big[ \widehat{A}^{(i,j)}(t) \big]_{i,j =1}^N$ 
and $\mathbf{D}_i(t) = \big[ D_i^{(j,k)}(t) \big]_{j,k = 1}^N$,  $i = 1, \dots, N$ are determined in a similar manner. 
The explicit forms of $\mathbf{B}(t)$, $\mathbf{C}(t)$ and $\mathbf{D}(t)$ are determined by~\eqref{E[Zi]} and 
$a(t, i) = \mathbb{E}[Z^i_{t+1}]$. 
\end{proof}

\section{The Algorithm}
\label{s:AlgoMain}

In this section, we summarize the FL algorithm.

\begin{algorithm}[H]
\SetKwInOut{Require}{Require}
\caption{Proposed Federated Learning Algorithm}
\label{alg:FedMain}
\SetAlgoLined
\LinesNumberedHidden
\Require{Expert models $\{\varphi^{i}\}_{i=1}^N$, Synchronization Frequency $\tau \in \mathbb{N}_+$}
\DontPrintSemicolon
\vspace{0.2cm}
\For(\tcp*[f]{Generate prediction for time $t$}){$t:1, 2, \dots$} 
 {
\For(\tcp*[f]{Update experts in parallel}){$i:1,\dots,N$ in parallel}{
$\beta_{t-1}^i \leftarrow \big( X^{i \top}_{t-1} X^i_{t-1} + \gamma_i I_{d_z} \big)^{-1}\, X_{t-1}^{i \top} \bar{y}^i_{t} $
\tcp*{Local weighted regression}
Generate $\varepsilon_{t-1}^i$
\tcp*{Generate random state}
$Z^i_{t-1}\leftarrow \varphi^{i}(x_{[0:t-1]}, Z_{t-2}^i, \varepsilon_{t-1}^i)$
\tcp*{Update hidden state}
$\hat{Y}^i_t \leftarrow \hat{Y}_{t-1}^i + Z^i_{t-1} \, \beta_{t-1}^i $
\tcp*{Update residual prediction}
} 
\tcp{Compute mixture weights} 
$\hat{\mathbf{Y}}_{t-1} \leftarrow  (\hat{Y}^{1\top}_{t-1}, ..., \hat{Y}^{N\top}_{t-1})^\top$ \\
$A \leftarrow 2\,N \hat{\mathbf{Y}}^\top_{t-1} \hat{\mathbf{Y}}_{t-1} + \kappa I_N$ \\
$b\leftarrow 2 (y^\top_{t-1}  \hat{\mathbf{Y}}_{t-1})^{\top}$ \\
$ w_{t-1} \leftarrow  A^{-1}\, \left( b - \frac{\mathbf{1}_N^{\top}A^{-1} b \, - \, \eta}{\strut \mathbf{1}_N^{\top}A^{-1}\mathbf{1}_N} \cdot\,\mathbf{1}_N
\right)$ \\ 
\If{$t \bmod \tau =  0$ and $t>0$}{
\tcp{Synchronize local agents via Nash game}
$\boldsymbol{\beta}^{\text{sync}}_{t-2}, \hat{\mathbf{Y}}^{\text{sync}}_{t-1} \leftarrow $ Run Synchronize Subroutine: Algorithm \ref{alg:sync_subroutine}
\tcp*{Synchronized Weights}
\For(\tcp*[f]{Compute new predictions}){$i:1,\dots,N$ in parallel}{
$\beta_{t-1}^i \leftarrow (\boldsymbol{\beta}^{\text{sync}}_{t-2})^i$ \;
Generate $\varepsilon_{t-1}^i$ 
\tcp*{Generate random state} 
$\hat{Y}^i_{t} \leftarrow (\hat{\mathbf{Y}}^{\text{sync}}_{t-1})^i +  \varphi^{i}(x_{[0:t-1]}, Z_{t-2}^i, \varepsilon_{t-1}^i) \beta^i_{t-1}$ \tcp*{Correct local predictions} 
}
}
$\hat{Y}_t = w_{t-1}^\top \hat{\mathbf{Y}}_{t}$
\tcp*{Ensemble final prediction for time $t$}
}
\end{algorithm}

Note that in Algorithm \ref{alg:sync_subroutine}, described below, for simplicity, the indexing is presented from $0$ to $T$, representing the historical length of time over which the Nash synchronization takes place, i.e. the ''look-back'' length. In practice, the input values should correspond to those of the FL algorithm starting at $T = t-1$, proceeding backwards from there.

\begin{algorithm}[H]
\SetKwInOut{Require}{Require}
\caption{Synchronization Subroutine}
\label{alg:sync_subroutine}
\SetAlgoLined
\LinesNumberedHidden
\Require{Observed target values $\{ ( y^1_t, \ldots, y^N_t)\}_{t=0}^T$, Hidden States $\{Z_t^1, \ldots, Z_t^N\}_{t=0}^T$ \\
Mixture weights $w_\cdot = \{ (w^1_t, \ldots, w^N_t) \}_{t=0}^T$ $\leftarrow$ From Algorithm \ref{alg:FedMain}}
\DontPrintSemicolon
\vspace{0.2cm}
\DontPrintSemicolon
\tcp{Initialize updates} 
$\mathbf{w}_T=(w^1_T I_{d_{y}}, \dots, w^N_T I_{d_{y}})^\top$\\
\For{$i:1,\dots,N$ in parallel}{
$P_i(T)  = 0$ \\
$S_i(T) = 0$ 
}
\tcp{Generate iterates}
\For(){$t=T-1,\dots,0$}{
\tcp{Update driving parameters} 
$\mathbf{w}_t=(w^1_t I_{d_{y}}, \dots, w^N_t I_{d_{y}})^\top$ \\
$\mathbf{A}(t) = \mathbb{E} \big\{ [ P_1(t+1) \mathbf{e}_1 Z^1_t, \, \dots, \, P_N(t+1) \mathbf{e}_N Z^N_t ]^\top 
 [ \mathbf{e}_1 Z^1_t, \dots, \mathbf{e}_N Z^N_t ] \big\}$ \\  
 $\widehat{\mathbf{A}}(t) = \mathbb{E} \big\{ [ \mathbf{e}_1 Z^1_t, \, \dots, \, \mathbf{e}_N Z^N_t ]^\top \mathbf{w}_t \mathbf{w}_t^\top  
 [ \mathbf{e}_1 Z^1_t, \dots, \mathbf{e}_N Z^N_t ] \big\}$ \\  
$\mathbf{B}(t)= \mathbb{E} [ P_1(t+1) \mathbf{e}_1 Z^1_t, \, \dots, \, P_N(t+1) \mathbf{e}_N Z^N_t ]^\top $ \\
$\mathbf{C}(t)= \mathbb{E} [ S_1^\top(t+1) \mathbf{e}_1 Z^1_t, \, \dots, \, S_N^\top(t+1) \mathbf{e}_N Z^N_t ]^\top$ \\
$\mathbf{D}(t)  
 = \mathbb{E}  
 \begin{bmatrix} 
    \mathbf{e}_1 Z^1_t    , \, 
    \mathbf{e}_2 Z^2_t   ,  \,
 \cdots , \, 
  \mathbf{e}_N Z^N_t    
 \end{bmatrix}$ \\ 
   $\mathbf{G}(t) = - \big[ e^{-\boldsymbol{\alpha}(T-1-t)} ( \Gamma + \widehat{\mathbf{A}}(t) ) 
  + \mathbf{A}(t)    \big]^{-1} 
 \big[ e^{-\boldsymbol{\alpha}(T-1-t)} \mathbf{D}^\top(t) \mathbf{w}_t \mathbf{w}_t^\top + \mathbf{B}(t) \big] $ 
 \notag \\ 
  $\mathbf{H}(t) = \big[ e^{-\boldsymbol{\alpha}(T-1-t)} (\Gamma + \widehat{\mathbf{A}}(t) ) + \mathbf{A}(t)   \big]^{-1} 
  \big[ e^{-\boldsymbol{\alpha}(T-1-t)} \mathbf{D}(t)^\top \mathbf{w}_t y_{t+1} - \mathbf{C}(t) \big] $
 \notag \\
\For{$i:1,\dots,N$ in parallel}{
\vspace{-0.4cm}
\begin{flalign*} 
\mathbf{D}_i(t) & =  \mathbb{E} \big\{   
 \begin{bmatrix} 
    \mathbf{e}_1 Z^1_t    , \, 
    \mathbf{e}_2 Z^2_t   ,  \,
 \dots , \, 
  \mathbf{e}_N Z^N_t    
 \end{bmatrix}^\top 
  P_i(t+1) 
  \begin{bmatrix} 
    \mathbf{e}_1 Z^1_t    , \, 
    \mathbf{e}_2 Z^2_t   ,  \,
 \dots , \, 
  \mathbf{e}_N Z^N_t    
 \end{bmatrix} 
 \big\} & \\ 
  P_i(t) &=  
  \mathbf{G}(t)^\top 
  \big[ \widehat{\mathbf{A}}(t) + \mathbf{D}_i(t) + e^{-\alpha_i(T-1-t)} 
  \gamma_i \hat{\mathbf{e}}_i \hat{\mathbf{e}}_i^\top  \big] 
   \mathbf{G}(t) &\\   
  & \qquad + \big[ e^{-\alpha_i(T-1-t)}\mathbf{w}_t \mathbf{w}_t^\top + P_i(t+1) \big] 
   \mathbf{D}(t) \mathbf{G}(t)  &\\ 
  & \qquad + \mathbf{G}(t)^\top \mathbf{D}(t)^\top \big[ e^{-\alpha_i(T-1-t)} \mathbf{w}_t \mathbf{w}_t^\top + P_i(t+1) \big]^\top
  &\\ 
& \qquad + e^{-\alpha_i(T-1-t)} \mathbf{w}_t \mathbf{w}_t^\top + P_i(t+1)  &\\
 S_i(t)  &= \mathbf{G}^\top(t) \big[ \widehat{\mathbf{A}}(t) + \qquad\mathbf{D}_i(t) 
   + e^{-\alpha_i(T-1-t)} \gamma_i \hat{\mathbf{e}}_i \hat{\mathbf{e}}_i^\top  \big] \mathbf{H}(t) &\\ 
& \qquad + \big[ e^{-\alpha_i(T-1-t)} \mathbf{w}_t \mathbf{w}_t^\top + P_i(t+1)  \big] \mathbf{D}(t) \mathbf{H}(t) &\\  
& \qquad + \mathbf{G}^\top(t) \mathbf{D}^\top(t) \big[ - \mathbf{w}_t y_{t+1} e^{-\alpha_i(T-1-t)} + S_i(t+1) \big]  &\\
& \qquad - \mathbf{w}_t y_{t+1} e^{-\alpha_i(T-1-t)} + S_i(t+1)
\end{flalign*}
\vspace{-0.6cm}
}
}
\tcp{Update game strategy} 
$\hat{Y}_0 = (y^{1\top}_0, \dots, y^{N\top}_0)^\top$ \\
\For{$t=0, \dots, T-1$}{
$\boldsymbol{\beta}_t 
=  \mathbf{G}(t) \hat{Y}_t + \mathbf{H}(t)$ \\
$\hat{Y}_{t+1} = \hat{Y}_t + \diag[Z^1_t, \dots, Z^N_t ] \boldsymbol{\beta}_t$
}

\Return $\boldsymbol{\beta}_{T-1}$, $\hat{\mathbf{Y}}_{T}$
\end{algorithm} 

\section{Additional Experiment Details}
\label{s:ExperimentDetials}

This appendix provides additional technical details for the experiments, including experimental design, dataset descriptions, and details in Appendix \ref{s:ExperimentDetials__ss::datasets}. Appendix \ref{s:OptimalParamterChoices} outlines the hyperparameter search space and the optimal hyperparameters used in the experiments. Further ablation studies in Appendix \ref{a:Ablations}, covering validation results for the BoC Exchange Rate and ETT datasets in which the performance on validation segments are reported. This is followed by a study informing the choice of activation function for ESNs and a discussion of the importance of data normalization in both ESNs and RFNs.

In this work, time-series performance is measured with mean-squared error (MSE) relative to the reference sequence. For a sequence of length $T$, predictions $\hat{y}_i \in \mathbb{R}^n$, and targets $y_i \in \mathbb{R}^n$, MSE is computed as 
\begin{align*}
    \text{MSE} = \frac{1}{T} \sum_{i=1}^T \Vert \hat{y}_i - y_i \Vert^2.
\end{align*}

\subsection{Datasets} 
\label{s:ExperimentDetials__ss::datasets}

In this appendix, the details for each of the time-series datasets considered in the numerical experiments are laid out. In all, five unique time series are considered. Each incorporates different complexities for time-series prediction, as well as varying length scales and dimensionalities. Throughout this section, time-series inputs are denoted as $\mathbf{x}$, targets as $\mathbf{y}$, and predictions as $\hat{\mathbf{y}}$

\subsubsection{Periodic Signal}

The Periodic Signal series is the simplest in terms of construction, but can be challenging for time-series prediction due to its highly oscillatory nature. The series has $200$ data points taking on a minimum value of -0.9 and a maximum value of 0.9. The target regularly oscillates between these two values throughout the series. At each time step, $t$, a single target lag is used to make model predictions such that $\mathbf{x}_{t-1} = \mathbf{y}_{t-1}$ in predicting $\mathbf{y}_t$. Figure \ref{fig:periodic_series} displays the target values of the series.

\begin{figure}[ht!]
    \centering
    \includegraphics[width=0.8\linewidth]{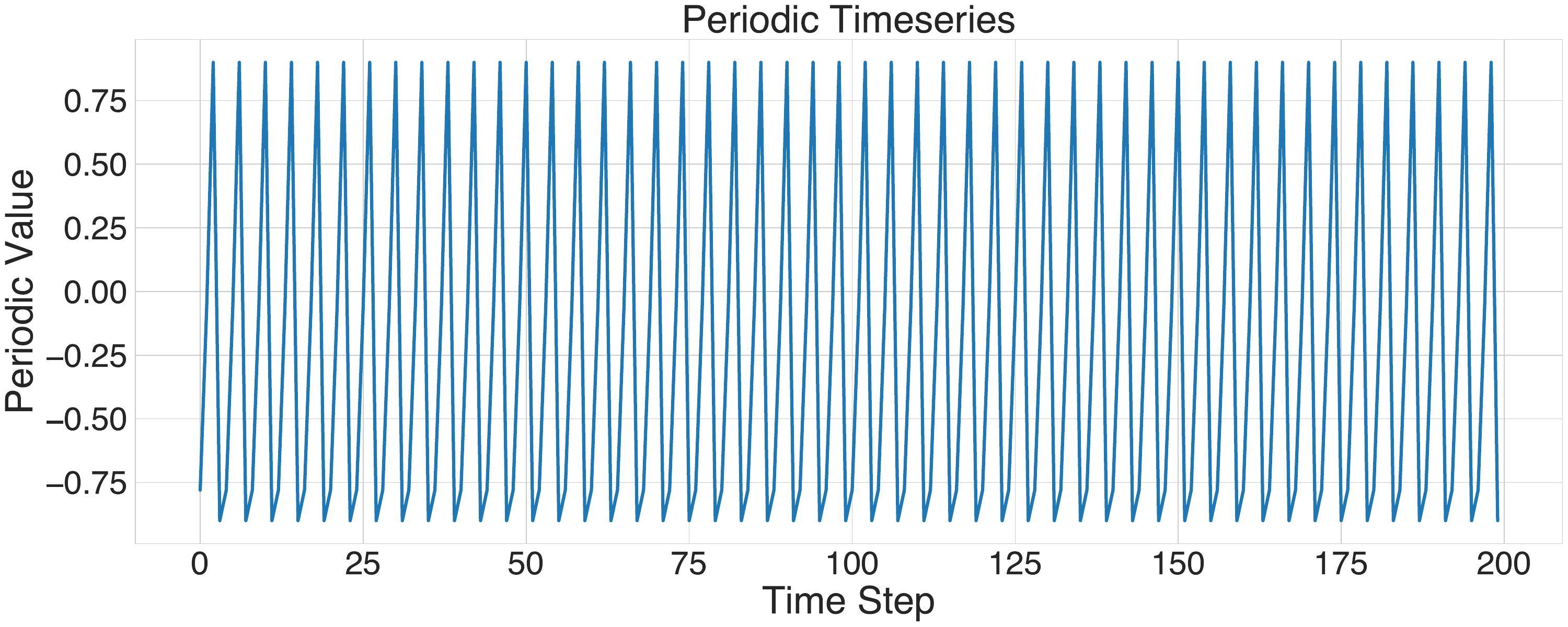}
    \caption{Periodic time series with 200 target values.}
    \label{fig:periodic_series}
\end{figure}

\begin{figure}[ht!]
    \centering
    \includegraphics[width=0.8\linewidth]{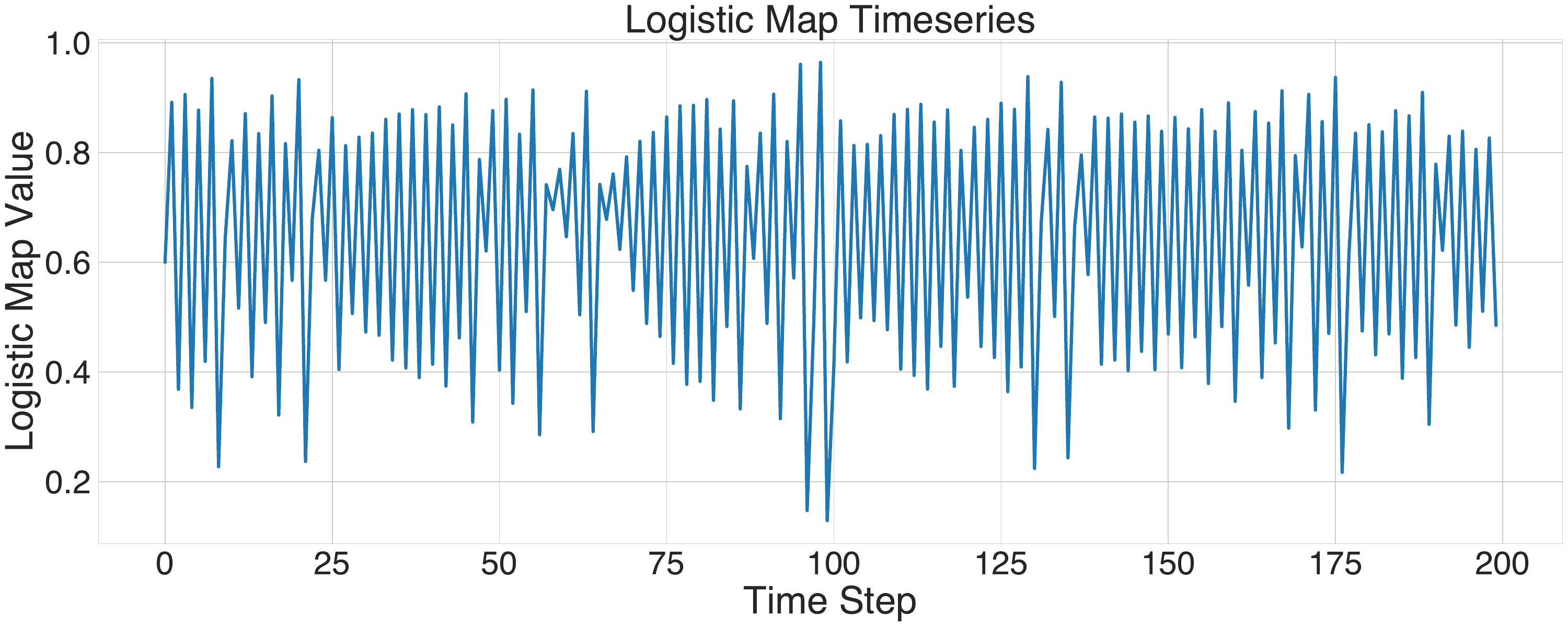}
    \caption{Logistic map series progressing over $200$ time steps.}
    \label{fig:logistic_map_dataset}
\end{figure}

\subsubsection{Logistic Map}

The Logistic Map dataset is a synthetic time series based on the dynamics of the logistic map. The generating function from the EchoTorch library \citep{echotorch} is used to produce a time series with $200$ total data points. The default parameters for the EchoTorch dataset are applied, namely $\alpha=5$, $\beta=11$, $\gamma=13$, $c=3.6$, and $b=0.13$. Figure \ref{fig:logistic_map_dataset} shows the resulting series. The target variable, $\mathbf{y}$, has a single dimension. As with the Periodic dataset, a single step lag is provided as input for model predictions. That is, for time step $t$ and target $\mathbf{y}_t$, the input is $\mathbf{x}_{t-1} = \mathbf{y}_{t-1}$.

\subsubsection{Concept Shift} \label{concept_drift_description}

The Concept Drift dataset considers a rapid, but smooth, change in the relationship of $\mathbf{x}$ and $\mathbf{y}$ halfway through the trajectory. The series consists of $200$ points drawn from the following analytical relationship. Let $p$ be evenly spaced in the interval $[0, 2\pi]$. Then
\begin{align*}
    x_1 = p, && x_2 = p, && x_3 = \sqrt{p}.
\end{align*}
For $p \in \left[0, \frac{7}{8}\pi \right]$, $\mathbf{y} \in \mathbb{R}^2$ is
\begin{align*}
    y_{1, 1}(x_1, x_2, x_3) &= x_1^2 + \sin x_2 + x_1  x_3 + \frac{1}{2} \cos(10 x_1), \\
    y_{2, 1}(x_1, x_2, x_3) &= x_1  \cos x_2 + x_3 - e^{-x_2}.
\end{align*}
For $p \in \left[\frac{9}{8} \pi, 2\pi\right]$,
\begin{align*}
    y_{1, 2}(x_1, x_2, x_3) &= x_1 + x_2 - \sin x_3, \\
    y_{2, 2}(x_1, x_2, x_3) &= \cos x_1  \sin x_2 + x_3^2 + \frac{1}{4} \cos(10 x_1).
\end{align*}
The phase transition, which occurs when $p \in \left[\frac{7}{8} \pi, \frac{9}{8} \pi \right]$ is facilitated by the piecewise function
\begin{align*}
    \alpha(x) = \begin{cases}
      1.0 & x < \frac{7}{8} \pi, \\
      \cos(2(x_1 - \frac{7}{8}\pi) & \frac{7}{8} \pi \geq x \leq \frac{9}{8} \pi, \\
      0.0 & x > \frac{9}{8} \pi. \\
    \end{cases}
\end{align*}
The final series is then written
\begin{align*}
    y_1(x_1, x_2, x_3) &= \alpha(x_1) \cdot y_{1, 1}(x_1, x_2, x_3) + (1-\alpha(x_1)) \cdot y_{1, 2}(x_1, x_2, x_3), \\
    y_2(x_1, x_2, x_3) &= \alpha(x_1) \cdot y_{2, 1}(x_1, x_2, x_3) + (1-\alpha(x_1)) \cdot y_{2, 2}(x_1, x_2, x_3).
\end{align*}
Figure \ref{fig:concept_drift} provides a visualization of the series as a function of $x_1$.
\begin{figure}[ht!]
    \centering
    \includegraphics[width=0.8\linewidth]{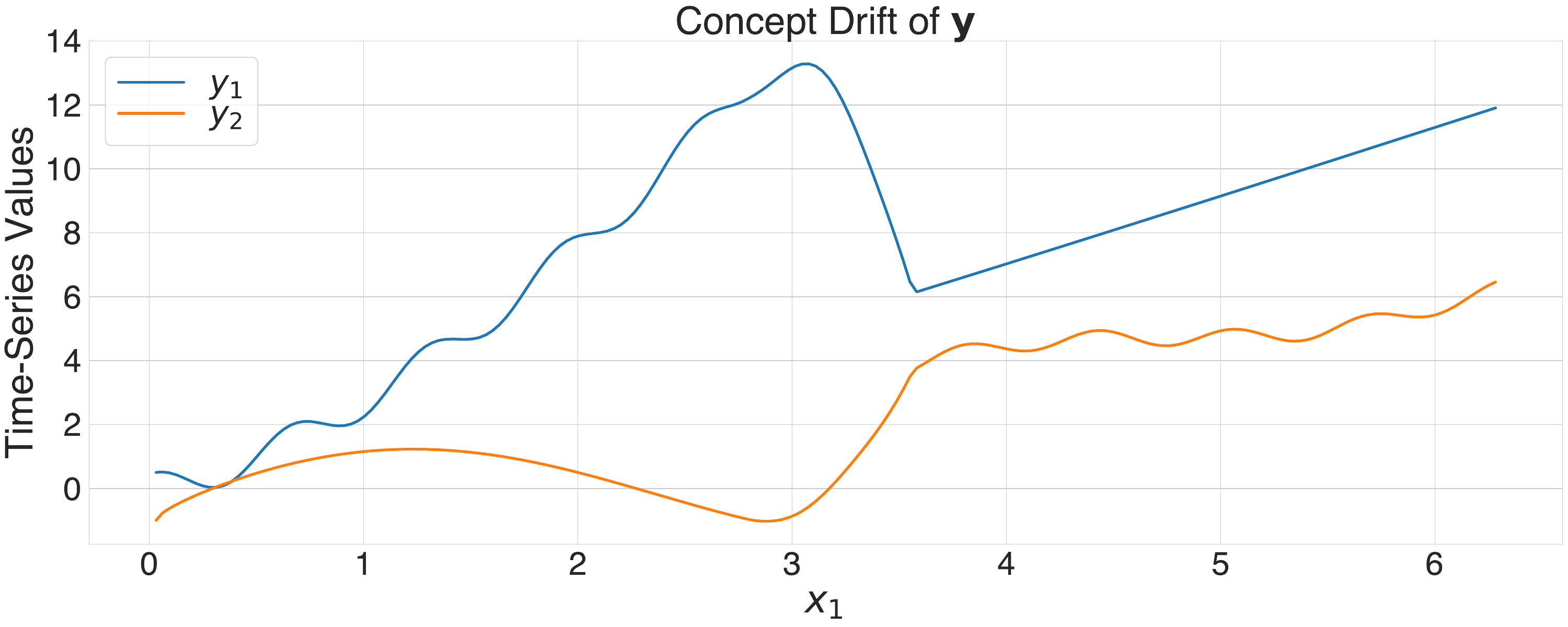}
    \caption{Visualization of the Concept Drift time series as a function of $x_1$.}
    \label{fig:concept_drift}
\end{figure}

\subsubsection{Bank of Canada Exchange Rate}
\label{s:ExperimentDetials__ss::datasets__sss:::BoC}
The Bank of Canada (BoC) Exchange Rate dataset \citep{boc_exchange_rates} records the daily exchange rates for $12$ different currencies, relative to the Canadian dollar (CAD), from May 1, 2007 to April 28, 2017. This constitutes $3,651$ observations per currency. Any one of the exchange rates may be used as a target with lagged currency values used as inputs. In the numerical experiments, the target exchange rate is USD and inputs for model predictions at a given time step $t$ are $\{\text{USD}_{t-1}$, $\text{USD}_{t-2}$, $\text{AUD}_{t-1}$, $\text{AUD}_{t-2}$, $\text{EUR}_{t-1}$, $\text{EUR}_{t-2}$, $\text{GBP}_{t-1}$, $\text{GBP}_{t-2}$, $\text{JPY}_{t-1}$, $\text{JPY}_{t-2}\}$. Here, for example, $\text{JPY}_{t-1}$ is the exchange rate for the Japanese Yen from the previous time step. The experiments in the body of the paper consider the first $2000$ dailey rates. Figure \ref{fig:boc_dataset} provides a visualization of the dataset for a subset of currencies over the entire time range.
\begin{figure}[ht!]
    \centering
    \includegraphics[width=0.8\linewidth]{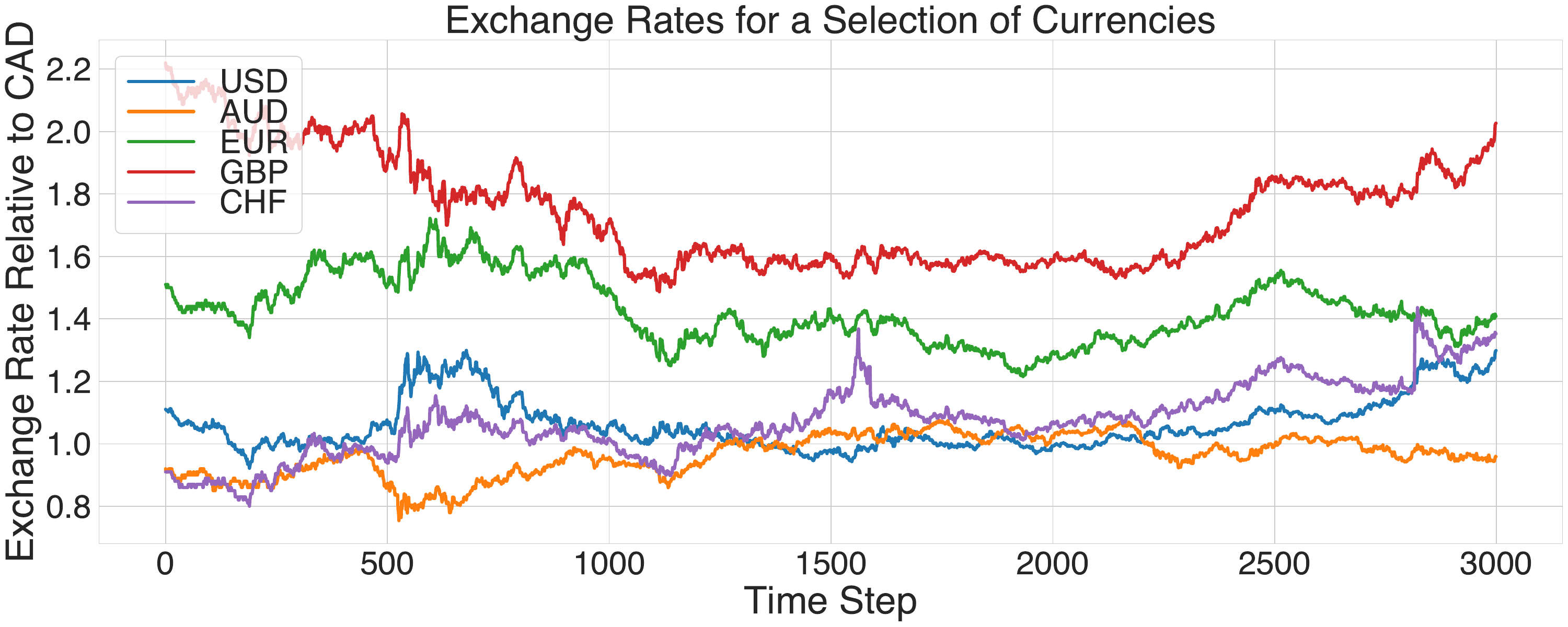}
    \caption{Exchange rates relative to the Canadian dollar for a subset of the $12$ currencies tracked in the BoC Exchange Rate time series.}
    \label{fig:boc_dataset}
\end{figure}

\subsubsection{Electricity Transformer Temperature} 
\label{s:ExperimentDetials__ss::datasets__sss:::ETT}
The Electricity Transformer Temperature (ETT) datasets were first introduced in \citep{haoyietal-informer-2021} as a difficult benchmark for the Informer time-series model. The target value for prediction in each dataset is the oil temperature (OT) at a given time within one of two electricity transformers. The datasets span a time-frame from July 1, 2016 to June 26, 2018. In particular, the ETT-small-h1 time series is studied, which records temperature measurements every hour during this date-span along with measurements of six other input features.

In the experiments in the body of this paper, the first $2000$ data points are used. All six input features (HUFL, HULL, MUFL, MULL, LUFL, LULL) are used to make predictions. When predicting $\text{OT}_t$, inputs include $\text{OT}_{t-1}$, $\text{OT}_{t-2}$, along with $\text{X}_{t-1}$, $\text{X}_{t-2}$, $\text{X}_{t-3}$, where $\text{X}$ represents an aforementioned input feature. Finally, all values in the dataset have been normalized by dividing the input and target sequences by the maximum value seen across the time series of interest. As will be discussed in Appendix \ref{normalization_study}, this is an important pre-processing step, because large input and target values tend to cause instabilities in the numerical computations and/or saturate the activation functions of the ESN and RFN models. While an exact normalization value may not be available during online prediction, an approximate normalizing scalar is likely sufficient. Figure \ref{fig:ett_dataset} exhibits a slice of the first $3,000$ steps in the time series along with the value of the input features at that time.

\begin{figure}[ht!]
    \centering
    \includegraphics[width=0.8\linewidth]{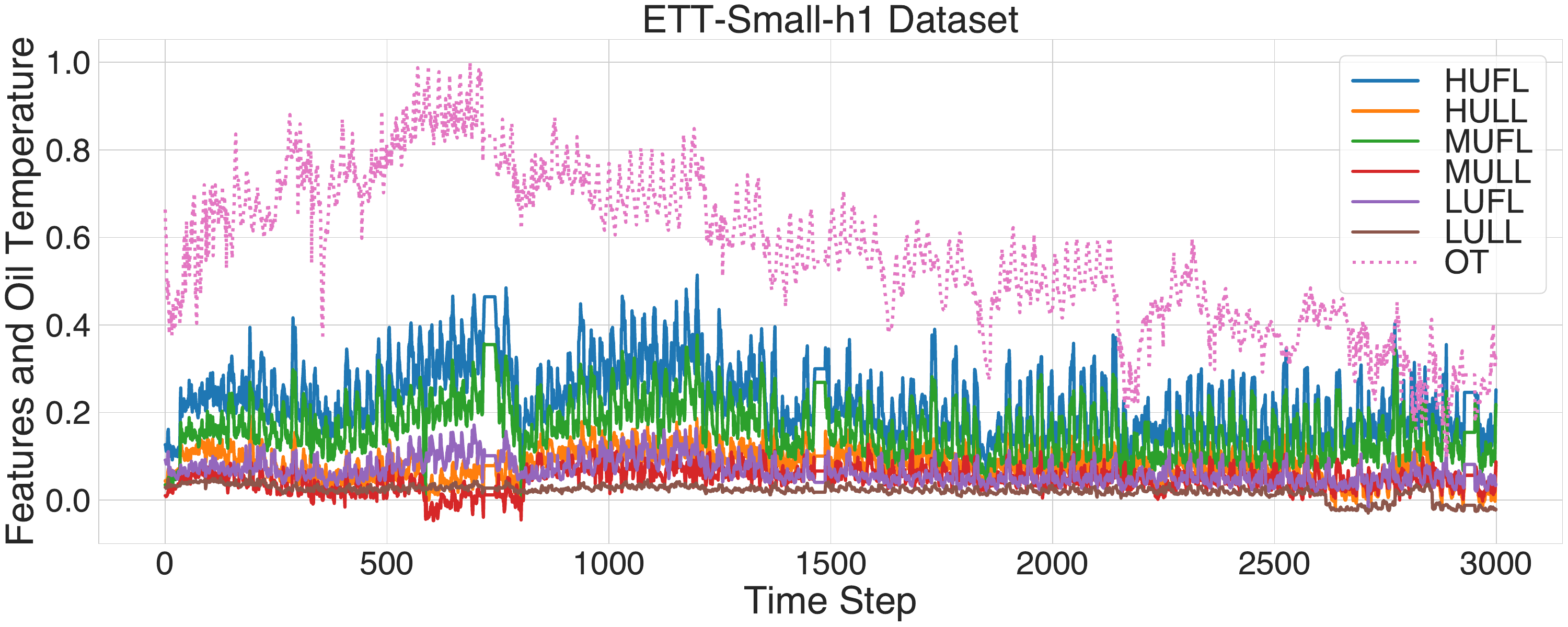}
    \caption{A snapshot of the normalized ETT-Small-h1 dataset with target value OT over time, along with the corresponding features used to help make predictions.}
    \label{fig:ett_dataset}
\end{figure}

\begin{table}[H]
\centering
\caption{Summary of hyperparameters searched for transformer models across datasets and whether Nash-game synchronization was used. The top section corresponds to the two expert setting and the bottom to five experts. In all settings, $\kappa=1.0$, $\eta=1.0$, $d_z = 2$, and synchronization frequency, $\tau = 1$, regardless of other hyperparameters. Bold numbers correspond to the optimal parameters based on mean MSE in the experiments.}
\label{tab:transformer_hyperparams}
\resizebox{0.65\linewidth}{!}{\begin{tabular}{llcccc}
\toprule
Dataset & Nash & $\alpha$ & $\gamma$ & Client $T$ & Game $T$ \\
\midrule
\multirow{2}{*}{Periodic} & True & $\{\mathbf{0.001}, 0.01, 0.1, 1, 5\}$ & $\{1, \mathbf{10}, 15\}$ & $\{\mathbf{3}, 4, 5, 6, \mathbf{7}\}$ & $\{2, \mathbf{3}, 4\}$ \\
& False & $\{0.001, 0.01, \mathbf{0.1}, 1, 5\}$ & $\{1, \mathbf{10}\}$ & $\{3, \mathbf{4}, 5, 6\}$ & $-$ \\
\multirow{2}{*}{Logistic} & True & $\{0.001, 0.01, 0.1, \mathbf{1}, 5\}$ & $\{1, \mathbf{10}, 15\}$ & $\{3, \mathbf{4}, 5, 6\}$ & $\{\mathbf{1}, 2, 3, 4\}$ \\
& False & $\{0.001, \mathbf{0.01}, 0.1, 1, 5\}$ & $\{1, \mathbf{10}, 15\}$ & $\{3, \mathbf{4}, 5, 6\}$ & $-$ \\
\multirow{2}{*}{Concept} & True & $\{0.001, 0.01, 0.1, \mathbf{1}, 5\}$ & $\{1, \mathbf{10}\}$ & $\{3, \mathbf{4}, 5\}$ & $\{\mathbf{2}, 3, 4\}$ \\
& False & $\{0.001, 0.01, 0.1, \mathbf{1}, 5\}$ & $\{\mathbf{1}, 10\}$ & $\{3, 4, 5, \mathbf{6}\}$ & $-$ \\
\multirow{2}{*}{BoC} & True & $\{0.001, 0.01, 0.1, 1, \mathbf{5}\}$ & $\{1, \mathbf{10}\}$ & $\{3, 4, 5, \mathbf{6}\}$ & $\{\mathbf{2}, 3, 4\}$ \\
& False & $\{0.001, 0.01, 0.1, 1, \mathbf{5}\}$ & $\{\mathbf{1}, 10\}$ & $\{3, 4, 5, \mathbf{6}\}$ & $-$ \\
\multirow{2}{*}{ETT} & True & $\{0.001, 0.01, 0.1, \mathbf{1}, 5\}$ & $\{\mathbf{1}, 10\}$ & $\{\mathbf{3}, 4, 5, 6\}$ & $\{\mathbf{2}, 3, 4\}$ \\
& False & $\{0.001, 0.01, 0.1, 1, \mathbf{5}\}$ & $\{\mathbf{1}, 10\}$ & $\{\mathbf{3}, 4, 5, 6\}$ & $-$ \\
\midrule
\multirow{2}{*}{Periodic} & True & $\{\mathbf{0.001}, 0.01, 0.1, 1, 5\}$ & $\{1, \mathbf{10}\}$ & $\{3, 4, \mathbf{5}, 6\}$ & $\{2, \mathbf{3}, 4\}$ \\
& False & $\{0.001, 0.01, \mathbf{0.1}, 1, 5\}$ & $\{1, \mathbf{10}\}$ & $\{3, \mathbf{4}, 5, 6\}$ & $-$ \\
\multirow{2}{*}{Logistic} & True & $\{0.001, 0.01, 0.1, \mathbf{1}, 5\}$ & $\{1, 10, \mathbf{15}\}$ & $\{3, \mathbf{4}, 5, 6\}$ & $\{\mathbf{1}, 2, 3, 4\}$ \\
& False & $\{\mathbf{0.001}, 0.01, 0.1, 1, 5\}$ & $\{1, \mathbf{10}, 15\}$ & $\{3, 4, 5, \mathbf{6}\}$ & $-$ \\
\multirow{2}{*}{Concept} & True & $\{0.001, 0.01, 0.1, \mathbf{1}, 5\}$ & $\{1, \mathbf{10}\}$ & $\{3, \mathbf{4}, 5\}$ & $\{\mathbf{2}, 3, 4\}$ \\
& False & $\{0.001, 0.01, 0.1, \mathbf{1}, 5\}$ & $\{\mathbf{1}, 10\}$ & $\{3, 4, 5, 6, 7, \mathbf{8}\}$ & $-$ \\
\multirow{2}{*}{BoC} & True & $\{0.001, 0.01, 0.1, 1, \mathbf{5}\}$ & $\{\mathbf{1}, 10\}$ & $\{\mathbf{3}, 4, 5, 6\}$ & $\{\mathbf{2}, 3, 4\}$ \\
& False & $\{0.001, 0.01, 0.1, 1, \mathbf{5}\}$ & $\{\mathbf{1}, 10\}$ & $\{3, 4, 5, \mathbf{6}\}$ & $-$ \\
\multirow{2}{*}{ETT} & True & $\{0.001, 0.01, 0.1, \mathbf{1}, 5\}$ & $\{\mathbf{1}, 10\}$ & $\{\mathbf{3}, 4, 5, 6\}$ & $\{2, 3, \mathbf{4}\}$ \\
& False & $\{0.001, 0.01, 0.1, 1, \mathbf{5}\}$ & $\{\mathbf{1}, 10\}$ & $\{3, 4, 5, \mathbf{6}\}$ & $-$ \\
\bottomrule
\end{tabular}}
\end{table}

\subsection{Best Hyperparameter Choices By Experiment}
\label{s:OptimalParamterChoices}

For all simulations, unless otherwise stated, a hyperparameter search is performed to find the best set of hyperparameters minimizing the average MSE for the time series in question. The full set of hyperparameters explored for each dataset, model, and Nash game setting are recorded in Tables \ref{tab:transformer_hyperparams}-\ref{tab:esn_hyperparams}. In each experiment, the parameters $\kappa$ and $\eta$ are each fixed to $1.0$. Additionally, for the transformer experiments, $d_z$ is fixed at $2$ for simplicity and empirical observations of better performance. Note that $\sigma$ is also not used in transformer models and is therefore irrelevant in hyperparameter tuning. The hyperparameter optimization is done post-hoc. However, the results compare the best possible performance of the two approaches in this idealized setting. In addition, several experiments in Appendix \ref{validation_studies} compare performance after hyperparameter tuning on validation time series.

\begin{table}[H]
\centering
\caption{Summary of hyperparameters searched for RFN models across datasets and whether Nash-game synchronization was used. The top section corresponds to the two expert setting and the bottom to five experts. In all settings, $\kappa=1.0$ and $\eta=1.0$, regardless of other hyperparameters. Bold numbers correspond to the optimal parameters based on mean MSE in the experiments.}
\label{tab:rfn_hyperparams}
\resizebox{0.99\linewidth}{!}{\begin{tabular}{llccccccc}
\toprule
Dataset & Nash & $\alpha$ & $\gamma$ & $\sigma$ & $d_z$ & Client $T$ & Game $T$ & Sync Freq. \\
\midrule
\multirow{2}{*}{Periodic} & True & $\{\mathbf{0.001}, 0.01, 0.1, 1\}$ & $\{0.1, 1, \mathbf{10}\}$ & $\{\mathbf{0.1}, 1\}$ & $\{\mathbf{3}, 6\}$ & $\{2, \mathbf{3}, 5\}$ & $\{2, \mathbf{3}, 4\}$ & $\{\mathbf{1}, 2\}$ \\
& False & $\{0.001, 0.01, \mathbf{0.1}, 1\}$ & $\{0.1, 1, \mathbf{10}\}$ & $\{\mathbf{1}\}$ & $\{\mathbf{3}, 6\}$ & $\{2, \mathbf{3}, 5\}$ & $-$ & $-$ \\
\multirow{2}{*}{Logistic} & True & $\{\mathbf{0.0001}, 0.001, 0.01, 0.1, 1\}$ & $\{0.1, 1, 10, \mathbf{15}\}$ & $\{\mathbf{0.01}, 0.1, 1\}$ & $\{3, 4, \mathbf{5}, 6, 7\}$ & $\{2, 3, 4, \mathbf{5}\}$ & $\{2, \mathbf{3}, 4, 5\}$ & $\{\mathbf{1}, 2\}$ \\
& False & $\{\mathbf{0.0001}, 0.001, 0.01, 0.1, 1\}$ & $\{0.1, 1, 10, \mathbf{15}\}$ & $\{0.01, \mathbf{0.1}, 1\}$ & $\{\mathbf{3}, 4, 5, 6, 7\}$ & $\{\mathbf{2}, 3, 4, 5, 6\}$ & $-$ & $-$ \\
\multirow{2}{*}{Concept} & True & $\{\mathbf{0.001}, 0.01, 0.1\}$ & $\{1, \mathbf{10}\}$ & $\{\mathbf{0.01}, 0.1, 1.0\}$ & $\{\mathbf{3}, 6\}$ & $\{\mathbf{2}, 3, 5\}$ & $\{\mathbf{2}, 3, 4\}$ & $\{\mathbf{1}, 2\}$ \\
& False & $\{0.001, \mathbf{0.01}, 0.1\}$ & $\{1, \mathbf{10}\}$ & $\{0.01, \mathbf{0.1}, 1.0\}$ & $\{\mathbf{3}, 6\}$ & $\{\mathbf{2}, 3, 5\}$ & $-$ & $-$ \\
\multirow{2}{*}{BoC} & True & $\{0.01, 0.1, \mathbf{1}\}$ & $\{0.1, 1, \mathbf{10}\}$ & $\{\mathbf{1}\}$ & $\{\mathbf{3}, 6\}$ & $\{\mathbf{2}, 5\}$ & $\{\mathbf{2}, 3\}$ & $\{\mathbf{1}, 2\}$ \\
& False & $\{\mathbf{0.01}, 0.1, 1\}$ & $\{0.1, 1, \mathbf{10}\}$ & $\{\mathbf{1}\}$ & $\{\mathbf{3}, 6\}$ & $\{\mathbf{2}, 5\}$ & $-$ & $-$ \\
\multirow{2}{*}{ETT} & True & $\{0.01, 0.1, 1, \mathbf{5}\}$ & $\{10, \mathbf{15}\}$ & $\{\mathbf{0.1}, 1\}$ & $\{1, \mathbf{2}\}$ & $\{2, 3, 4, \mathbf{5}\}$ & $\{\mathbf{2}, 3\}$ & $\{\mathbf{1}\}$ \\
& False & $\{\mathbf{0.01}, 0.1, 1, 5\}$ & $\{\mathbf{10}, 15\}$ & $\{0.1, \mathbf{1}\}$ & $\{1, \mathbf{2}\}$ & $\{2, \mathbf{3}, 4, 5\}$ & $-$ & $-$ \\
\midrule
\multirow{2}{*}{Periodic} & True & $\{\mathbf{0.001}, 0.01, 0.1, 1\}$ & $\{1, \mathbf{10}\}$ & $\{\mathbf{0.1}, 1\}$ & $\{\mathbf{3}, 6\}$ & $\{2, \mathbf{3}, 5\}$ & $\{2, \mathbf{3}, 4\}$ & $\{\mathbf{1}\}$ \\
& False & $\{0.001, 0.01, \mathbf{0.1}, 1\}$ & $\{1, \mathbf{10}\}$ & $\{\mathbf{1}\}$ & $\{\mathbf{3}, 6\}$ & $\{2, \mathbf{3}, 5\}$ & $-$ & $-$ \\
\multirow{2}{*}{Logistic} & True & $\{\mathbf{0.0001}, 0.001, 0.01\}$ & $\{10, \mathbf{15}\}$ & $\{0.01, \mathbf{0.1}\}$ & $\{3, 4, \mathbf{5}\}$ & $\{\mathbf{3}, 4, 5\}$ & $\{2, \mathbf{3}, 4\}$ & $\{\mathbf{1}, 2\}$ \\
& False & $\{0.0001, 0.001, \mathbf{0.01}\}$ & $\{10, \mathbf{15}\}$ & $\{\mathbf{0.01}, 0.1\}$ & $\{\mathbf{3}, 4, 5\}$ & $\{\mathbf{3}, 4, 5\}$ & $-$ & $-$ \\
\multirow{2}{*}{Concept} & True & $\{\mathbf{0.001}, 0.01, 0.1\}$ & $\{1, \mathbf{10}\}$ & $\{0.01, 0.1, \mathbf{1.0}\}$ & $\{3, \mathbf{6}\}$ & $\{2, 3, \mathbf{5}\}$ & $\{\mathbf{2}, 3, 4\}$ & $\{\mathbf{1}, 2\}$ \\
& False & $\{0.001, 0.01, \mathbf{0.1}\}$ & $\{1, \mathbf{10}\}$ & $\{0.01, \mathbf{0.1}, 1.0\}$ & $\{\mathbf{3}, 6\}$ & $\{\mathbf{2}, 3, 5\}$ & $-$ & $-$ \\
\multirow{2}{*}{BoC} & True & $\{\mathbf{0.01}, 0.1, 1\}$ & $\{0.1, 1, \mathbf{10}\}$ & $\{\mathbf{1}\}$ & $\{3, \mathbf{6}\}$ & $\{2, \mathbf{5}\}$ & $\{\mathbf{2}, 3\}$ & $\{\mathbf{1}, 2\}$ \\
& False & $\{0.01, 0.1, \mathbf{1}\}$ & $\{0.1, 1, \mathbf{10}\}$ & $\{\mathbf{1}\}$ & $\{\mathbf{3}, 6\}$ & $\{2, \mathbf{5}\}$ & $-$ & $-$ \\
\multirow{2}{*}{ETT} & True & $\{0.01, 0.1, 1, \mathbf{5}\}$ & $\{\mathbf{10}, 15\}$ & $\{0.1, 1\}$ & $\{\mathbf{1}, 2\}$ & $\{\mathbf{2}, 3, 4, 5\}$ & $\{2, \mathbf{3}\}$ & $\{\mathbf{1}\}$ \\
& False & $\{0.01, \mathbf{0.1}, 1, 5\}$ & $\{\mathbf{10}, 15\}$ & $\{0.1, \mathbf{1}\}$ & $\{1, \mathbf{2}\}$ & $\{2, \mathbf{3}, 4, 5\}$ & $-$ & $-$ \\
\bottomrule
\end{tabular}}
\end{table}

\begin{table}[ht]
\centering
\caption{Summary of hyperparameters searched for ESN models across datasets and whether Nash-game synchronization was used. The top section corresponds to the two expert setting and the bottom to five experts. In all settings, $\kappa=1.0$ and $\eta=1.0$, regardless of other hyperparameters. Bold numbers correspond to the optimal parameters based on mean MSE in the experiments.}
\label{tab:esn_hyperparams}
\resizebox{0.99\linewidth}{!}{\begin{tabular}{llccccccc}
\toprule
Dataset & Nash & $\alpha$ & $\gamma$ & $\sigma$ & $d_z$ & Client $T$ & Game $T$ & Sync Freq. \\
\midrule
\multirow{2}{*}{Periodic} & True & $\{\mathbf{0.001}, 0.01, 0.1\}$ & $\{0.1, 1, \mathbf{10}\}$ & $\{0.01, \mathbf{1}\}$ & $\{\mathbf{3}, 6\}$ & $\{\mathbf{2}, 5, 10\}$ & $\{\mathbf{3}, 4, 5\}$ & $\{\mathbf{1}, 2, 3\}$ \\
& False & $\{0.001, 0.01, \mathbf{0.1}\}$ & $\{0.1, 1, \mathbf{10}\}$ & $\{\mathbf{0.01}, 1\}$ & $\{\mathbf{3}, 6\}$ & $\{2, \mathbf{5}, 10\}$ & $-$ & $-$ \\
\multirow{2}{*}{Logistic} & True & $\{\mathbf{0.001}, 0.01, 0.1\}$ & $\{0.1, 1, \mathbf{10}\}$ & $\{0.01, \mathbf{1}\}$ & $\{\mathbf{3}, 6, 12\}$ & $\{\mathbf{2}, 5, 10\}$ & $\{\mathbf{3}, 5\}$ & $\{\mathbf{1}, 2\}$ \\
& False & $\{0.001, 0.01, \mathbf{0.1}\}$ & $\{0.1, 1, \mathbf{10}\}$ & $\{\mathbf{0.01}, 1\}$ & $\{\mathbf{3}, 6, 12\}$ & $\{2, \mathbf{5}, 10\}$ & $-$ & $-$ \\
\multirow{2}{*}{Concept} & True & $\{0.001, 0.01, \mathbf{0.1}\}$ & $\{0.1, \mathbf{1}, 10\}$ & $\{\mathbf{0.01}, 1\}$ & $\{\mathbf{3}, 6, 12\}$ & $\{2, 5, \mathbf{10}\}$ & $\{\mathbf{3}, 5\}$ & $\{\mathbf{1}, 2\}$ \\
& False & $\{\mathbf{0.001}, 0.01, 0.1\}$ & $\{0.1, 1, \mathbf{10}\}$ & $\{0.01, \mathbf{1}\}$ & $\{3, \mathbf{6}, 12\}$ & $\{\mathbf{2}, 5, 10\}$ & $-$ & $-$ \\
\multirow{2}{*}{BoC} & True & $\{\mathbf{0.001}, 0.01, 0.1\}$ & $\{\mathbf{1}, 10\}$ & $\{\mathbf{0.01}, 1\}$ & $\{3, \mathbf{6}, 12\}$ & $\{2, \mathbf{5}, 10\}$ & $\{\mathbf{3}, 5\}$ & $\{\mathbf{1}\}$ \\
& False & $\{\mathbf{0.001}, 0.01, 0.1\}$ & $\{1, \mathbf{10}\}$ & $\{0.01, \mathbf{1}\}$ & $\{3, \mathbf{6}, 12\}$ & $\{\mathbf{2}, 5, 10\}$ & $-$ & $-$ \\
\multirow{2}{*}{ETT} & True & $\{0.01, 0.1, 1, \mathbf{5}\}$ & $\{\mathbf{1}, 10\}$ & $\{\mathbf{1}\}$ & $\{\mathbf{1}, 2, 3\}$ & $\{\mathbf{2}, 3, 4, 5\}$ & $\{\mathbf{2}, 3\}$ & $\{\mathbf{1}\}$ \\
& False & $\{0.01, 0.1, 1, \mathbf{5}\}$ & $\{\mathbf{1}, 10\}$ & $\{\mathbf{1}\}$ & $\{1, 2, \mathbf{3}\}$ & $\{\mathbf{2}, 3, 4, 5\}$ & $-$ & $-$ \\
\midrule
\multirow{2}{*}{Periodic} & True & $\{\mathbf{0.001}, 0.01, 0.1, 1, 10\}$ & $\{0.01, 0.1, 1, \mathbf{10}\}$ & $\{0.01, \mathbf{1}\}$ & $\{\mathbf{3}, 6\}$ & $\{\mathbf{2}, 5, 10\}$ & $\{\mathbf{3}, 4, 5\}$ & $\{\mathbf{1}, 2, 3\}$ \\
& False & $\{0.001, 0.01, \mathbf{0.1}, 1, 10\}$ & $\{0.01, 0.1, 1, \mathbf{10}\}$ & $\{\mathbf{0.01}, 1\}$ & $\{\mathbf{3}, 6\}$ & $\{2, \mathbf{5}, 10\}$ & $-$ & $-$ \\
\multirow{2}{*}{Logistic} & True & $\{\mathbf{0.001}, 0.01, 0.1\}$ & $\{0.1, 1, \mathbf{10}\}$ & $\{0.01, \mathbf{1}\}$ & $\{3, \mathbf{6}, 12\}$ & $\{2, \mathbf{5}, 10\}$ & $\{\mathbf{3}, 5\}$ & $\{\mathbf{1}, 2\}$ \\
& False & $\{\mathbf{0.001}, 0.01, 0.1\}$ & $\{0.1, 1, \mathbf{10}\}$ & $\{\mathbf{0.01}, 1\}$ & $\{\mathbf{3}, 6, 12\}$ & $\{\mathbf{2}, 5, 10\}$ & $-$ & $-$ \\
\multirow{2}{*}{Concept} & True & $\{\mathbf{0.001}, 0.01, 0.1\}$ & $\{\mathbf{1}, 10\}$ & $\{\mathbf{0.01}, 1\}$ & $\{\mathbf{3}, 6\}$ & $\{2, \mathbf{5}\}$ & $\{\mathbf{2}, 3\}$ & $\{1\}$ \\
& False & $\{0.001, 0.01, \mathbf{0.1}\}$ & $\{1, \mathbf{10}\}$ & $\{\mathbf{0.01}, 1\}$ & $\{\mathbf{3}\}$ & $\{\mathbf{2}, 5\}$ & $-$ & $-$ \\
\multirow{2}{*}{BoC} & True & $\{\mathbf{0.001}, 0.01, 0.1\}$ & $\{\mathbf{10}\}$ & $\{\mathbf{1}\}$ & $\{3, \mathbf{6}\}$ & $\{3, \mathbf{5}\}$ & $\{\mathbf{2}, 3\}$ & $\{\mathbf{1}\}$ \\
& False & $\{0.001, 0.01, \mathbf{0.1}\}$ & $\{\mathbf{10}\}$ & $\{\mathbf{1}\}$ & $\{3, \mathbf{6}\}$ & $\{\mathbf{3}, 5\}$ & $-$ & $-$ \\
\multirow{2}{*}{ETT} & True & $\{0.01, 0.1, 1, \mathbf{5}\}$ & $\{\mathbf{1}, 10\}$ & $\{\mathbf{1}\}$ & $\{\mathbf{1}, 2, 3\}$ & $\{\mathbf{2}, 3, 4, 5\}$ & $\{\mathbf{2}, 3\}$ & $\{\mathbf{1}\}$ \\
& False & $\{0.01, 0.1, 1, \mathbf{5}\}$ & $\{\mathbf{1}, 10\}$ & $\{\mathbf{1}\}$ & $\{1, 2, \mathbf{3}\}$ & $\{2, \mathbf{3}, 4, 5\}$ & $-$ & $-$ \\
\bottomrule
\end{tabular}}
\end{table}

Experiments considering ESN and RFN models demonstrate an inherent randomness due, for example, to the construction of their affine maps and perturbative noise. Thus, for these models, simulations are run with three random seeds ($2024$, $2025$, $2026$) in all settings and results are averaged over the three runs. This is done for both hyperparameter tuning and reporting of final results.

\subsection{Pre-training Details for transformer Models}
\label{a:TransExpHyper}

During the pre-training phase of the transformer model, the encoder and the final linear decoder layer are trained to predict the target sequence given the input sequence. This is achieved using a causal mask, which ensures that the transformer relies only on the current and previous time step inputs to predict  $\hat{y}_{t+1}$. The output of the transformer’s encoder is then passed to a final trainable linear layer to predict the target at $t+1$. After pre-training for a few epochs, the encoder of the transformer, excluding the final linear layer, is used as the encoder model in the main experiments.

The first layer of the transformer is an embedding layer that maps the input dimension to the transformer's hidden dimension. A positional encoding layer is then added to the input. The resulting input is then passed to the transformer encoder.  The transformer encoder consists of two transformer layers, each consisting of self-attention and feed-forward networks. The number of expected features is set to $d_z \times d_y$, the number of heads is set to $2$, and the dimension of the feed-forward network is set to $16$, which are selected based on some initial hyperparameter tuning of the model. The dropout is set to $0.1$ for all experiments. The decoder layer is the final linear layer that transforms the encoder’s output $Z_t$ from a dimension of $d_y \times d_z$ to $d_y$.  This final layer is dropped and tuned following the main algorithm.

\subsection{Further Ablation Studies}
\label{a:Ablations}

This section contains additional, relevant ablation studies, the conclusions of which can be useful for the successful use of the FL algorithm.

\subsubsection{Validation Time Series Results} \label{validation_studies}

Studies applying ESN and RFN models with and without Nash-game synchronization on validation slices of the BoC Exchange Rate and ETT time series are reported in this section. In the experiments of Sections \ref{s:Examples__ss:RFN__sss:Results} and \ref{s:Examples__ss:ESN__sss:Results}, the first $2000$ data points from each of these series are considered. Herein, a slice of the BoC Exchange Rate series from index $1600$ to $3600$ is used, along with a slice from the ETT series from index $3500$ to $5500$. Note that there is a small amount of overlap for the BoC Exchange Rate dataset due to length limitations. 

As in previous results, three prediction runs with random seeds of $2024$, $2025$, and $2026$ are used and the resulting MSEs averaged. The optimal hyperparameters from previous experiments for each setting are applied. A summary of the results is presented in Table \ref{tab:validation_studies}. In all but one of the settings, predictions with Nash-game synchronizations significantly outperform those without on the unseen data. For the setting where predictions without the Nash game outperform those with it, the drop in performance for Nash-game synchronization is traced back to a single prediction in one of the runs that generated an obvious and significant outlier. It is likely that this aberrant prediction is caused by a numerical instability in the Nash optimization and could easily be detected and discarded during online prediction. The remaining two runs produced MSEs on the order of $1\mathrm{e}{\text{-}5}$, which is a marked improvement over the non-Nash setting.

\begin{table}[ht!]
    \centering
    \caption{Results, averaged over three runs for ESN and RFN models on validation slices of the BoC Exchange Rate and ETT time series. The optimal hyperparameters from the studies in Sections \ref{s:Examples__ss:RFN__sss:Results} and \ref{s:Examples__ss:ESN__sss:Results} are used.}
    \label{tab:validation_studies}
    \begin{tabular}{lllcc}
    \toprule
    Model & \# of Experts & Dataset & Nash game & No Nash game \\
    \midrule
    \multirow{4}{*}{ESN} & \multirow{2}{*}{2} & BoC Exchange Rate & $\mathbf{2.85886\mathrm{e}{\text{-}5}}$& $1.47168\mathrm{e}{\text{-}4}$ \\
    & & ETT Series & $\mathbf{2.64996\mathrm{e}{\text{-}3}}$ & $4.15773\mathrm{e}{\text{-}3}$ \\
    & \multirow{2}{*}{5} & BoC Exchange Rate & $\mathbf{5.87193\mathrm{e}{\text{-}5}}$ & $1.22207\mathrm{e}{\text{-}3}$ \\
    & & ETT Series & $\mathbf{2.64991\mathrm{e}{\text{-}3}}$ & $3.82797\mathrm{e}{\text{-}3}$ \\
    \midrule
    \multirow{4}{*}{RFN} & \multirow{2}{*}{2} & BoC Exchange Rate & $3.17312\mathrm{e}{\text{-}2}$ & $\mathbf{7.54179\mathrm{e}{\text{-}3}}$ \\
    & & ETT Series & $\mathbf{5.12075\mathrm{e}{\text{-}3}}$ & $1.92244\mathrm{e}{\text{-}2}$ \\
    & \multirow{2}{*}{5} & BoC Exchange Rate & $\mathbf{4.20660\mathrm{e}{\text{-}5}}$ & $6.28034\mathrm{e}{\text{-}2}$ \\
    & & ETT Series & $\mathbf{2.65513\mathrm{e}{\text{-}3}}$ & $1.28415\mathrm{e}{\text{-}1}$ \\
    \bottomrule
    \end{tabular}
\end{table}

\subsubsection{ReLU Activations and ESN Recursion} \label{activation_discussion}

Recently, it has been shown that ESN models with ReLU activations have desirable properties, including Universality \citep{Gonon1}. A natural question then is why use Hard Sigmoid activations in the models explored here. Peripherally, the work in \citep{Li6} extends many of the properties proven in \citep{Gonon1} to ESNs for general activation functions. However, the primary motivation is that the unbounded nature of the ReLU activation function can result in significant growth in the Froebenius norm of the latent representations, $Z_t$. Such growth produces poorly conditioned matrices. This can be problematic for both the numerical solution of the ridge-regression problems required for predictions without the Nash game and the more complex systems required for the proposed Nash-game synchronization.

Consider the following code to replicate the recurrence relationship for the latent space of the ESN. 

\vspace{-2.5cm}
\hspace{-3.3cm}
\includegraphics[trim=0 450 0 0, clip, width=22cm,height=13.5cm]{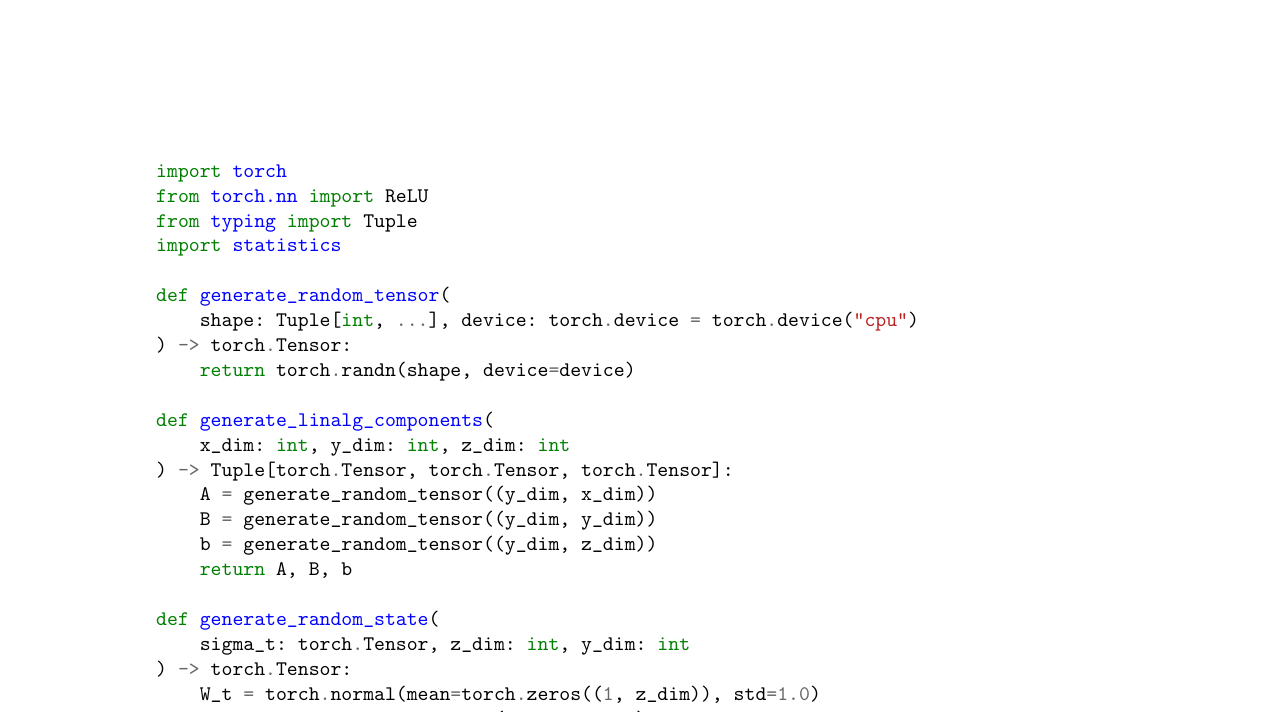}


\hspace{-3.3cm}
\includegraphics[trim=0 0 0 340, clip, width=22cm,height=17.5cm]{code_python.pdf}


\vspace{-5.3cm}

\noindent In the above, the input features are quite reasonable and the length of the time series is only 10 time steps long. In spite of this, for unlucky draws, the magnitudes of elements in $Z$ grow extremely rapidly. The largest of which is $87025$ in this simulation. Over the $100$ simulation trajectories, the mean value for the largest element in $Z$ is $1937$. While the distribution is fairly skewed by outliers, this example demonstrates the potential hazard in using ReLU activations in ESNs over long time series. Using a bounded activation function, such as the Hard Sigmoid, in the experiments above, largely mitigates this issue.

\subsubsection{ESN and RFN Models on Unnormalized Data} \label{normalization_study}

In Appendix \ref{s:ExperimentDetials__ss::datasets__sss:::ETT}, the importance of normalizing the inputs and outputs of the ETT time series is briefly mentioned. In this section, some studies considering the impact of failing to scale data when using ESN or RFN models is discussed. Consider a toy dataset composed of three Brownian motion trajectories with drift \citep{Glasserman1}. For each trajectory, the Brownian motion has a $\mu = 1.0$ and $\sigma = 1.0$. At each time step, $t$, the target is simply the sum of $x_t = 0.1\cdot t$ and the value of trajectory $i$ at that time step, $b_{i, t}$. The inputs are simply $x_t$ and $b_{i, t}$. So the model need only learn to properly add these two values together to make an accurate prediction. In the examples below, a time series of length $200$ is used, three trajectories are sampled as targets, two experts are present for the mixture, and ESN models are used.

As seen in Figure \ref{fig:normalization}, the values of the Brownian trajectories quickly grow fairly large, implying that the model inputs also grow steadily. Examining the server predictions without Nash synchronization, it is clear that the server is unable to mix predictions from the two clients to make predictions for target variable $y_2$. The root cause of this failure is seen in considering the predictions produced by the individual clients. Both clients fail to produce meaningful predictions for $y_2$. Furthermore, Client 1 produces extremely poor predictions for all target dimensions. Thus, the server must rely exclusively on predictions from Client 0 to produce any meaningful predictions at all. Note that these issues remain present, even after significant hyperparameter tuning.

\begin{figure}[ht!]
    \centering
    \includegraphics[width=0.8\linewidth]{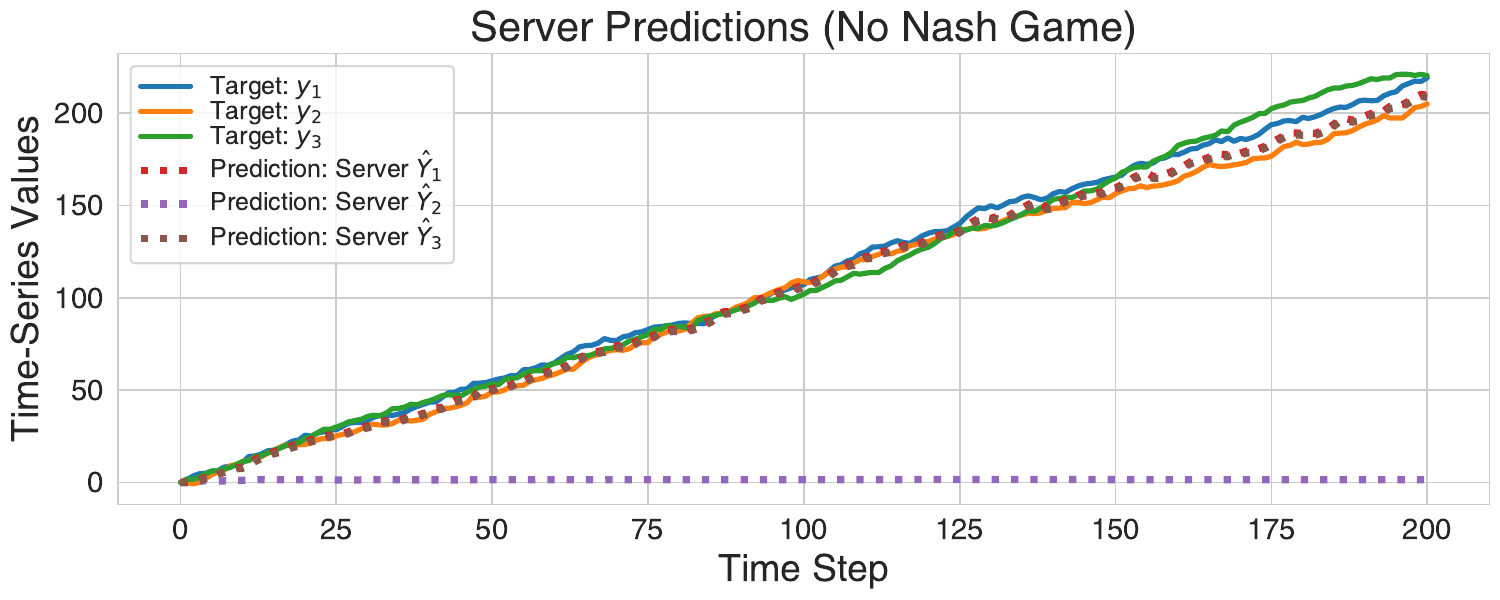}
    \includegraphics[width=0.8\linewidth]{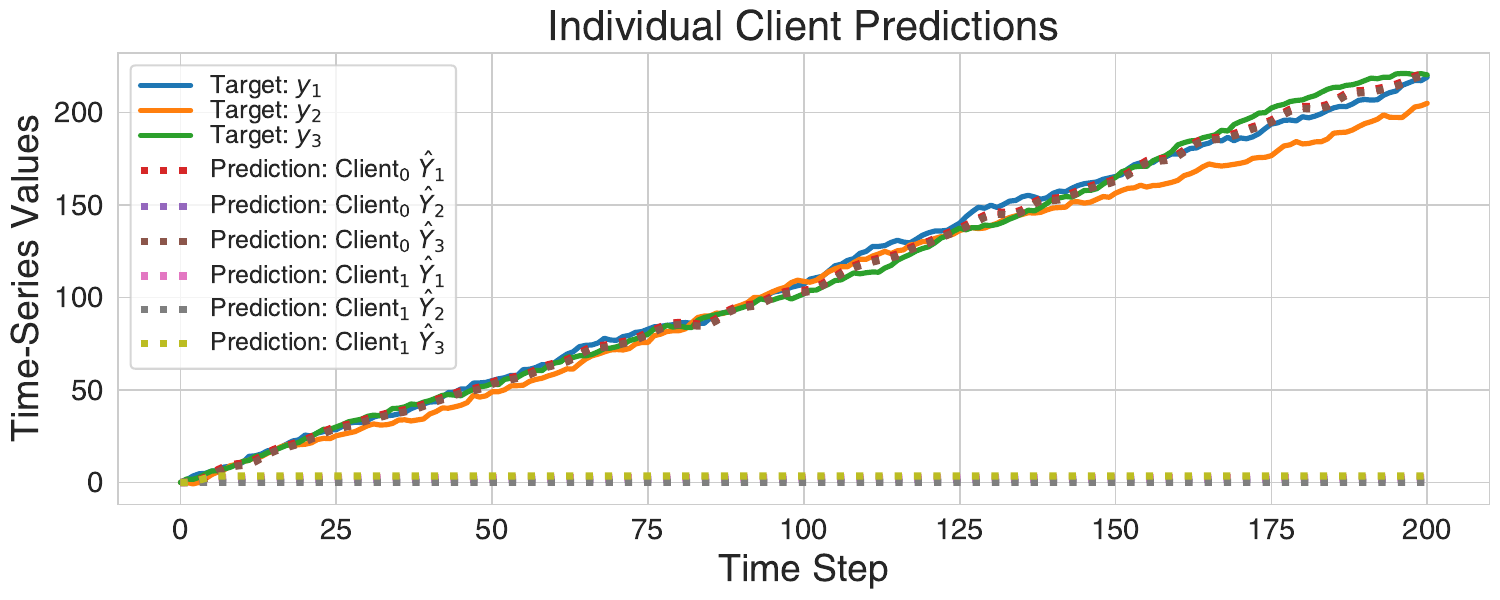}
    \caption{Predictions for the server (top) and individual clients (bottom) for the synthetic Brownian motion dataset. Due to the large input and target values, the clients, especially Client 1, fail to produce good predictions using ESN models.}
    \label{fig:normalization}
\end{figure}

The driver of these failures is a collapse in the latent representations, $Z_t$, from which clients derive their predictions. Recall that, for ESN models, 
\begin{align*}
    Z_t^i &= \text{Hard Sigmoid} \bullet \left (A^i x_t + B^i Z_{t-1}^i b^i + \sigma^i W_t^i \right) \\
    \hat{Y}_{t+1}^i &= \hat{Y}_t^i + Z_t^i \beta_t^i.
\end{align*}
If a row of $Z_t^i$ were to become zero, the previous prediction $\hat{Y}_t^i$ remains constant in that dimension, regardless of $\beta_t^i$. If the inputs to the Hard Sigmoid become negative enough, this can happen repeatedly. In the case of Figure \ref{fig:normalization}, after a few steps, the latent matrix $Z$ becomes zero across the board for Client 1 and the second row becomes zero for Client 0. 

This deficiency may be addressed in several ways. For example, a $\tanh$ function might be substituted for the Hard Sigmoid activation, thereby removing the vanishing representation problem. However, the issue of saturating the extremes of the activation function remains present. Experimentally, a suitable solution for the ETT time series is normalization of the input and output variables by some computed maximum. This approach was applied for all experiments. However, it is likely that an ideal setup would be a combination of $\tanh$ activation and normalization.

\end{document}

%% file: inputs/packages.tex
\usepackage{amsmath,amssymb,graphicx,mathrsfs,bbm,url}

\usepackage{hyperref}
\hypersetup{colorlinks=true,citecolor=blue,linkcolor=blue,urlcolor=black}

\usepackage[table]{xcolor}

\usepackage{float}
\usepackage{graphicx}
\usepackage{subcaption}
\usepackage{caption}
\usepackage{lscape}
\usepackage{booktabs}
\usepackage{multirow}

\usepackage{xparse}

\usepackage{comment}

\usepackage{amsmath,amssymb}
\usepackage{mathtools}
\usepackage{latexsym}
\usepackage{dsfont}
\usepackage{graphicx}
\usepackage{float}
\usepackage{color}
\usepackage{mathrsfs}

\usepackage{verbatim}
\usepackage{epstopdf}

\usepackage{stmaryrd}

\usepackage{url}

\usepackage[utf8]{inputenc}

\usepackage{natbib}



%% file: inputs/mathcommands.tex
\DeclareMathOperator\diag{diag}

\newcommand{\eqdef}{\ensuremath{\stackrel{\mbox{\upshape\tiny def.}}{=}}}

\usepackage[colorinlistoftodos]{todonotes}
\usepackage{soul}

\def\1{\bm{1}}

\definecolor{darkcyan}{rgb}{0.0, 0.55, 0.55}
\definecolor{MidnightBlue}{RGB}{25,25,112}
\definecolor{MidnightBlueComplementingGreen}{RGB}{25,112,25}
\definecolor{MidnightBlueComplementingPurple}{RGB}{112,25,112}
\definecolor{MidnightBlueComplementingRed}{RGB}{112,25,69}
\definecolor{WowColor}{rgb}{.75,0,.75}
\definecolor{MildlyAlarming}{rgb}{0.85,0.25,0.1}
\definecolor{SubtleColor}{rgb}{0,0,.50}
\definecolor{antiquefuchsia}{rgb}{0.57, 0.36, 0.51}
\definecolor{fashionfuchsia}{rgb}{0.96, 0.0, 0.63}
\definecolor{jade}{rgb}{0.0, 0.66, 0.42}
\definecolor{caribbeangreen}{rgb}{0.0, 0.8, 0.6}
\definecolor{aquamarine}{rgb}{0.5, 0.8, 0.85}
\definecolor{lightseagreen}{rgb}{0.13, 0.7, 0.67}
\definecolor{darkgreen}{rgb}{0.0, 0.2, 0.13}
\definecolor{darkspringgreen}{rgb}{0.09, 0.45, 0.27}
\definecolor{attentioncolor}{RGB}{152,90,81}
\definecolor{burgred}{RGB}{40,3,22}
\definecolor{AnnieGreen}{RGB}{17,123,92}
\definecolor{Turquoise}{RGB}{64,224,208}

\definecolor{darkjade}{RGB}{0,122,84}
\definecolor{Window1}{RGB}{92,150,31}%
    \definecolor{Window1dark}{RGB}{41,67,13}%
\definecolor{Window2}{RGB}{255,168,28}
    \definecolor{Window2dark}{RGB}{114,75,12}
\definecolor{Window3}{RGB}{255,96,33}
    \definecolor{Window3dark}{RGB}{97,36,12}
\definecolor{InputColor}{RGB}{20,255,177}
    \definecolor{InputColorlight}{RGB}{222,237,229}

\definecolor{RedAlizarin}{rgb}{0.82, 0.1, 0.26}

\definecolor{darkcerulean}{rgb}{0.03, 0.27, 0.49}
\definecolor{smokyblack}{rgb}{0.06, 0.05, 0.03}
\definecolor{warmblack}{rgb}{0.0, 0.26, 0.26}
\definecolor{cobalt}{rgb}{0.0, 0.28, 0.67}
\definecolor{darkerred}{RGB}{139, 0, 0}
\definecolor{darkergreen}{RGB}{0, 100, 0}
\definecolor{warmred}{RGB}{205, 92, 92} 
\definecolor{warmgreen}{RGB}{34, 139, 34} 
\definecolor{warmblue}{RGB}{70, 130, 180} 
\definecolor{arrowredfig1}{HTML}{B85450} 
\definecolor{expertfig1}{HTML}{6A9153} 
\definecolor{mixturefig1}{HTML}{6C8EBF} 

\usepackage{algpseudocode}
\usepackage[linesnumbered,lined,boxed,commentsnumbered,ruled,longend]{algorithm2e}
    \RestyleAlgo{ruled}

\SetCommentSty{mycommfont}

\usepackage{cleveref}
\newcounter{termcounter}
\crefname{term}{term}{terms}
\creflabelformat{term}{#2\textup{(#1)}#3}

\makeatletter
\def\term{\@ifnextchar[\term@optarg\term@noarg}
\def\term@optarg[#1]#2{%
  \textup{#1}%
  \def\@currentlabel{#1}%
  \def\cref@currentlabel{[][2147483647][]#1}%
  \cref@label[term]{#2}}
\def\term@noarg#1{%
  \refstepcounter{termcounter}%
  \textup{(\thetermcounter)}%
  \cref@label[term]{#1}}
\makeatother

%% file: inputs/Commenting.tex
\usepackage{marginnote}
\definecolor{MidnightBlue}{RGB}{25,25,112}
\definecolor{MidnightBlueComplementingGreen}{RGB}{25,112,25}
\definecolor{MidnightBlueComplementingPurple}{RGB}{112,25,112}
\definecolor{MidnightBlueComplementingRed}{RGB}{112,25,69}
\definecolor{coolblack}{rgb}{0.0, 0.18, 0.39}

\definecolor{deepjunglegreen}{rgb}{0.0, 0.29, 0.29}
\definecolor{applegreen}{rgb}{0.55, 0.71, 0.0}
\definecolor{WowColor}{rgb}{.75,0,.75}
\definecolor{MildlyAlarming}{rgb}{0.85,0.25,0.1}
\definecolor{SubtleColor}{rgb}{0,0,.50}
\definecolor{SubtleColor2}{rgb}{0.6,0.21,.50}

\definecolor{lasallegreen}{rgb}{0.03, 0.47, 0.19}
\newcounter{margincounter}

\NewDocumentCommand{\AK}{mo}{
    \IfValueF{#2}{
                        {{\scriptsize
                            \textcolor{deepjunglegreen}{
                            \hfill\\
                                \textbf{A:}
                                \textit{{#1}}
                            \hfill\\
                            }
                        }}
        }
    \IfValueT{#2}{
                        \marginnote{{\scriptsize
                            \textcolor{deepjunglegreen}{ 
                            \textbf{A:}
                            \textit{{#1}}
                            }
                        }}
        }
                    }

\NewDocumentCommand{\Xuwei}{mo}{
    \IfValueF{#2}{
                        {{\scriptsize
                            \textcolor{coolblack}{ 
                            \textbf{X:}
                            \textit{{#1}}
                            }
                        }}
        }
    \IfValueT{#2}{
                        \marginnote{{\scriptsize
                            \textcolor{coolblack}{ 
                            \textbf{X:}
                            \textit{{#1}}
                            }
                        }}
        }
                    }

\NewDocumentCommand{\IE}{mo}{
    \IfValueF{#2}{
                        {{\scriptsize
                            \textcolor{orange}{ 
                            \textbf{IE:}
                            \textit{{#1}}
                            }
                        }}
        }
    \IfValueT{#2}{
                        \marginnote{{\scriptsize
                            \textcolor{orange}{ 
                            \textbf{IE:}
                            \textit{{#1}}
                            }
                        }}
        }
                    }

\NewDocumentCommand{\GA}{mo}{
    \IfValueF{#2}{
                        {{\scriptsize
                            \textcolor{violet}{ 
                            \textbf{GA:}
                            \textit{{#1}}
                            }
                        }}
        }
    \IfValueT{#2}{
                        \marginnote{{\scriptsize
                            \textcolor{violet}{ 
                            \textbf{GA:}
                            \textit{{#1}}
                            }
                        }}
        }
                    }
